\documentclass{article}
\usepackage{ijcai24}
\usepackage[toc,page,header]{appendix}
\usepackage{cite}
\usepackage{mathdots}
\usepackage[dvipsnames]{xcolor}
\usepackage{color}
\usepackage{times}
\usepackage{soul}
\usepackage{url}
\usepackage[hidelinks]{hyperref}
\usepackage[utf8]{inputenc}
\usepackage[small]{caption}
\usepackage{graphicx}
\usepackage{amsmath}
\usepackage{amsthm}
\usepackage{amsmath,amsfonts}
\usepackage{booktabs}
\usepackage[switch]{lineno}

\usepackage[normalsize]{subfigure}
\usepackage{wrapfig}
\usepackage{algorithmic}
\usepackage[ruled, vlined, commentsnumbered, linesnumbered]{algorithm2e}

\newtheorem{property}{\bf Property}[section]
\newtheorem{theorem}{\bf{Theorem}}[section]

\newtheorem{proposition}[theorem]{\bf{Proposition}}

\newtheorem{definition}[theorem]{\bf{Definition}}
\newcommand{\gr}{\color{ForestGreen}}
\newcommand{\dcheck}{\color{black}}

\newcommand{\DEL}[1]{\iffalse #1 \fi}
\newcommand{\rd}{\color{red}}
\newcommand{\squishlist}{
\begin{list}{$\bullet$}
  { \setlength{\itemsep}{0pt}
     \setlength{\parsep}{0pt}
     \setlength{\topsep}{0pt}
     \setlength{\partopsep}{0pt}
     \setlength{\leftmargin}{0em}
     \setlength{\labelwidth}{0em}
     \setlength{\labelsep}{0.2em} } }

\title{Enhancing Scalability of Metric Differential Privacy via Secret Dataset Partitioning and Benders Decomposition}
\vspace{-0.0in}

\author{
    Chenxi Qiu
    \affiliations
    Department of Computer Science and Engineering, University of North Texas
    \emails
    chenxi.qiu@unt.edu
}

\begin{document}

\maketitle

\begin{abstract}
\emph{Metric Differential Privacy (mDP)} extends the concept of Differential Privacy (DP) to serve as a new paradigm of data perturbation. It is designed to protect secret data represented in general metric space, such as text data encoded as word embeddings or geo-location data on the road network or grid maps. To derive an optimal data perturbation mechanism under mDP, a widely used method is \emph{linear programming (LP)}, which, however, might suffer from a polynomial explosion of decision variables, rendering it impractical in large-scale mDP.


In this paper, our objective is to develop a new computation framework to enhance the scalability of the LP-based mDP. Considering the connections established by the mDP constraints among the secret records, we partition the original secret dataset into various subsets. Building upon the partition, we reformulate the LP problem for mDP and solve it via \emph{Benders Decomposition}, which is composed of two stages: (1) a master program to manage the perturbation calculation across subsets and (2) a set of subproblems, each managing the perturbation derivation within a subset. Our experimental results on multiple datasets, including geo-location data in the road network/grid maps, text data, and synthetic data, underscore our proposed mechanism's superior scalability and efficiency.

\end{abstract}

\vspace{-0.20in}
\section{Introduction}
\vspace{-0.05in}

Among the array of data privacy protection mechanisms, \emph{data perturbation} has emerged as a widely employed technique to protect individuals' information. Data perturbation involves intentionally introducing noise into a database, rendering individual information unreadable by unauthorized users even in the event of data breaches on the server side. Particularly, \emph{Differential Privacy (DP)} \cite{Dwork-TC2006} has become a paradigm of choice for data perturbation 
due to its strong and theoretically provable privacy guarantees. 

In its original definition, DP requires data perturbation to maintain a uniform \emph{indistinguishability} level for the queries of the \emph{neighboring databases} \cite{Dwork-TC2006}. Specifically, the classification of ``neighboring databases'' relies on their \emph{Hamming distance}, i.e., two databases are considered ``neighbors'' if they differ by at most one record. However, this definition limits DP's applicability in the domains where data similarity is evaluated by other distance metrics, and also varying degrees of indistinguishability are necessitated by nonbinary distance values among neighbors. 
Recognizing these constraints, DP has been extended to \emph{metric DP (mDP)} \cite{ImolaUAI2022}, which generalizes the classification of ``neighboring databases'' by accounting for the diverse underlying distance metric spaces.


\vspace{0.02in}
\noindent \textbf{Related Work of mDP}. mDP originated in the domain of geo-location privacy protection \cite{Andres-CCS2013}, requiring ``\emph{geo-indistinguishability}'' for each pair of locations with the Euclidean distance lower than a predetermined threshold. In other words, mDP defines ``neighboring locations'' based on Euclidean distance, diverging from the original DP which relies on Hamming distance. Over time, mDP has been explored across a spectrum of metric choices, including 
Manhattan distance \cite{Chatzikokolakis-PETS2013}, Hyperbolic distance \cite{Feyisetan-ICDM2019}, Haversine distance \cite{Pappachan-EDBT2023}, Word Mover’s distance \cite{Fernandes-PST2019},  and others \cite{Feyisetan-2021-private}.


Compared to the conventional DP, the optimization of mDP poses additional challenges due to the diverse sensitivity of utility loss to data perturbation in general distance metric spaces  \cite{Qiu-TMC2022}. A commonly employed approach is to discretize both the secret dataset and the corresponding perturbed data domain into finite sets \cite{ImolaUAI2022}, which allows for explicit measurement of the utility loss caused by each perturbation choice. As such, the problem of optimizing the perturbation distribution of each secret record can be formulated as a \emph{linear programming (LP)} problem \cite{Fawaz-CCS2014}, of which the objective is to minimize the expected utility loss caused by data perturbation while satisfying the mDP constraints for each pair of neighboring records.

\begin{figure}[t]
\centering
\hspace{0.00in}
\begin{minipage}{0.44\textwidth}
  \subfigure{
\includegraphics[width=1.00\textwidth]{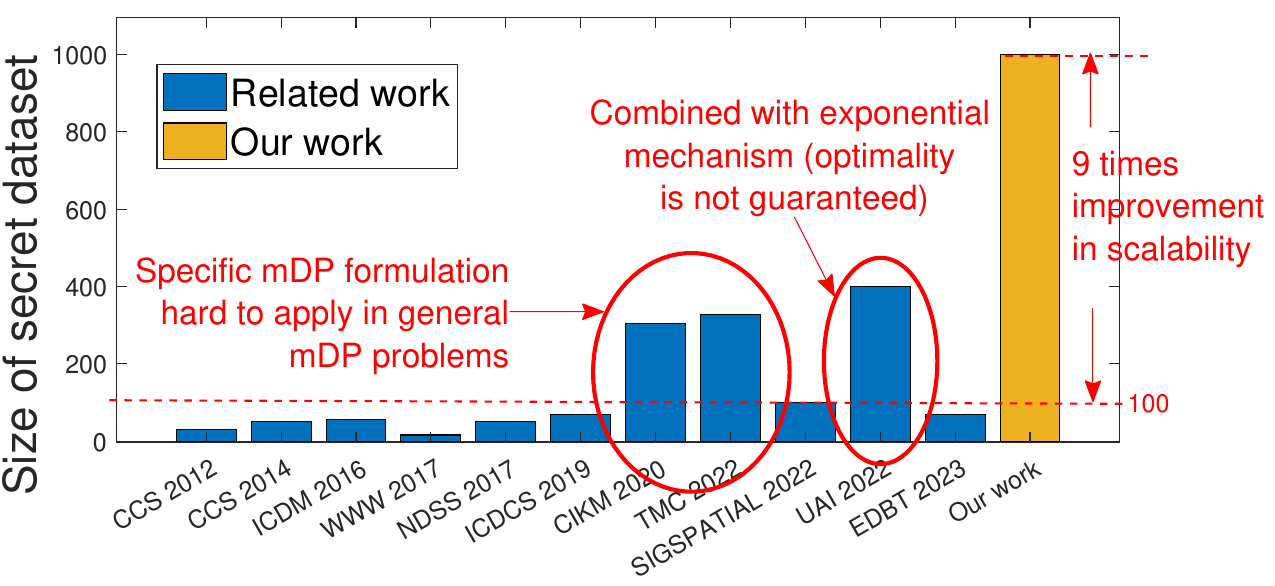}}
\end{minipage}
\vspace{-0.2in}
\caption{Comparison of the secret dataset size in the related LP-based mDP works and our work.
\vspace{0.02in}
\newline \small CCS 2012 \cite{Shokri-CCS2012}, CCS 2014 \cite{Fawaz-CCS2014}, ICDM 2016 \cite{Wang-CIDM2016}, WWW 2017 \cite{Wang-WWW2017}, NDSS 2017 \cite{Yu-NDSS2017}, ICDCS 2019 \cite{Qiu-ICDCS2019}, CIKM \cite{Qiu-CIKM2020}, TMC 2022 \cite{Qiu-TMC2022}, SIGSPATIAL 2022 \cite{Qiu-SIGSPATIAL2022}, UAI 2022 \cite{ImolaUAI2022}, EDBT 2023 \cite{Pappachan-EDBT2023}.
}
\label{fig:relatedworks}
\vspace{-0.16in}
\end{figure}

However, the LP formulation of mDP typically requires $O(N^2)$ decision variables and $O(N^2 K)$ linear constraints \cite{Bordenabe-CCS2014}, where $N$ and $K$ are the size of the secret dataset and perturbed dataset, respectively. The high complexity of LP makes it hard to apply in a large-scale mDP. As illustrated by Fig. \ref{fig:relatedworks}, most current LP-based mDP works have to limit their applications to small-scale secret datasets (i.e., up to 100 records in the secret dataset \cite{Chatzikokolakis-PoPETs2015}). While some recent efforts \cite{Qiu-CIKM2020, Qiu-TMC2022} have managed to extend the size of secret datasets to around 300 records using Dantzig-Wolfe decomposition and column generation algorithms, their focus has been on specific LP formulations where the mDP constraints are imposed for all pairs of secret records, simplifying the initialization of column generation. This specific LP formulation, however, is hard to apply in general mDP problems where the constraints are required only for neighboring records. Alternative approaches like \cite{ImolaUAI2022} integrate LP with the exponential mechanism to enhance scalability, which, however, comes at the cost of sacrificing the optimality of mDP.


\vspace{0.02in}
\noindent \textbf{Our main contributions}. Our primary aim is to enhance the scalability of the LP-based mDP by introducing a new computation framework. We first construct an \emph{mDP graph} to describe the mutual mDP constraints among secret records and then partition the entire secret dataset into well-balanced subsets based on the mDP graph (\textbf{Contribution 1}). Building upon this partition, we then reformulate the LP and solve it using \emph{Benders Decomposition} (\textbf{Contribution 2}), which consists of two stages: (1) a master program to manage the perturbation distribution calculations across subsets and (2) a set of subproblems, each deriving the perturbation distributions of records within a subset. The problems in both stages exhibit a relatively small scale, enabling efficient solutions. This efficiency facilitates the iterative derivation of a near-optimal solution for the original LP through the communication between the two stages. 

Finally, we test the performance of the new computation framework using geo-location data (in the road network and grid maps), text embeddings, and synthetic datasets, with a comparison of several benchmarks  (\textbf{Contribution 3}). The experimental results demonstrate that our new computation framework can derive the optimal mDP for the secret dataset with 1,000 records, marking approximately 9 times improvement in scalability compared to the state-of-the-art methods for general mDP problems, as shown in Fig. \ref{fig:relatedworks}. 


\vspace{-0.10in}
\section{Preliminaries}
\vspace{-0.03in}

In this part, we introduce the preliminary knowledge of mDP in \textbf{Section \ref{subsec:mDP}} and its LP computation framework in \textbf{Section \ref{subsec:LP}}. Table \ref{Tb:Notationmodel} in the Supplementary file lists the main notations used throughout this paper.

\vspace{-0.05in}
\subsection{Metric DP}
\label{subsec:mDP}
In general, a data perturbation mechanism can be represented as a \emph{probabilistic function} $Q$: $\mathcal{R} \rightarrow \mathcal{O}$, of which the domain $\mathcal{R}$ is users' \emph{secret dataset} and the range $\mathcal{O}$ is the \emph{perturbed dataset}. Given the underlying data features of $\mathcal{R}$, a metric $d: \mathcal{R}^2 \mapsto \mathbb{R}$ is defined to measure the \emph{distance} between any two records $r_i, r_j \in \mathcal{R}$, where the distance between $r_i$ and $r_j$ is denoted by $d_{i, j}$. 

\begin{definition}
\label{def:neighbor}
(Neighboring records) Given the secret dataset $\mathcal{R}$ and its distance measure $d$, any pair of records $r_i, r_j \in \mathcal{R}$ are called \emph{neighboring records} if their distance $d_{i, j} \leq \eta$, where $\eta >0$ is a pre-determined threshold. 
\end{definition}
In what follows, we use $\mathcal{E} = \left\{\left(r_i, r_j\right) \in \mathcal{R}^2 \left|d_{i, j} \leq \eta\right.\right\}$ to represent the whole set of neighboring records in $\mathcal{R}$. 
\begin{definition}
\label{def:metricDP}
(Metric DP) For each pair of neighboring records, $(r_i, r_j) \in \mathcal{E}$, $\epsilon$-mDP ensures that the probability distributions of their perturbed data $Q(r_i)$ and $Q(r_j)$ are sufficiently close so that it is hard for an attacker to distinguish $r_i$ and $r_j$ according to the probability distributions of $Q(r_i)$ and $Q(r_j)$
\vspace{-0.13in}
\begin{equation}
\label{eq:DP}
\small \frac{\mathbb{P}\mathrm{r}\left\{Q(r_i) \in O\right\}}{\mathbb{P}\mathrm{r}\left\{Q(r_j) \in O\right\}} \leq e^{\epsilon d_{i, j}}, ~\forall O \subseteq \mathcal{O}. 
\vspace{-0.00in}
\end{equation}
Here, $\epsilon > 0$ is called the \emph{privacy budget}, reflecting how much information can be disclosed from the perturbed data, i.e., lower $\epsilon$ implies a higher privacy level. 
\end{definition}


\vspace{-0.08in}
\subsection{Linear Programing-based Approaches}
\label{subsec:LP}
\vspace{-0.02in}
Considering the computational challenges of optimizing a perturbation function $Q$ defined on a continuous domain $\mathcal{R}$ and a continuous range $\mathcal{O}$, it is common practice to discretize both $\mathcal{R}$ and $\mathcal{O}$ to finite sets \cite{ImolaUAI2022}. In this case, the function $Q$ can be represented as a stochastic \emph{perturbation matrix} $\mathbf{Z} = \left\{z_{i,k}\right\}_{(r_i,o_k) \in \mathcal{R}\times \mathcal{O}}$, where each entry $z_{i,k}$ represents the probability of selecting $o_k \in \mathcal{O}$ as the perturbed record given the real record $r_i \in \mathcal{R}$. 
As such, the mDP constraints formulated in Equ. (\ref{eq:DP}) can be rewritten as the following linear constraints
\vspace{-0.02in}
\begin{equation}
\label{eq:DPdiscrete}
\small z_{i,k}- e^{\epsilon d_{i, j}} z_{j,k} \leq 0, ~\forall o_k \in \mathcal{O}, \forall (r_i, r_j) \in \mathcal{E}. 
\end{equation}
Additionally, the sum probability of perturbed record $o_k \in \mathcal{O}$ for each real record $r_i$ should be equal to 1 (the \emph{unit measure} of probability theory \cite{probability}), i.e., 
\begin{equation}
\label{eq:unitmeasure}
\textstyle \small \sum_{o_k \in \mathcal{O}}z_{i,k} = 1,~ \forall r_i \in \mathcal{R}. 
\vspace{-0.00in}
\end{equation}
We let $\mathbf{z}_i = [z_{i, 1}, ..., z_{i,K}]$ $(i =1,..., K)$ denote the record $r_i$'s \emph{perturbation vector}, i.e., the probability distribution of its perturbed records. 
We let $c_{i,k}$ denote the utility loss caused by the perturbed record $o_k$ when the real record is $r_i$. Then, the expected utility loss caused by the perturbation matrix $\mathbf{Z}$ can be represented by $\sum_{i=1}^N \mathbf{c}_i \mathbf{z}_{i}^{\top}$, where 
$\mathbf{c}_i = [p_i c_{i, 1}, ..., p_i c_{i,K}]$ and $p_i$ denotes the prior probability of the record being $r_i$.

Consequently, the objective of the \emph{perturbation matrix optimization (PMO)} problem is to minimize the expected data utility loss $\sum_{i=1}^N \mathbf{c}_i \mathbf{z}_{i}^{\top}$ while satisfying the constraints of both mDP and the probability unit measure, which can be formulated as the following LP problem: 
\normalsize

\small 
\vspace{-0.12in}
\begin{eqnarray}
\label{eq:LPobjective}
\min && \textstyle  \sum_{i=1}^N \mathbf{c}_i \mathbf{z}_{i}^{\top} \\
\mathrm{s.t.} && \mbox{mDP constraints (Equ. (\ref{eq:DPdiscrete})) are satisfied } \forall r_i \in \mathcal{R} \\
&& \mbox{Unit measure (Equ. (\ref{eq:unitmeasure})) is satisfied } \forall r_i \in \mathcal{R}\\ \label{eq:LPconstraint1}
&& 0 \leq z_{i,k} \leq 1, \forall (r_i, o_k) \in \mathcal{R} \times \mathcal{O}. 
\end{eqnarray}
\normalsize
The decision variables in the LP problem in Equ. (\ref{eq:LPobjective})-(\ref{eq:LPconstraint1}) are the perturbation matrix $\mathbf{Z}$, including $O(NK)$ decision variables (entries), constrained by $O(N^2 K)$ linear constraints. Such a high complexity makes this LP computation framework hard to apply in large-scale mDP applications \cite{Pappachan-EDBT2023}. 

\vspace{-0.05in}
\section{Methodology}
\vspace{-0.02in}
In this section, we present our approach to enhance the scalability of the mDP calculation via LP decomposition. 

\vspace{0.02in}
\noindent \textbf{Roadmap}. 
We first present the computational framework in \textbf{Section \ref{subsec:framework}}, followed by the detailed descriptions of its two components: \emph{Secret dataset partitioning} and \emph{Benders decomposition (BD)}. In our framework, the secret dataset partitioning precedes BD, but since the objective for partitioning the secret data is contingent on the complexity analysis of BD, we introduce BD in \textbf{Section \ref{subsec:benders}} and its time complexity analysis in \textbf{Section \ref{subsec:timecomplexity}} before delving into secret dataset partitioning in \textbf{Section \ref{subsec:cluster}}.


\vspace{-0.05in}
\subsection{Computation Framework}
\label{subsec:framework}
\vspace{-0.02in}
In general, the decomposition of an optimization problem highly depends on how its decision variables are coupled by constraints \cite{Palomar-JSAC2006}. 
A common strategy involves partitioning the original large-scale problem into a set of subproblems with smaller sizes. Within each subproblem, the decision variables are strongly linked by constraints, while the connections between decision variables across different subproblems are comparatively weaker. 

In PMO, the decision variables are the \emph{perturbation vectors} of secret records, $\mathbf{z}_i = [z_{i, 1}, ..., z_{i,K}]$ $(i =1,..., K)$, linked by the mDP constraints (Equ. (\ref{eq:DPdiscrete})). Here, we describe how the perturbation vectors of each pair of records are coupled by the mDP constraints in $\mathcal{R}$ by an undirected graph, called the \emph{mDP graph} (\textbf{Definition \ref{def:mDPgraph}}):  
\vspace{-0.02in}
\begin{definition}
\label{def:mDPgraph}
(mDP graph) The mDP graph $\mathcal{G}$ is defined as the ordered pair  $\mathcal{G} = \left(\mathcal{R}, \mathcal{E}\right)$, where the node set $\mathcal{R}$ is the secret dataset and the edge set $\mathcal{E}$ is the whole set of neighboring records in $\mathcal{R}$. The weight assigned to each edge $(r_i, r_j)$ is the distance $d_{i,j}$ between the two records.
\end{definition}
\vspace{-0.02in}

As Fig. \ref{fig:framework} shows, given the mDP graph $\mathcal{G}$, our intuitive idea is to first partition $\mathcal{G}$ into $M$ subgraphs: 
$\mathcal{G}_l = \left(\mathcal{R}_1, \mathcal{E}_1\right), ..., \mathcal{G}_M = \left(\mathcal{R}_M, \mathcal{E}_M\right)$,
such that the nodes within each subgraph are \emph{strongly connected} by the mDP constraints, and the nodes across different subgraphs are \emph{weakly connected} by the mDP constraints. Subsequently, a subproblem, denoted by $\mathrm{Sub}_l$, can be formulated to determine the perturbation vectors of the secret records within each subgraph $\mathcal{G}_l$ ($l = 1, ..., N$). Here, we use $\mathbf{z}_{\mathcal{R}_l} = \left\{\mathbf{z}_i, \middle| r_i \in \mathcal{R}_l\right\}$ to denote the perturbation vectors of the set of records in $\mathcal{R}_l$. 
\vspace{-0.00in}

Next, we introduce Property \ref{property:mDPcompo} and Property \ref{property:chainrule} of the mDP graph to facilitate the computation of the PMO problem.  The detailed proofs of the two properties  
can be found in Section \ref{subsec:proof:mDPcompo} and  Section \ref{subsec:proof:chainrule} in the supplementary file.
\vspace{-0.03in}
\begin{property}
\label{property:mDPcompo}
(Independent computation of mDP components) 
Suppose that the mDP graph $\mathcal{G}$ has $m$ components, denoted by $\mathcal{C}_l = \left(\mathcal{R}_l, \mathcal{E}_l\right)$ ($l = 1, ..., m$). Here, each component is a connected subgraph that is not part of any larger connected subgraph \cite{GraphTheory}.  Consider a subproblem  $\mathrm{Sub}_l$ formulated to determine the optimal perturbation vectors for each $\mathcal{R}_l$, where only the mDP constraints in $\mathcal{E}_l$ are considered. In this case, we can solve $\mathrm{Sub}_1, ..., \mathrm{Sub}_M$ independently, and their collective optimal solutions form the optimal solution for PMO.
\end{property}
\vspace{-0.07in}
\begin{property}
\label{property:chainrule}
Given that two records $r_i$ and $r_j$ (not necessarily neighbors) are connected by a path in the mDP graph $\mathcal{G}$, of which the \emph{shortest path distance} (i.e., the sum weight of the edges in the path) is $D_{i,j}$, then $z_{i,k}$ and $z_{j,k}$ are restricted by the following constraints: 
\vspace{-0.06in}
\begin{equation}
\label{eq:chainrule}
\small \small z_{i,k}- e^{\epsilon D_{i, j}} z_{j,k} \leq 0, ~\forall o_k \in \mathcal{O}. 
\vspace{-0.00in}
\end{equation}
\end{property}
\vspace{-0.00in}

\begin{figure}[t]
\centering
\hspace{0.00in}
\begin{minipage}{0.48\textwidth}
  \subfigure{
\includegraphics[width=1.00\textwidth]{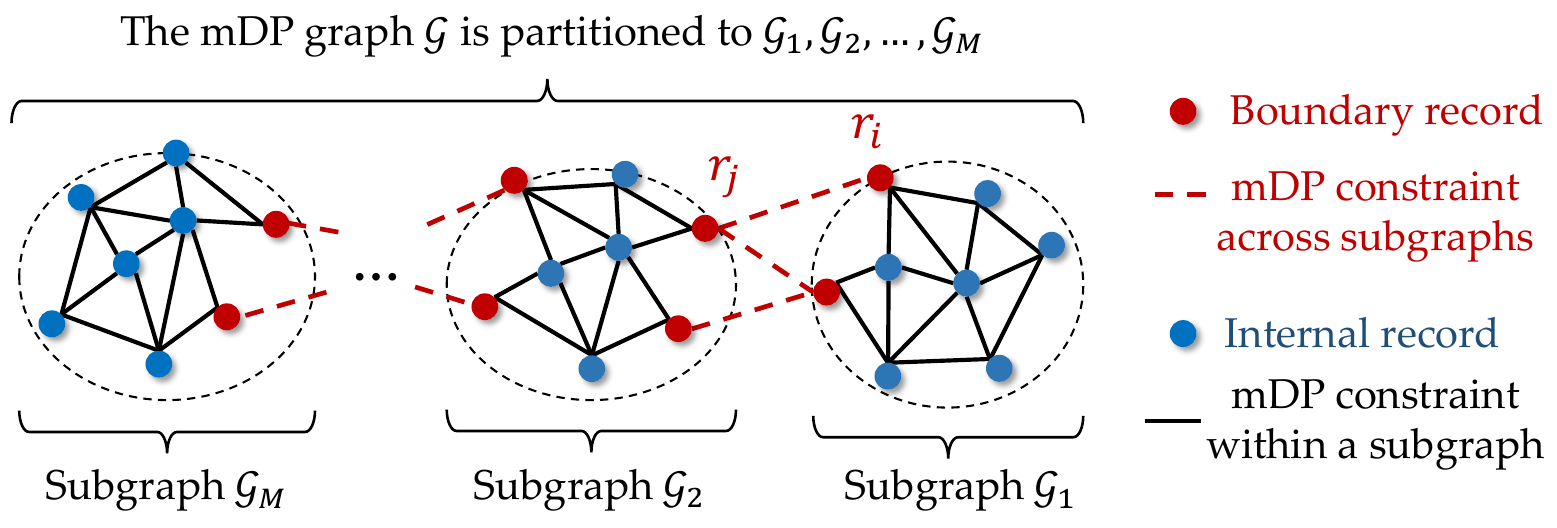}}
\vspace{-0.22in}
\end{minipage}
\caption{Computational framework.}
\label{fig:framework}
\vspace{-0.1in}
\end{figure}

In what follows, we restrict our attention to the case where the mDP graph $\mathcal{G}$ has a single component, considering the case of multiple components as straightforward to generalize according to Property \ref{property:mDPcompo}.  When $\mathcal{G}$ is a single-component graph, although dividing $\mathcal{G}$ facilitates the parallel computation of the major perturbation vectors within each subset, it remains imperative to jointly derive the perturbation vectors that are linked by mDP constraints across multiple subsets. For instance, in Fig. \ref{fig:framework}, the perturbation vector $\mathbf{z}_i$ of $r_i$ in $\mathcal{G}_1$ must adhere to mDP constraints with the perturbation vector $\mathbf{z}_j$ of $r_j$ in $\mathcal{G}_2$. Based on whether the records in each $\mathcal{R}_l$ are adjacent to records in other subsets by mDP constraints, we classify the records in each $\mathcal{R}_l$ into either ``\emph{boundary records}'' or ``\emph{internal records}'', as defined in \textbf{Definition \ref{def:bdry_int}}:
\vspace{-0.03in}
\begin{definition}
\label{def:bdry_int}
Each $r_i \in \mathcal{R}_l$ is a \textbf{boundary record} if it has at least a neighbor $r_j$ in another subset $\mathcal{R}_n$ ($n \neq l$) in the mDP graph $\mathcal{G}$; otherwise, $r_i$ is an \textbf{internal record}. 
\end{definition}
\vspace{-0.03in}




We use $\mathcal{X}_l$ and $\mathcal{Y}_l$ ($\mathcal{X}_l, \mathcal{Y}_l \subseteq \mathcal{R}_l$) to represent the \emph{internal record set} and the \emph{boundary record set} in $\mathcal{R}_l$, respectively. We let $\mathbf{z}_{\mathcal{X}_l}$ and $\mathbf{z}_{\mathcal{Y}_l}$ denote the perturbation vectors of $\mathcal{X}_l$ and $\mathcal{Y}_l$, i.e., 
$\mathbf{z}_{\mathcal{X}_l} = \left\{\mathbf{z}_i \left|r_i \in \mathcal{X}_l\right.\right\}$ and  $\mathbf{z}_{\mathcal{Y}_l} = \left\{\mathbf{z}_i \left|r_i \in \mathcal{Y}_l\right.\right\}$.

\vspace{0.03in}
\noindent \textbf{The PMO problem reformulation}. After categorizing the records in $\mathcal{R}_l$ into $\mathcal{X}_l$ and $\mathcal{Y}_l$ ($l = 1, ..., M$), the original PMO formulated in Equ. (\ref{eq:LPobjective})--(\ref{eq:LPconstraint1}) can be rewritten in a \emph{block ladder structure}, as shown in Fig. \ref{fig:blockstructure}: 
\newline (1) The \textbf{objective (OBJ)} $\sum_{i=1}^N \mathbf{c}_i \mathbf{z}_{i}$ is rewritten as the sum of 
\begin{itemize}
\item $\sum_{l=1}^M \mathbf{c}_{\mathcal{Y}_l} \mathbf{z}_{\mathcal{Y}_l}$  in the red block, representing the data utility loss of the boundary records in $\mathcal{Y}_1, ..., \mathcal{Y}_M$, and 
\item $\mathbf{c}_{\mathcal{X}_l} \mathbf{z}_{\mathcal{X}_l}$ ($l = 1, ..., M$) in the blue blocks, representing the data utility loss of the internal records in $\mathcal{X}_l$; 
\end{itemize}
(2) The \textbf{constraints (CSTR)} includes 
\begin{itemize}
\item $\mathbf{B}_{\mathcal{Y}_1, ..., \mathcal{Y}_M} \left[\begin{array}{c}\mathbf{z}_{\mathcal{Y}_1} \\
\vdots \\ \mathbf{z}_{\mathcal{Y}_M}\end{array}\right] \geq \mathbf{b}_{\mathcal{Y}_1, ..., \mathcal{Y}_M}$, including the mDP constraints that connect the boundary perturbation vectors $\mathbf{z}_{\mathcal{Y}_1}, ..., \mathbf{z}_{\mathcal{Y}_M}$ across $\mathcal{R}_1$, ..., $\mathcal{R}_M$, and the unit measure constraints of each $\mathbf{z}_{\mathcal{Y}_l}$ ($l = 1, ..., M$),   
\vspace{-0.05in}
\item $\mathbf{B}_{\mathcal{Y}_l} \mathbf{z}_{\mathcal{Y}_l} + \mathbf{A}_{\mathcal{X}_l} \mathbf{z}_{\mathcal{X}_l} \geq \mathbf{b}_{\mathcal{X}_l, \mathcal{Y}_l}$ ($l = 1, ..., M$), including the mDP constraints between the perturbation vectors within $\mathcal{R}_l$ (including both $\mathcal{X}_l$ and $\mathcal{Y}_l$), and their unit measure constraints. 
\end{itemize}
Such block ladder structure lends the reformulated PMO well to \emph{Benders decomposition (BD)} \cite{Rahmaniani-EJOR2017}. Due to the limit of space, we list the detailed formulations of the coefficient matrices $\mathbf{B}_{\mathcal{Y}_1, ..., \mathcal{Y}_M}$, $\mathbf{B}_{\mathcal{Y}_l}$, $\mathbf{A}_{\mathcal{X}_l}$, and coefficient vectors $\mathbf{b}_{\mathcal{Y}_1, ..., \mathcal{Y}_M}$, and $\mathbf{b}_{\mathcal{X}_l, \mathcal{Y}_l}$ in \textbf{Section \ref{subsec:PMONotationExp}} in the supplementary file. 

\begin{figure}[t]
\centering
\hspace{0.00in}
\begin{minipage}{0.50\textwidth}
  \subfigure{
\includegraphics[width=0.95\textwidth]{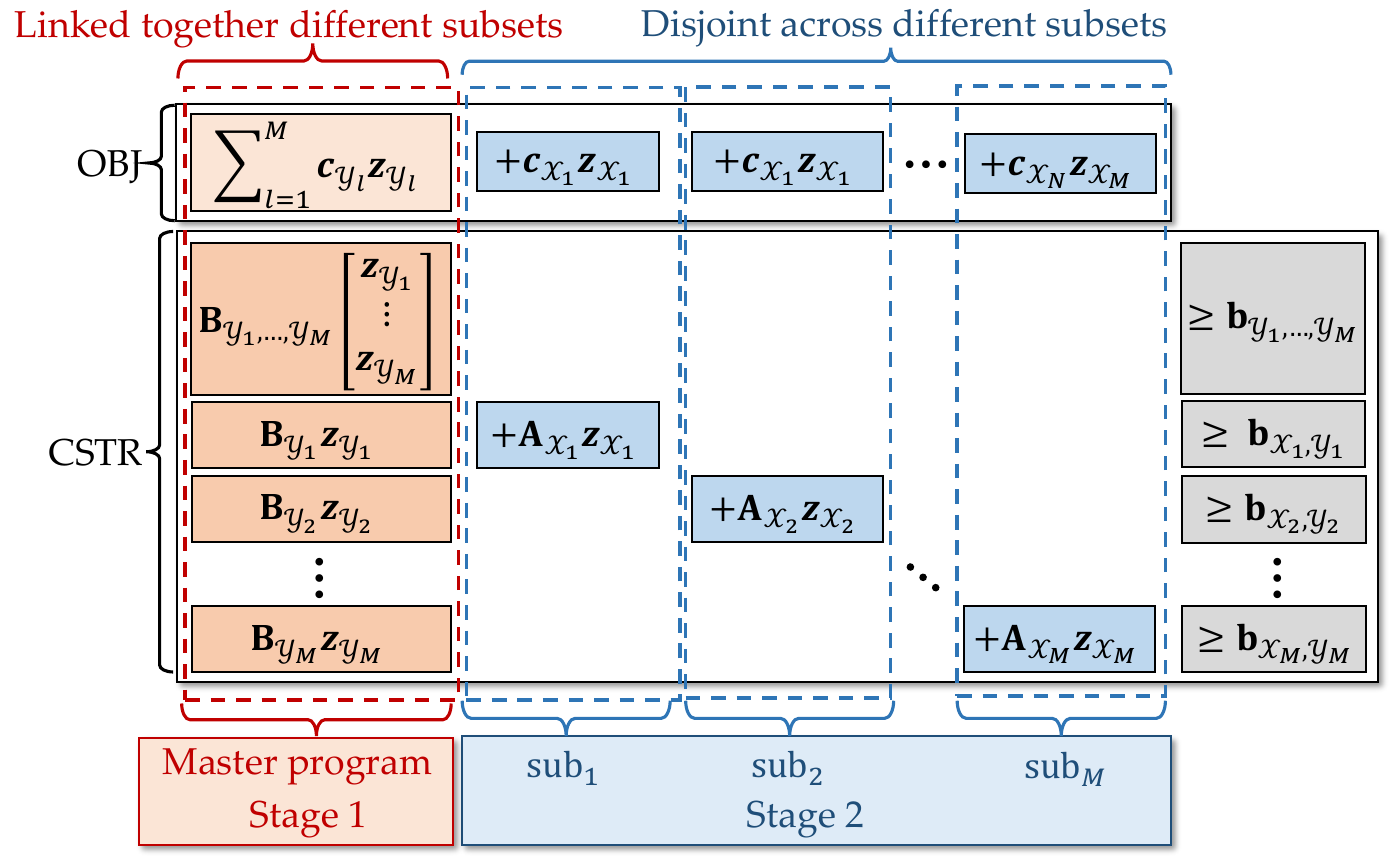}}
\vspace{-0.15in}
\end{minipage}
\caption{Block ladder structure of the PMO formulation.}
\label{fig:blockstructure}
\vspace{-0.15in}
\end{figure}
\vspace{-0.00in}
\DEL{
\begin{proposition}
\label{prop:PMOreform}
The PMO reformulation is equivalent to the original PMO problem defined in Equ. (\ref{eq:LPobjective})-(\ref{eq:LPconstraint1}). 
\end{proposition}}

\vspace{-0.05in}
\subsection{Benders Decomposition}
\label{subsec:benders}
\vspace{-0.02in}
As Fig. \ref{fig:BDdecomposition}(a) shows, BD is composed of two stages, a \textbf{master program (MP)} to derive the perturbation vectors of all the \textbf{boundary records}, and a set of \textbf{subproblems}, labeled as $\mathrm{Sub}_l$ ($l = 1, ..., M$), to derive the perturbation vectors of the \textbf{internal records} in each subset $\mathcal{R}_l$. 

\vspace{0.03in}
\noindent \textbf{Stage 1: Master program}. The MP derives the boundary records' perturbation vectors $\mathbf{z}_{\mathcal{Y}_1}, ..., \mathbf{z}_{\mathcal{Y}_M}$ and replaces the data utility loss of the internal records in each $\mathcal{X}_l$ by a single decision variable $w_l$, i.e., $w_l = \mathbf{c}_{\mathcal{X}_l} \mathbf{z}_{\mathcal{X}_l}$. The MP is formally formulated as the following LP problem
\normalsize
\small 
\begin{eqnarray}
\label{eq:MPObj}
\min && \textstyle \sum_{l=1}^M \mathbf{c}_{\mathcal{Y}_l} \mathbf{z}_{\mathcal{Y}_l} + \sum_{l=1}^M w_l \\
\mathrm{s.t.} && \mathbf{B}_{\mathcal{Y}_1, ..., \mathcal{Y}_M} \left[\begin{array}{c}\mathbf{z}_{\mathcal{Y}_1} \\
\label{eq:MPzy}
\vdots \\ \mathbf{z}_{\mathcal{Y}_M}\end{array}\right] \geq \mathbf{b}_{\mathcal{Y}_1, ..., \mathcal{Y}_M} \\
&& \mathcal{H}: \mbox{Cut set of $\mathbf{z}_{\mathcal{Y}_1}$, ..., $\mathbf{z}_{\mathcal{Y}_M}$, $w_1, ..., w_M$}\\ \label{eq:MPzy0}
&& \mathbf{z}_{\mathcal{Y}_l} \geq \mathbf{0}, ~ l = 1, ..., M. 
\end{eqnarray}
\normalsize
where each \emph{cut} in $\mathcal{H}$ is a \emph{linear inequality} of the decision variables $\mathbf{z}_{\mathcal{Y}_1}$, ..., $\mathbf{z}_{\mathcal{Y}_M}$, $w_1, ..., w_M$. 
According to PMO's reformulation (in Fig. \ref{fig:blockstructure}), $w_l$ is given by 
\begin{equation}
\label{eq:w_l}
\textstyle 
\small w_l = \min_{\mathbf{z}_{\mathcal{X}_l} \geq \mathbf{0}} \left\{\mathbf{c}_{\mathcal{X}_l} \mathbf{z}_{\mathcal{X}_l}\left|\mathbf{A}_{\mathcal{X}_l}  \mathbf{z}_{\mathcal{X}_l}  \geq \mathbf{b}_{\mathcal{X}_l, \mathcal{Y}_l} - \mathbf{B}_{\mathcal{Y}_l} \mathbf{z}_{\mathcal{Y}_l}\right.\right\}. 
\end{equation}
Since the MP doesn't know the optimal values of $\mathbf{z}_{\mathcal{X}_l}$, instead of using Equ. (\ref{eq:w_l}), it ``guesses'' the value of $w_l$ based the \emph{cut set} $\mathcal{H}$. In the subsequent \textbf{Stage 2}, each $\mathrm{Sub}_l$ verifies whether the ``guessed'' value of $w_l$ is feasible and achieves the minimum data utility loss as defined in Equ. (\ref{eq:w_l}); if not, $\mathrm{Sub}_l$ proposes the addition of a new cut to be included in $\mathcal{H}$, thereby guiding the MP to refine $w_l$ during the next iteration.

\vspace{0.03in}
\noindent \textbf{Initial cut}. According to \textbf{Property \ref{property:chainrule}} of mDP constraints, the MP can initialize the cut set $\mathcal{H}$ by 
$$\mathcal{H} = \left\{\mathbf{z}_{i}, \mathbf{z}_{j} | \mathbf{z}_{i}  \leq e^{{\epsilon D_{i,j}}}  \mathbf{z}_{j}, ~\forall r_i, r_j \in \cup_l \mathcal{Y}_l, (r_i, r_j) \notin \mathcal{E}\right\},$$ 
where $D_{i,j}$ is the shortest path distance between $r_i$ and $r_j$ in the mDP graph $\mathcal{G}$.  

In the following, we use $\left\{\overline{\mathbf{z}}_{\mathcal{Y}_1}, ..., \overline{\mathbf{z}}_{\mathcal{Y}_M}, \overline{w}_1, ..., \overline{w}_N\right\}$ to represent the optimal solution of the MP. 

\vspace{0.05in}
{\rd 
\DEL{
\noindent \textbf{Cut Initialization}. 
We use $\mathcal{G}_l$ to represent the \emph{induced subgraph} of the neighbor graph $\mathcal{G}_l$ that contains records $\mathcal{R}_l$, i.e., $\mathcal{G}_l$ contains all of the edges (from $\mathcal{G}_l$) connecting pairs of vertices in $\mathcal{R}_l$. 
\begin{theorem}
For each $\mathcal{R}_l$ ($l = 1, ..., M$), if its records satisfy mDP, then for each pair of boundary records $\forall r_i, r_j \in \mathcal{Y}_l$ that are connected in $\mathcal{G}_l$ by at least a path, the following constraint is satisfied 
\begin{equation}
z_{i,k}- e^{\epsilon D_{\mathcal{R}}\left(r_i, r_j\right)} z_{j,k} \leq 0.
\end{equation}
where $D_{\mathcal{R}}(r_i,r_j)$ denotes the \emph{the shortest path distance} between $r_i, r_j$ in $\mathcal{G}_l$. 
\end{theorem}}
}

\begin{figure}[t]
\centering
\hspace{0.00in}
\begin{minipage}{0.50\textwidth}
  \subfigure{
\includegraphics[width=1.00\textwidth]{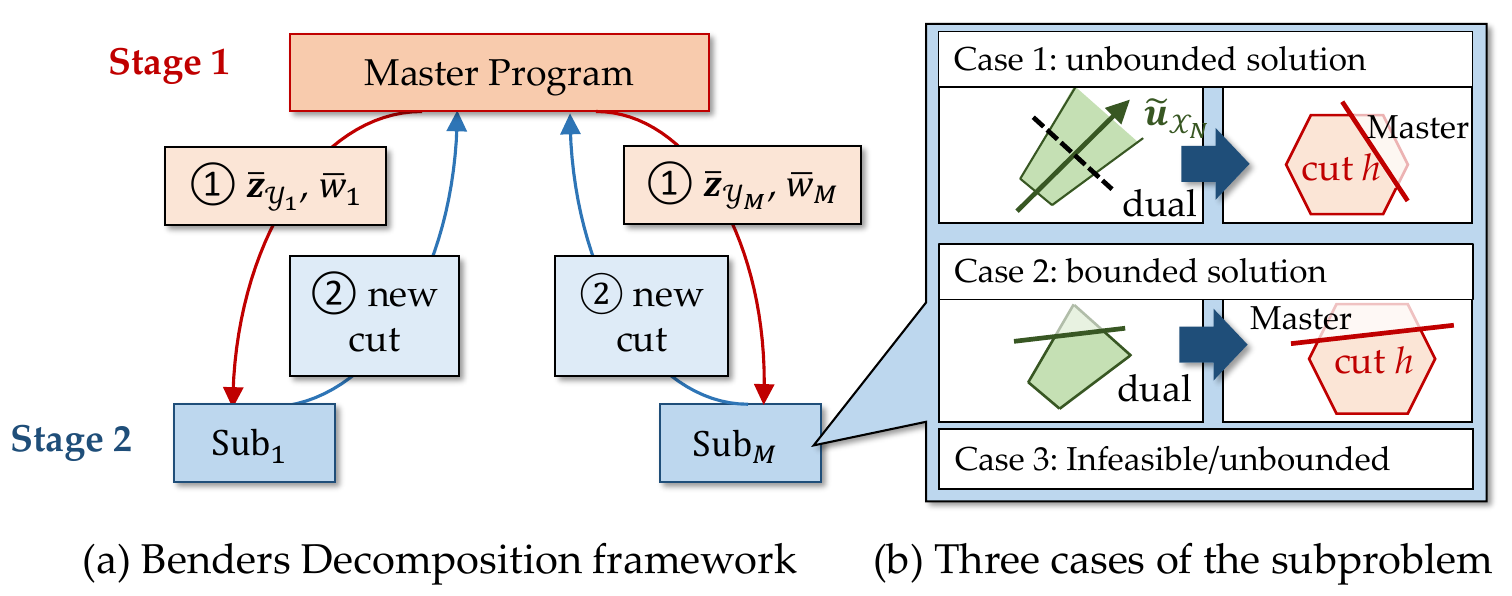}}
\vspace{-0.25in}
\end{minipage}
\caption{The Benders decomposition framework.}
\label{fig:BDdecomposition}
\vspace{-0.18in}
\end{figure}

\noindent \textbf{Stage 2: Subproblems}. After the MP derives its optimal solution $\left\{\overline{\mathbf{z}}_{\mathcal{Y}_1}, ..., \overline{\mathbf{z}}_{\mathcal{Y}_M}, \overline{w}_1, ..., \overline{w}_N\right\}$ in \textbf{Stage 1}, each $\mathrm{Sub}_l$ validates whether $\overline{w}_l$ has achieved the minimum data utility loss, 
\begin{equation}
\label{eq:w_l_primal}
\textstyle 
\small 
\overline{w}_l = \min_{\mathbf{z}_{\mathcal{X}_l} \geq \mathbf{0}} \left\{\mathbf{c}_{\mathcal{X}_l} \mathbf{z}_{\mathcal{X}_l}\left|\mathbf{A}_{\mathcal{X}_l}  \mathbf{z}_{\mathcal{X}_l}  \geq \mathbf{b}_{\mathcal{X}_l, \mathcal{Y}_l} - \mathbf{B}_{\mathcal{Y}_l} \overline{\mathbf{z}}_{\mathcal{Y}_l}\right.\right\},  
\end{equation} 
\DEL{
\begin{eqnarray}
\min && w_l = \mathbf{c}_{\mathcal{X}_l} \mathbf{z}_{\mathcal{X}_l} \\
\mathrm{s.t.} && \mathbf{A}_{\mathcal{X}_l}  \mathbf{z}_{\mathcal{X}_l}  \geq \mathbf{b}_{\mathcal{X}_l, \mathcal{Y}_l} - \mathbf{B}_{\mathcal{Y}_l} \overline{\mathbf{z}}_{\mathcal{Y}_l} \\ 
&& \mathbf{z}_{\mathcal{X}_l} \geq \mathbf{0}, 
\end{eqnarray}}
of which the \emph{dual problem} can be formulated as the following LP problem: 
\vspace{-0.10in}
\normalsize
\small 
\begin{eqnarray}
\label{eq:dualobj}
\max && \left(\mathbf{b}_{\mathcal{X}_l, \mathcal{Y}_l} -  \mathbf{B}_{\mathcal{Y}_l} \overline{\mathbf{z}}_{\mathcal{Y}_l}\right)^{\top} \mathbf{a}_{\mathcal{X}_l} + \mathbf{c}_{\mathcal{Y}_l} \overline{\mathbf{z}}_{\mathcal{Y}_l}  \\ \label{eq:dualconstr1}
\mathrm{s.t.} && \mathbf{A}_{\mathcal{X}_l}^{\top} \mathbf{a}_{\mathcal{X}_l} \leq \mathbf{c}_{\mathcal{X}_l},~ \mathbf{a}_{\mathcal{X}_l} \geq \mathbf{0}. 
\end{eqnarray}
\normalsize
As Fig. \ref{fig:BDdecomposition}(b) shows, there are three cases of the dual problem: 
\begin{itemize}

\item [\textbf{C1}:] The optimal objective value is \textbf{unbounded}: By \emph{weak duality} \cite{Linear&Nonlinear}, $\overline{\mathbf{z}}_{\mathcal{Y}_l}$ does not satisfy $\mathbf{A}_{\mathcal{X}_l}  \mathbf{z}_{\mathcal{X}_l}  \geq \mathbf{b}_{\mathcal{X}_l, \mathcal{Y}_l} - \mathbf{B}_{\mathcal{Y}_l} \overline{\mathbf{z}}_{\mathcal{Y}_l}$ for any $\mathbf{z}_{\mathcal{X}_l} \geq \mathbf{0}$. 
Since the dual problem is unbounded, there exists an \emph{extreme ray} $\tilde{\mathbf{a}}_{\mathcal{X}_l}$ s.t. 
$\mathbf{A}_{\mathcal{X}_l}^{\top} \tilde{\mathbf{a}}_{\mathcal{X}_l} \leq \mathbf{0}$ and $\left(\mathbf{b}_{\mathcal{X}_l, \mathcal{Y}_l} -  \mathbf{B}_{\mathcal{Y}_l} \overline{\mathbf{z}}_{\mathcal{Y}_l}\right)^{\top} \tilde{\mathbf{a}}_{\mathcal{X}_l} > 0$. To ensure that $\tilde{\mathbf{a}}_{\mathcal{X}_l}$ won't be an extreme ray in the next iteration, $\mathrm{Sub}_l$ suggests a \emph{new cut} $h$ (\emph{feasibility cut}) to the MP: 
\vspace{-0.08in}
\begin{equation}
\small h:~ \left(\mathbf{b}_{\mathcal{X}_l, \mathcal{Y}_l} -  \mathbf{B}_{\mathcal{Y}_l} \mathbf{z}_{\mathcal{Y}_l}\right)^{\top} \tilde{\mathbf{a}}_{\mathcal{X}_l} \leq \mathbf{0}.
\end{equation}
\item [\textbf{C2}:] The optimal objective value is \textbf{bounded} with the solution $\overline{\mathbf{a}}_{\mathcal{X}_l}$:  By \emph{weak duality}, the optimal value of the dual problem is equal to the optimal value of $w_l$ constrained on the choice of $\overline{\mathbf{z}}_{\mathcal{Y}_l}$. {\dcheck In this case, $\mathrm{Sub}_l$ checks whether $\overline{w}_l < \left(\mathbf{b}_{\mathcal{X}_l, \mathcal{Y}_l} -  \mathbf{B}_{\mathcal{Y}_l} \overline{\mathbf{z}}_{\mathcal{Y}_l}\right)^{\top} \overline{\mathbf{a}}_{\mathcal{X}_l} + \mathbf{c}_{\mathcal{Y}_l} \overline{\mathbf{z}}_{\mathcal{Y}_l}$}. If yes, then
\newline 
$\overline{w}_l < \min_{\mathbf{z}_{\mathcal{X}_l} \geq \mathbf{0}} \left\{\mathbf{c}_{\mathcal{X}_l} \mathbf{z}_{\mathcal{X}_l}\left|\mathbf{A}_{\mathcal{X}_l}  \mathbf{z}_{\mathcal{X}_l}  \geq \mathbf{b}_{\mathcal{X}_l, \mathcal{Y}_l} - \mathbf{B}_{\mathcal{Y}_l} \overline{\mathbf{z}}_{\mathcal{Y}_l}\right.\right\}$,  
indicating that $\overline{w}_l$ guessed by the MP is lower than the minimum data utility loss. Therefore, $\mathrm{Sub}_l$ suggests a \emph{new cut} $h:~ w_l \geq \left(\mathbf{b}_{\mathcal{X}_l, \mathcal{Y}_l} -  \mathbf{B}_{\mathcal{Y}_l} \mathbf{z}_{\mathcal{Y}_l}\right)^{\top} \overline{\mathbf{a}}_{\mathcal{X}_l} + \mathbf{c}_{\mathcal{Y}_l} \mathbf{z}_{\mathcal{Y}_l}$ to the MP to improve $w_l$ in the next iteration. 
\item [\textbf{C3}:] There is \textbf{no feasible solution}: By \emph{weak duality}, the primal problem either has no feasible/unbounded solution. The algorithm terminates.
\end{itemize}
After adding the new cuts (from all the subproblems) to the cut set $\mathcal{H}$, the BD moves to the next iteration by recalculating the MP and obtaining updated $\overline{\mathbf{z}}_{\mathcal{Y}_l}$ and $\overline{w}_l$ $(l = 1, ..., M)$. As \textbf{Stage 1} and \textbf{Stage 2} are repeated over iterations, the MP collects more cuts from the subproblems, 
converging the solution $\overline{\mathbf{z}}_{\mathcal{Y}_l}$ and $\overline{w}_l$ to the optimal.

\begin{proposition}
\label{prop:BDbound}
(\emph{Upper and lower bounds of PMO's optimal}) \cite{Rahmaniani-EJOR2017}
\newline (1) The optimal solution of the MP (Equ. (\ref{eq:MPObj}) -- (\ref{eq:MPzy0})) offers a \textbf{lower bound} of the optimal solution of the original PMO (Equ. (\ref{eq:LPobjective})--(\ref{eq:LPconstraint1})) (as the MP relaxes the constraints of PMO). 
\newline (2) The solution of the subproblems (Equ. (\ref{eq:dualobj})-(\ref{eq:dualconstr1})), if it exists, combined with the solution $\overline{\mathbf{z}}_{\mathcal{Y}_l}$ of the MP, provides an \textbf{upper bound} of the PMO's optimal (since their solutions form a feasible solution of the original PMO).
\end{proposition}

\DEL{
\begin{wrapfigure}{r}{0.220\textwidth}
\vspace{-0.15in}
\begin{minipage}{0.220\textwidth}
\centering
    \subfigure{
\includegraphics[width=1.00\textwidth, height = 0.13\textheight]{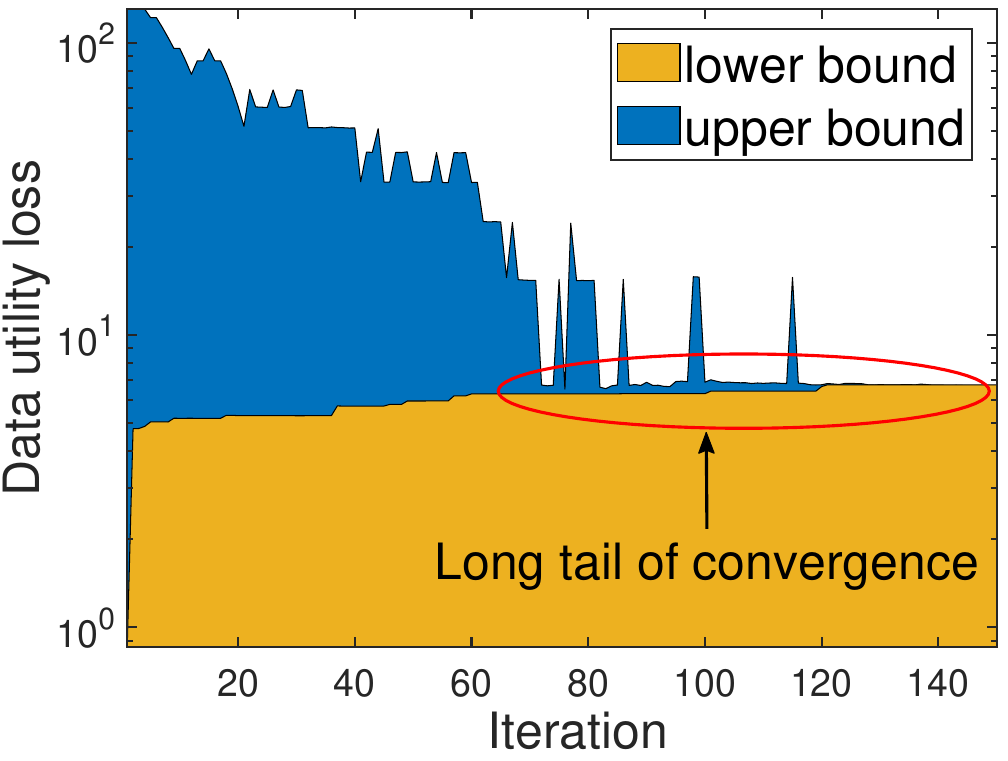}}
\vspace{-0.28in}
\caption{An example of BD's convergence.}
\vspace{-0.15in}
\label{fig:BDconvergenceEg}
\end{minipage}
\end{wrapfigure}}
Note that the optimal solution for PMO must lie within the gap between the upper bound and the lower bound in \textbf{Proposition \ref{prop:BDbound}}. The smaller the gap between these two bounds, the closer the solution of the BD is approaching the optimal solution of PMO. Fig. \ref{fig:convergence_rome_downtown}--Fig. \ref{fig:convergence_random} (in Section \ref{subsec:BDconvergence} in Appendix) gives some examples of how the upper and lower bounds change over iterations. 
Considering a prolonged convergence tail, we termine the algorithm when the gap between the best upper and lower bounds falls below a predetermined threshold $\xi$ (e.g., $\xi = 0.01$ in our experiments). \looseness = -1

\DEL{
\begin{figure}[t]
\centering
\hspace{0.00in}
\begin{minipage}{0.47\textwidth}
  \subfigure{
\includegraphics[width=1.00\textwidth, height = 0.13\textheight]{./fig/convergence/convergence_rome_d_kmean_dist_example}}
\end{minipage}
\vspace{-0.10in}
\caption{An example of the BD's convergence.}
\label{fig:BDconvergenceEg}
\vspace{-0.10in}
\end{figure}}

\vspace{-0.02in}
\subsection{Time Complexity Analysis}
\label{subsec:timecomplexity}
\vspace{-0.01in}
In the two-stage framework of BD, the number of decision variables in the MP in \textbf{Stage 1} is $K(\sum_{l=1}^M|\mathcal{Y}_l|)+M$, including $K(\sum_{l=1}^M|\mathcal{Y}_l|)$ variables ($\mathbf{z}_{\mathcal{Y}1}, \ldots, \mathbf{z}_{\mathcal{Y}_M}$) to determine the perturbation vectors of the boundary records, and the $M$ variables ($w_1,..., w_M$) to ``guess'' the utility loss of subproblems. Each $\mathrm{Sub}_l$ in \textbf{Stage 2} has $K |\mathcal{X}_l|$ decision variables to determine the perturbation vectors of the internal records $\mathbf{z}_{\mathcal{X}_l}$. \looseness = -1

The MP and the subproblems are both formulated as LP problems, with time complexity depending on the number of decision variables \cite{Cohen-STOC2019}. To facilitate analysis, we use a monotonically increasing function $O(T(n))$ to represent the time complexity of LP given the number of decision variables $n$. In each iteration, the MP in Stage 1 has $O(T(K \sum_{l=1}^M|\mathcal{Y}_l|))$ time complexity, and each subproblem $l$ in Stage 2 (run in parallel) has $O(T(K |\mathcal{X}_l|))$ time complexity. Since Stage 2 is terminated only after all subproblems are completed, its time complexity is $O(\max_{l} T(K |\mathcal{X}_l|)) = O(T( K \max_{l}|\mathcal{X}_l|))$ assuming all the subproblems are run in parallel. Hence, the total computation time of each iteration is given by $O(T(K \sum_{l=1}^M|\mathcal{Y}_l|) + T(K \max_{l} |\mathcal{X}_l|))$. Considering that both $\sum_{l=1}^M|\mathcal{Y}_l|$ and $\max_{l} |\mathcal{X}_l|$ is much smaller than the total number of records $|\mathcal{R}|$, the time complexity of each iteration of BD is significantly lower than that of the original PMO. There are two questions remain: 

(1) \emph{How many iterations are needed to converge the solution to the optimal?} 
In theory, if $|\cup_{l=1}^M \mathcal{Y}_l|$ is finite, generalized BD (including LP formulations) ends within a finite number of iterations for any given $\xi > 0$ \cite{Floudas2009}. To assess practical applicability, we will examine the convergence of our specific BD framework in \textbf{Section \ref{subsec:expEfficiency}} using different partitioning algorithms based on multiple datasets. \looseness = -1

(2) \emph{How to optimize the dataset partitioning to minimize the computation time $O(T(K \sum_{l=1}^M|\mathcal{Y}_l|) + T(K \max_{l}|\mathcal{X}_l|))$ in each iteration?} We introduce the detailed methods to address this problem in \textbf{Section \ref{subsec:cluster}}. 

\vspace{-0.04in}
\subsection{Secret Dataset Partitioning}
\label{subsec:cluster}
\vspace{-0.02in}
According to the time complexity analysis of the BD framework in Section \ref{subsec:timecomplexity}, we outline two primary objectives O-1 and O-2 for enhancing its computational efficiency when partitioning the secret dataset: 
\vspace{-0.00in}
\newline (O-1) \emph{Maintain strong mDP constraints within each subset and weak mDP constraints across different subsets}. This not only facilitates relatively independent computation among subproblems but also reduces the number of decision variables (boundary records) in the MP. 
\vspace{-0.00in}
\newline (O-2) \emph{Balance the number of records in different subsets}, to decrease the maximum time complexity of the subproblems. 

\vspace{-0.00in}
Achieving the above two objectives is similar to solving the \emph{RatioCut} problem \cite{Hagen-TCADICS1992}, which seeks to partition a graph with minimal cuts while ensuring a well-balanced size of subgraphs, a problem known to be NP-hard. 

Considering the computational trajectability, we apply a \emph{Distance-Vector (DV)}-based dataset partitioning algorithm. As the mDP constraints of two perturbation vectors depend on their records' distance, we embed each record $r_i$ by the \emph{distance vector} $\mathbf{d}_i = [d_{i,1}, ..., d_{i,N}]$ to characterize its mDP relationship with all other records in the mDP graph. We consider that two records, $r_i$ and $r_j$, exhibit a stronger connection when the Euclidean distance between their distance vectors, $\|\mathbf{d}_i - \mathbf{d}_j\|_2 = \sqrt{\sum_{l = 1}^N \left(d_{i,l} - d_{j,l}\right)^2}$, is lower, meaning that the two records are strongly coupled by the mDP constraints and also share similar mDP constraint relationship with other records. Accordingly, we partition the dataset to the subsets $\mathcal{R}_1, \ldots, \mathcal{R}_M$ using the distance vectors, which can be cast as the following \emph{k-means clustering} formulation:
\vspace{-0.03in}
\begin{equation}
\label{eq:kmean}
\textstyle \small \min ~ \sum_{l=1}^M \underbrace{\textstyle \sum_{r_i\in \mathcal{R}_l} \|\mathbf{d}_i-\boldsymbol{\mu}_l\|_2}_{\mbox{Reflect the connection within $\mathcal{R}_l$}}
\end{equation}
where $\boldsymbol{\mu}_l$ represents the \emph{centroid} of the distance vectors in $\mathcal{R}_l$, i.e., $\boldsymbol{\mu}_l = \frac{\sum_{r_i \in \mathcal{R}_l}\mathbf{d}_i}{|\mathcal{R}_l|}$ ($l = 1, \ldots, M$). Note that the objective function in Equ. (\ref{eq:kmean}) attains a lower value when the records within each partition are strongly connected and the sizes of $\mathcal{R}_1, ..., \mathcal{R}_M$ are well-balanced.

Besides the DV-based partitioning algorithm (labeled as ``\emph{k-mean-DV}'' or ``\emph{k-m-DV}''), for comparison, we carried out three other benchmarks, \emph{record-based partitioning}, \emph{adjacency matrix-based partitioning}, and \emph{balanced spectral clustering} \cite{Hagen1992}, labeled as ``\emph{k-m-rec}'', ``\emph{k-m-adj}'', and ``\emph{BSC}'', respectively. The details of the benchmarks are introduced in Section \ref{sec:benchmarks} in the supplementary file.


\vspace{-0.10in}
\section{Performance Evaluation}
\vspace{-0.03in}
\label{sec:performance}
In this section, we test the performance of our new computation approach using multiple datasets. We first introduce the experiment settings in \textbf{Section \ref{subsec:settings}}, and then evaluate the performance of our method in terms of partitioning balance in \textbf{Section \ref{subsec:expBalance}} and computation efficiency in \textbf{Section \ref{subsec:expEfficiency}}. More comprehensive experimental results (including the comparison with exponential mechanisms and the performance evaluation with different parameters, etc.) can be found in \textbf{Section \ref{sec:addExp}} in the supplementary file\footnote{The MATLAB source code of our method is available at: \url{https://github.com/chenxiunt/MetricDP_BendersDecomposition}}.

\vspace{-0.07in}
\subsection{Settings}
\label{subsec:settings}
\vspace{-0.03in}
\noindent \textbf{Datasets}. As mDP has been primarily used in text embeddings and geo-location data \cite{ImolaUAI2022}, we specifically choose text and geo-location datasets for the performance evaluation. Additionally, we assess our methods on a synthetic dataset. \looseness = -1
\newline \textbf{(1) Geo-location dataset in the road network}, composed of a set of ``nodes'' in the road network (retrieved by OpenStreetMap \cite{openstreetmap}), including road intersections, forks, junctions where roads intersect with others, and points where the road changes direction. We selected the city \emph{Rome, Italy} as the target region (specifically, the bounding area with coordinate $(lat=41.66, lon=12.24)$ as the south-west corner, and coordinate $(lat=42.10, lon=12.81)$ as the north-east corner). The distance metric between locations is defined by \emph{Haversine distance}, i.e.,  the angular distance between two points on the surface of a sphere. 

\noindent \textbf{(2) Geo-location dataset in 10 grid maps}, where each grid map is composed of 20$\times$25 grid cells (the size of each cell is 1km$\times$1km) in the target region, Rome, Italy. The center of each grid cell serves as a proxy for the cell's location. 
The distance metric between cells is defined as the \emph{Euclidean distance} between their centers. In both datasets (1) and (2), we set $\eta = 2.00 km$ and $\epsilon = 10.0/km$ by default. 

\noindent \textbf{(3) Text dataset}, which comprises 2,000 words. Each word is represented by a 300-dimensional vector using MATLAB's \texttt{word2vec} function \cite{word2vec} (pre-trained for 1 million English words), capturing both the semantic meaning of the word and its contextual usage. The distance between words is measured by the \emph{Euclidean distance} between their vectors. 

\begin{wrapfigure}{r}{0.24\textwidth}
\vspace{-0.07in}
\begin{minipage}{0.24\textwidth}
\centering
    \subfigure{
\includegraphics[width=1.00\textwidth]{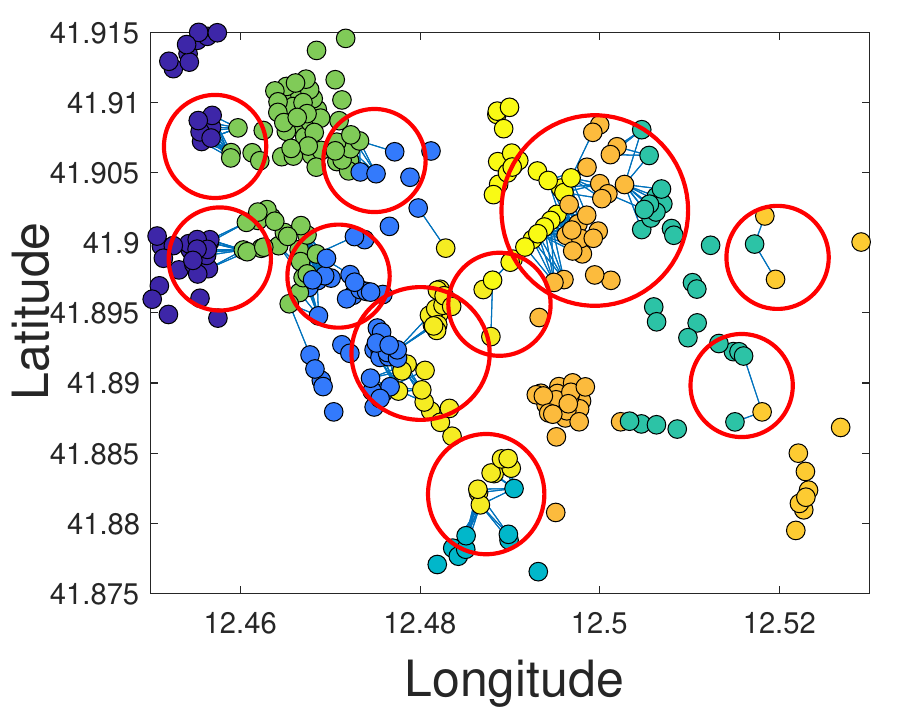}}
\vspace{-0.25in}
\caption{Example of MP components (each red circle highlights an MP's component).}
\vspace{-0.14in}
\label{fig:MPcomponents}
\end{minipage}
\end{wrapfigure}
\noindent \textbf{(4) Synthetic dataset}, which comprises 2,000 3-dimensional vectors, following the multivariate Gaussian distribution. The distance metric is defined as vectors' \emph{Euclidean distance}. In both datasets (3) and (4), we set $\eta = 2.0$ and $\epsilon = 10.0$ by default.

\vspace{0.03in}

\noindent \textbf{Metrics}: \textbf{(1) Problem sizes} of both subproblems and MPs. We measure the size of each decomposed problem by counting the number of perturbation vectors (each corresponds to a record) to derive. 
This metric reflects how effectively the dataset partitioning algorithms achieve balance in dividing the secret dataset. Note that the mDP graph of the MP may consist of several independent components. Fig. \ref{fig:MPcomponents} gives an example using a geo-location dataset in the road network, each solvable in isolation (as per \textbf{Property \ref{property:mDPcompo}}). Therefore, we measure the size of each component in the MP rather than the entire MP.

\noindent \textbf{(2) Computation time} to execute algorithms. The experiments are performed by a desktop with 13th Gen Intel Core i7 processor, 16 cores.

\noindent \textbf{(3) Number of iterations for the BD to converge} to the optiml solution. Considering the prolonged convergence tails of BD, we end the algorithm when the gap between its upper and lower bounds is lower than the threshold $\xi = 0.01$. 


\vspace{-0.07in}
\subsection{Secret Dataset Partitioning}
\label{subsec:expBalance}
\vspace{-0.03in}

\begin{table}[t]
\vspace{-0.00in}
\caption{Size of suproblems. Mean$\pm$1.96$\times$ standard deviation. }
\vspace{-0.1in}
\label{Tb:exp:partitionSub}
\centering
\footnotesize 
\begin{tabular}{ c|c|c|c|c}
\hline
\hline
\multicolumn{1}{ c  }{}
&\multicolumn{4}{ c }{Average size}
\\
\cline{2-5}
\multicolumn{1}{ c|  }{Datasets}
&\multicolumn{1}{ |c| }{\textbf{k-m-DV}}
 & \multicolumn{1}{ |c| }{k-m-rec}&\multicolumn{1}{ |c }{k-m-adj}&\multicolumn{1}{ |c }{BSC}
 \\ 
\hline
\hline
\multicolumn{1}{ c|  }{Road} & \textbf{9.5$\pm$1.2} & 10.1$\pm$0.8 & 17.7$\pm$1.3 & 15.4$\pm$0.8 \\ 
\multicolumn{1}{ c|  }{Grid} & \textbf{4.7$\pm$3.0} & 3.7$\pm$6.9 & 14.9$\pm$1.7 &3.0$\pm$4.7\\ 
\multicolumn{1}{ c|  }{Text} & \textbf{12.6$\pm$2.6} & 4.9$\pm$6.9 & 44.7$\pm$13.4 &33.0$\pm$16.3\\  
\multicolumn{1}{ c|  }{Syn} & \textbf{9.6$\pm$1.9} & 10.2$\pm$1.9 & 8.5$\pm$1.2 &8.4$\pm$5.2\\ 
\hline
\hline
\multicolumn{1}{ c  }{}
&\multicolumn{4}{ c }{Maximum size}
\\
\cline{2-5}
\multicolumn{1}{ c|  }{Datasets}
&\multicolumn{1}{ |c| }{\textbf{k-m-DV}}
 & \multicolumn{1}{ |c| }{k-m-rec}&\multicolumn{1}{ |c }{k-m-adj}&\multicolumn{1}{ |c }{BSC}
 \\ 
\hline
\hline
\multicolumn{1}{ c|  }{Road} & \textbf{18.6$\pm$8.9} & 16.1$\pm$10.1 & 34.6$\pm$12.5 & 32.1$\pm$15.3 \\ 
\multicolumn{1}{ c|  }{Grid} & \textbf{11.3$\pm$5.5} & 10.2$\pm$7.6 & 28.3$\pm$11.0 & 17.2$\pm$10.3 \\ 
\multicolumn{1}{ c|  }{Text} & \textbf{31.8$\pm$6.4} & 42.7$\pm$16.4 & 103.2$\pm$30.4 & 62.4$\pm$38.4 \\ 
\multicolumn{1}{ c|  }{Syn} & \textbf{17.9$\pm$2.9} & 19.8$\pm$5.1 & 90.0$\pm$24.6 & 90.2$\pm$56.6 \\ 
\hline
\end{tabular}
\vspace{-0.10in}
\end{table}
\begin{table}[t]
\caption{Size of MP components. Mean$\pm$1.96$\times$ standard deviation. }
\vspace{-0.1in}
\label{Tb:exp:partitionMP}
\centering
\footnotesize 
\begin{tabular}{ c|c|c|c|c}
\hline
\hline
\multicolumn{1}{ c  }{}
&\multicolumn{4}{ c }{Average size}
\\
\cline{2-5}
\multicolumn{1}{ c|  }{Datasets}
&\multicolumn{1}{ |c| }{\textbf{k-m-DV}}
 & \multicolumn{1}{ |c| }{k-m-rec}&\multicolumn{1}{ |c }{k-m-adj}&\multicolumn{1}{ |c }{BSC}
 \\ 
\hline
\hline
\multicolumn{1}{ c|  }{Road} & \textbf{11.3$\pm$6.0} & 37.9$\pm$24.7 & 40.4$\pm$34.3 & 64.7$\pm$14.5 \\ 
\multicolumn{1}{ c|  }{Grid} & \textbf{3.6$\pm$2.9} & 8.8$\pm$4.3 & 9.0$\pm$60.9 & 18.6$\pm$3.3 \\ 
\multicolumn{1}{ c|  }{Text} & \textbf{8.0$\pm$4.4} & 9.7$\pm$9.2 & 8.8$\pm$84.2 & 11.4$\pm$16.2 \\ 
\multicolumn{1}{ c|  }{Syn} & \textbf{10.9$\pm$4.3} & 12.4$\pm$8.7 & 17.9$\pm$9.0 & 18.3$\pm$9.2 \\ 
\hline
\hline
\multicolumn{1}{ c  }{}
&\multicolumn{4}{ c }{Maximum size}
\\
\cline{2-5}
\multicolumn{1}{ c|  }{Datasets}
&\multicolumn{1}{ |c| }{\textbf{k-m-DV}}
 & \multicolumn{1}{ |c| }{k-m-rec}&\multicolumn{1}{ |c }{k-m-adj}&\multicolumn{1}{ |c }{BSC}
 \\ 
\hline
\hline
\multicolumn{1}{ c|  }{Road} & \textbf{42.5$\pm$46.8} & 98.75$\pm$44.4 & 117.8$\pm$59.2 & 121.1$\pm$79.1 \\ 
\multicolumn{1}{ c|  }{Grid} & \textbf{22.1$\pm$3.4} & 10.3$\pm$5.1 & 99.5$\pm$95.1&27.1$\pm$4.2\\ 
\multicolumn{1}{ c|  }{Text} & \textbf{36.0$\pm$15.4} & 53.2$\pm$25.4&112.7$\pm$85.4 &79.2$\pm$34.2\\  
\multicolumn{1}{ c|  }{Syn} & \textbf{35.9$\pm$45.5} & 46.3$\pm$48.5 & 41.3$\pm$32.4 &63.6$\pm$41.0\\ 
\hline
\end{tabular}
\vspace{-0.18in}
\end{table}

Table \ref{Tb:exp:partitionSub} and Table \ref{Tb:exp:partitionMP} compare the sizes of subproblems and MP's components generated by the dataset partitioning algorithms, ``k-m-DV'', ``k-m-rec'', ``k-m-adj'', and ``BSC''. Each method undergoes testing on 10 instances for every dataset. The tables present the mean value$\pm$1.96$\times$standard deviation of 
the average/maximum size of subproblems and MP components across different instances. 
The detailed visual representation of the dataset partitioning in different datasets is given in Fig. \ref{fig:cluster_rome_downtown}--Fig. \ref{fig:cluster_random} in the supplementary file.

\vspace{-0.00in}
\begin{table*}[t]
\caption{Computation time (seconds). Mean$\pm$1.96$\times$ standard deviation. ``max iter. ex.'' means ``maximum iterations exceeded''}
\vspace{-0.10in}
\label{Tb:exp:time}
\centering
\footnotesize 
\begin{tabular}{ c|c|c|c|c|c|c}
\hline
\hline
\multicolumn{1}{ c  }{}
&\multicolumn{6}{ c }{LP-based methods}
\\
\cline{2-7}
\multicolumn{1}{ c  |}{} &\multicolumn{4}{ c |}{Benders decomposition} &\multicolumn{2}{ c }{Classic LP}
\\
\cline{2-7}
\multicolumn{1}{ c|  }{Datasets}
&\multicolumn{1}{ |c| }{\textbf{k-m-DV}}
 & \multicolumn{1}{ |c| }{k-m-rec}&\multicolumn{1}{ |c }{k-m-adj}&\multicolumn{1}{ |c| }{BSC}&\multicolumn{1}{ |c| }{Dual-simplex}&\multicolumn{1}{ |c }{Interior-point}
 \\ 
\hline
\hline
\multicolumn{1}{ c|  }{Geo-location dataset (road network)} & \textbf{290.7$\pm$33.5} & 1066.9$\pm$154.4 & max iter. ex. & max iter. ex. & max iter. ex. & max iter. ex. \\ 
\multicolumn{1}{ c|  }{Geo-location dataset (grid maps)} & \textbf{161.4$\pm$14.3} & 159.8$\pm$4.8 & max iter. ex. & 238.2$\pm$6.0 & max iter. ex. & max iter. ex. \\ 
\multicolumn{1}{ c|  }{Text dataset} & \textbf{27.2$\pm$0.3} & 925.0$\pm$79.0 & 58.7$\pm$9.7 &  1218.6$\pm$78.3 & max iter. ex. & max iter. ex. \\ 
\multicolumn{1}{ c|  }{Synthetic dataset} & \textbf{ 185.1$\pm$23.9} & 162.7$\pm$12.0 & max iter. ex. & 434.1$\pm$21.1 & max iter. ex. & max iter. ex. \\
\hline
\end{tabular}
\vspace{-0.10in}
\end{table*}

\begin{table}[t]
\vspace{-0.05in}
\caption{Number of iterations to converge to end the algorithm. \\
Mean$\pm$1.96$\times$ standard deviation.}
\vspace{-0.05in}
\label{Tb:exp:convergency}
\centering
\footnotesize 
\begin{tabular}{ c|c|c|c|c}
\hline
\hline
\multicolumn{1}{ c  }{}
&\multicolumn{4}{ c }{Dataset partitioning algorithms}
\\
\cline{2-5}
\multicolumn{1}{ c|  }{Datasets}
&\multicolumn{1}{ |c| }{\textbf{k-m-DV}}
 & \multicolumn{1}{ |c| }{k-m-rec}&\multicolumn{1}{ |c }{k-m-adj}&\multicolumn{1}{ |c }{BSC}
 \\ 
\hline
\hline
\multicolumn{1}{ c|  }{Road} & \textbf{128$\pm$10} & 180$\pm$17& $\geq$ 500 & $\geq$ 500 \\ 
\multicolumn{1}{ c|  }{Grid} & \textbf{143$\pm$4} & 130$\pm$6 & $\geq$ 500 & 185$\pm$7 \\ 
\multicolumn{1}{ c|  }{Text} & \textbf{41$\pm$6} & 204$\pm$19 & 42$\pm$5 & 154$\pm$14 \\ 
\multicolumn{1}{ c|  }{Syn} & \textbf{142$\pm$11} & 144$\pm$15 & $\geq$ 500 & 237$\pm$19 \\ 
\hline
\end{tabular}
\vspace{-0.05in}
\end{table}

The tables show that, in all four datasets, k-m-DV partitions the PMO (containing 500 secret records) into sufficiently small subproblems and MP components, with the mean maximum size not exceeding 50. This decomposition facilitates efficient resolution of the individual problems using classical LP solvers such as the \texttt{dual-simplex} method \cite{Linear&Nonlinear}.

In contrast to k-m-DV, the size variation in the decomposed problems is higher for the other three methods. Specifically, the maximum problem size of the decomposed problems for k-m-adj and BSC is higher than 100, indicating a challenge in solving these larger problems efficiently. Both k-m-adj and BSC are less effective in balancing the size of decomposed problems, as they focus solely on the mDP constraint between neighboring records, without considering the overall constraint distribution of all the other pairs. While k-m-rec demonstrates comparable performance to k-m-DV when applied to the geo-location data in the road network/grid maps, and synthetic data, it falls short in achieving balanced partitioning for text data. This discrepancy arises from its limited capability of capturing the mDP constraints between records 
when the data is represented in a relatively high-dimensional (300-dimensional) space.

\vspace{-0.03in}
\subsection{Computational Efficiency}
\label{subsec:expEfficiency}
\vspace{-0.00in}

Table \ref{Tb:exp:time} and Table \ref{Tb:exp:convergency} compare the computation time of the BD framework using the four dataset partitioning algorithms, ``k-m-DV'', ``k-m-rec'', ``k-m-adj'', and ``BSC'', and two classic LP solvers, \texttt{dual-simplex} and \texttt{interior-point}, both of which are provided by the MATLAB LP toolbox \texttt{linprog} \cite{MATLABlinprog}. A detailed comparison of the BD convergence of the four partitioning algorithms using the four datasets is shown in Fig. \ref{fig:convergence_rome_downtown}--Fig. \ref{fig:convergence_random} in the supplementary file.

As shown in Table \ref{Tb:exp:time}, both classic LP methods stop prematurely without achieving the optimal solution as they exceed the maximum number of iterations (the parameter \texttt{MaxIter} is set by 10,000). Within the BD framework, on average, k-m-DV achieves the lowest computation time. This efficiency can be attributed to the balanced sizes of decomposed subproblems, as shown in Table \ref{Tb:exp:partitionSub} and Table \ref{Tb:exp:partitionMP}. As discussed in Section \ref{subsec:timecomplexity}, the computation time of BD in each iteration is significantly influenced by the computation time of the largest subproblem, which is relatively smaller when the size variation of decomposed problems is minimized. Additionally, a smaller subproblem, on average, needs fewer iterations to converge to a feasible solution within the BD framework (refer to Fig. \ref{fig:subproblemvsconvergence1} in the supplementary file, which demonstrates a positive correlation between the subproblem size and the number of iterations to find a feasible solution). This relationship is further supported by Table \ref{Tb:exp:convergency}, where it is evident that k-m-DV exhibits the fewest iterations to reach convergence, contributing to its lower computation time.

\DEL{
\begin{wrapfigure}{r}{0.24\textwidth}
\vspace{-0.10in}
\begin{minipage}{0.24\textwidth}
\centering
    \subfigure{
\includegraphics[width=1.00\textwidth, height = 0.13\textheight]{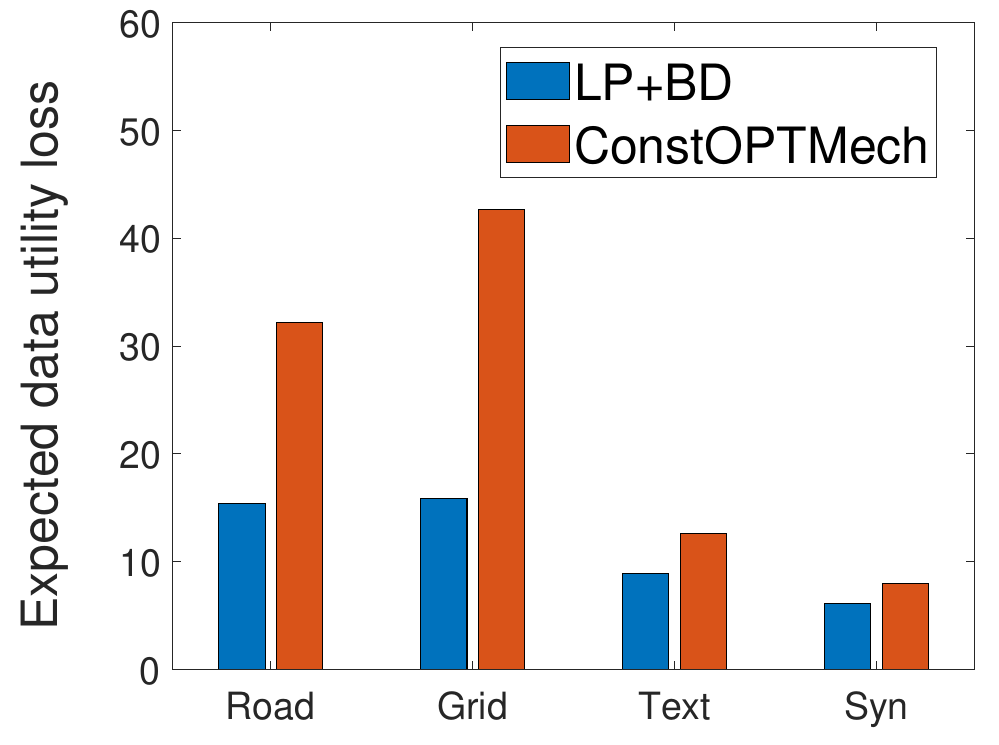}}
\vspace{-0.2in}
\caption{Data utility loss.}
\vspace{-0.05in}
\label{fig:scalability}
\end{minipage}
\end{wrapfigure}}

\begin{figure}[t]
\centering
\begin{minipage}{0.5\textwidth}
\centering
  \subfigure[\footnotesize Average computation time]{
\includegraphics[width=0.45\textwidth, height = 0.13\textheight]{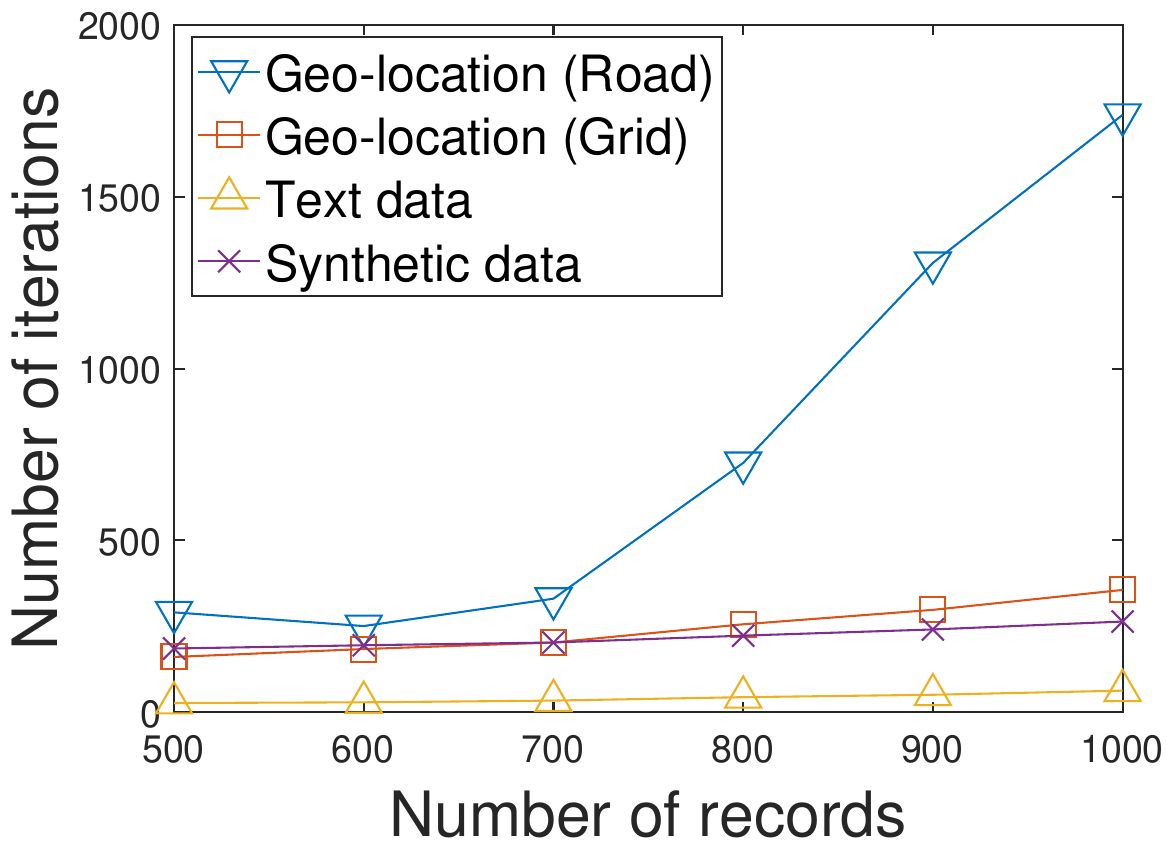}}
\label{}
\centering
  \subfigure[\footnotesize Avg. number of iterations]{
\includegraphics[width=0.45\textwidth, height = 0.13\textheight]{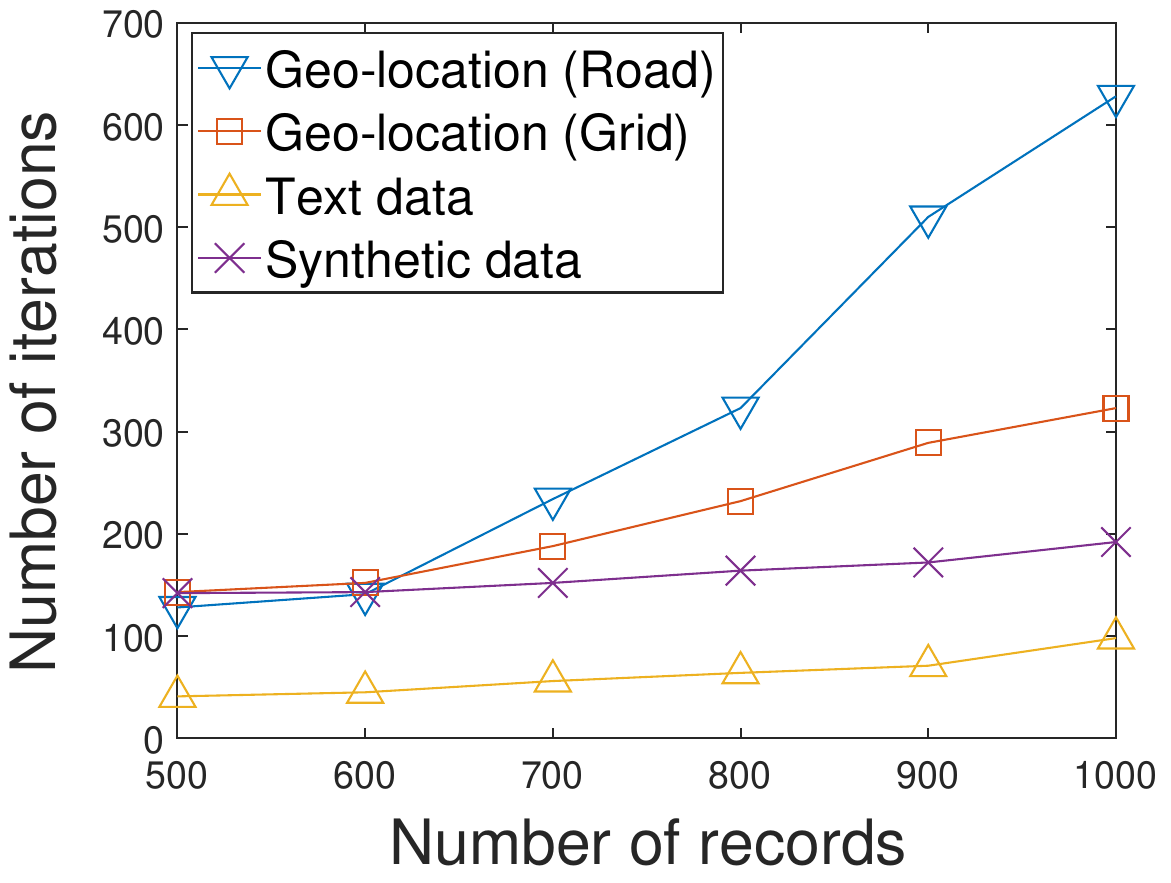}}
\label{}
\end{minipage}
\vspace{-0.12in}
\caption{Computation efficiency of k-m-DV in larger scale datasets.}
\label{fig:scalability}
\vspace{-0.12in}
\end{figure}
Finally, we expand the size of the secret dataset from 500 records to 1,000 records and assess the computational efficiency and number of iterations of ``k-m-DV'' in Fig. \ref{fig:scalability}(a)(b), respectively. Notably, the computation time for ``k-m-DV'' remains below 6 minutes when using grid maps, text data, and synthetic datasets. However, when using the road map location dataset, the computation time experiences a significant increase, reaching approximately 30 minutes when the dataset size is 1,000. This difference can be attributed to the uneven distribution of nodes in the large-scale road network, with low node density in suburban areas and high node density in downtown areas. This leads to some larger decomposed problems, resulting in a longer time to solve.

\vspace{-0.08in}
\section{Conclusion}
\vspace{-0.01in}

In this paper, we proposed to improve the scalability of LP-based mDP through secret dataset partitioning and BD. The experimental results, employing diverse datasets, demonstrate that the new computation approach expands the scale of the existing LP-based mDP by approximately 9 times. Additionally, our findings emphasize the \emph{significance of maintaining a balanced size between the decomposed subproblems and MP components when implementing BD}.

We identify several promising directions for further improving the proposed methods. First, while BD improves the scalability of metric-DP, its current convergence (which takes up to 650 iterations in our experiment) poses challenges for time-sensitive applications. A potential solution is to leverage \emph{reinforcement learning (RL)} to accelerate the convergence speed of BD by considering the cut selection of each subproblem in Stage 2 as a parameterized stochastic policy. 
The trained RL model is expected to identify a sequence of cuts given the fixed coefficients of each subproblem, eliminating the need for re-training with each new problem instance.

Second, recognizing that when the size of secret datasets is higher (e.g., over 1,000), the decomposed MP might remain too large for current LP solvers to handle. Therefore, a potential direction is to explore a combination of multiple decomposition techniques, such as Danzig-Wolfe decomposition and multi-layer BD decomposition, based on the mDP constraint features. 

\section{Acknowledgments}
This research is partially supported by U.S. NSF grants CNS2136948 and CNS-2313866.




\begin{thebibliography}{35}


\ifx \showCODEN    \undefined \def \showCODEN     #1{\unskip}     \fi
\ifx \showDOI      \undefined \def \showDOI       #1{#1}\fi
\ifx \showISBNx    \undefined \def \showISBNx     #1{\unskip}     \fi
\ifx \showISBNxiii \undefined \def \showISBNxiii  #1{\unskip}     \fi
\ifx \showISSN     \undefined \def \showISSN      #1{\unskip}     \fi
\ifx \showLCCN     \undefined \def \showLCCN      #1{\unskip}     \fi
\ifx \shownote     \undefined \def \shownote      #1{#1}          \fi
\ifx \showarticletitle \undefined \def \showarticletitle #1{#1}   \fi
\ifx \showURL      \undefined \def \showURL       {\relax}        \fi
\providecommand\bibfield[2]{#2}
\providecommand\bibinfo[2]{#2}
\providecommand\natexlab[1]{#1}
\providecommand\showeprint[2][]{arXiv:#2}

\bibitem[ope(2020)]%
        {openstreetmap}
 \bibinfo{year}{2020}\natexlab{}.
\newblock \bibinfo{title}{openstreetmap}.
\newblock \bibinfo{howpublished}{\url{https://www.openstreetmap.org/}}.
\newblock
\newblock
\shownote{Accessed: 2020-04-07}.


\bibitem[MAT(2024)]%
        {MATLABlinprog}
 \bibinfo{year}{2024}\natexlab{}.
\newblock \bibinfo{title}{{linprog: Solve linear programming problems}}.
\newblock \bibinfo{howpublished}{\url{https://www.mathworks.com/help/optim/ug/linprog.html}}.
\newblock
\newblock
\shownote{Accessed in January 2024}.


\bibitem[wor(2024)]%
        {word2vec}
 \bibinfo{year}{2024}\natexlab{}.
\newblock \bibinfo{title}{{word2vec: Map word to embedding vector}}.
\newblock \bibinfo{howpublished}{\url{https://www.mathworks.com/help/textanalytics/ref/wordembedding.word2vec.html}}.
\newblock
\newblock
\shownote{Accessed in January 2024}.


\bibitem[Andr{\'e}s et~al\mbox{.}(2013)]%
        {Andres-CCS2013}
\bibfield{author}{\bibinfo{person}{M.~E. Andr{\'e}s}, \bibinfo{person}{N.~E. Bordenabe}, \bibinfo{person}{K. Chatzikokolakis}, {and} \bibinfo{person}{C. Palamidessi}.} \bibinfo{year}{2013}\natexlab{}.
\newblock \showarticletitle{Geo-indistinguishability: Differential Privacy for Location-based Systems}. In \bibinfo{booktitle}{\emph{Proc. of ACM CCS}}. \bibinfo{pages}{901--914}.
\newblock


\bibitem[Bordenabe et~al\mbox{.}(2014)]%
        {Bordenabe-CCS2014}
\bibfield{author}{\bibinfo{person}{N.~E. Bordenabe}, \bibinfo{person}{K. Chatzikokolakis}, {and} \bibinfo{person}{C. Palamidessi}.} \bibinfo{year}{2014}\natexlab{}.
\newblock \showarticletitle{Optimal Geo-Indistinguishable Mechanisms for Location Privacy}. In \bibinfo{booktitle}{\emph{Proc. of ACM CCS}}. \bibinfo{pages}{251--262}.
\newblock


\bibitem[Chatzikokolakis et~al\mbox{.}(2013)]%
        {Chatzikokolakis-PETS2013}
\bibfield{author}{\bibinfo{person}{Konstantinos Chatzikokolakis}, \bibinfo{person}{Miguel~E. Andr{\'e}s}, \bibinfo{person}{Nicol{\'a}s~Emilio Bordenabe}, {and} \bibinfo{person}{Catuscia Palamidessi}.} \bibinfo{year}{2013}\natexlab{}.
\newblock \showarticletitle{Broadening the Scope of Differential Privacy Using Metrics}. In \bibinfo{booktitle}{\emph{Proc. of Privacy Enhancing Technologies}}, \bibfield{editor}{\bibinfo{person}{Emiliano De~Cristofaro} {and} \bibinfo{person}{Matthew Wright}} (Eds.). \bibinfo{publisher}{Springer Berlin Heidelberg}, \bibinfo{address}{Berlin, Heidelberg}, \bibinfo{pages}{82--102}.
\newblock
\showISBNx{978-3-642-39077-7}


\bibitem[Chatzikokolakis et~al\mbox{.}(2015)]%
        {Chatzikokolakis-PoPETs2015}
\bibfield{author}{\bibinfo{person}{Konstantinos Chatzikokolakis}, \bibinfo{person}{Catuscia Palamidessi}, {and} \bibinfo{person}{Marco Stronati}.} \bibinfo{year}{2015}\natexlab{}.
\newblock \showarticletitle{Constructing elastic distinguishability metrics for location privacy}.
\newblock \bibinfo{journal}{\emph{Privacy Enhancing Technologies (PoPETs)}}  \bibinfo{volume}{2015} (\bibinfo{year}{2015}), \bibinfo{pages}{156{\textendash}170}.
\newblock
\urldef\tempurl%
\url{http://www.degruyter.com/view/j/popets.2015.2015.issue-2/popets-2015-0023/popets-2015-0023.xml}
\showURL{%
\tempurl}


\bibitem[Cohen et~al\mbox{.}(2019)]%
        {Cohen-STOC2019}
\bibfield{author}{\bibinfo{person}{Michael~B. Cohen}, \bibinfo{person}{Yin~Tat Lee}, {and} \bibinfo{person}{Zhao Song}.} \bibinfo{year}{2019}\natexlab{}.
\newblock \showarticletitle{Solving Linear Programs in the Current Matrix Multiplication Time}. In \bibinfo{booktitle}{\emph{Proceedings of the 51st Annual ACM SIGACT Symposium on Theory of Computing}} (Phoenix, AZ, USA) \emph{(\bibinfo{series}{STOC 2019})}. \bibinfo{publisher}{Association for Computing Machinery}, \bibinfo{address}{New York, NY, USA}, \bibinfo{pages}{938–942}.
\newblock
\showISBNx{9781450367059}
\urldef\tempurl%
\url{https://doi.org/10.1145/3313276.3316303}
\showDOI{\tempurl}


\bibitem[Dwork et~al\mbox{.}(2006)]%
        {Dwork-TC2006}
\bibfield{author}{\bibinfo{person}{C. Dwork}, \bibinfo{person}{}, \bibinfo{person}{F. McSherry}, \bibinfo{person}{K. Nissim}, {and} \bibinfo{person}{A. Smith}.} \bibinfo{year}{2006}\natexlab{}.
\newblock \showarticletitle{Calibrating Noise to Sensitivity in Private Data Analysis}. In \bibinfo{booktitle}{\emph{Theory of Cryptography}}. \bibinfo{publisher}{Springer Berlin Heidelberg}, \bibinfo{pages}{265--284}.
\newblock
\showISBNx{978-3-540-32732-5}


\bibitem[Fawaz and Shin(2014)]%
        {Fawaz-CCS2014}
\bibfield{author}{\bibinfo{person}{K. Fawaz} {and} \bibinfo{person}{K.~G. Shin}.} \bibinfo{year}{2014}\natexlab{}.
\newblock \showarticletitle{Location Privacy Protection for Smartphone Users}. In \bibinfo{booktitle}{\emph{Proc. of ACM CCS}} (Scottsdale, Arizona, USA). \bibinfo{publisher}{ACM}, \bibinfo{address}{New York, NY, USA}, \bibinfo{pages}{239--250}.
\newblock
\showISBNx{978-1-4503-2957-6}
\urldef\tempurl%
\url{https://doi.org/10.1145/2660267.2660270}
\showDOI{\tempurl}


\bibitem[Fernandes et~al\mbox{.}(2019)]%
        {Fernandes-PST2019}
\bibfield{author}{\bibinfo{person}{Natasha Fernandes}, \bibinfo{person}{Mark Dras}, {and} \bibinfo{person}{Annabelle McIver}.} \bibinfo{year}{2019}\natexlab{}.
\newblock \showarticletitle{Generalised Differential Privacy for Text Document Processing}. In \bibinfo{booktitle}{\emph{Proc. of Principles of Security and Trust}}, \bibfield{editor}{\bibinfo{person}{Flemming Nielson} {and} \bibinfo{person}{David Sands}} (Eds.). \bibinfo{publisher}{Springer International Publishing}, \bibinfo{address}{Cham}, \bibinfo{pages}{123--148}.
\newblock
\showISBNx{978-3-030-17138-4}


\bibitem[Feyisetan et~al\mbox{.}(2019)]%
        {Feyisetan-ICDM2019}
\bibfield{author}{\bibinfo{person}{O. Feyisetan}, \bibinfo{person}{T. Diethe}, {and} \bibinfo{person}{T. Drake}.} \bibinfo{year}{2019}\natexlab{}.
\newblock \showarticletitle{Leveraging Hierarchical Representations for Preserving Privacy and Utility in Text}. In \bibinfo{booktitle}{\emph{2019 IEEE International Conference on Data Mining (ICDM)}}. \bibinfo{publisher}{IEEE Computer Society}, \bibinfo{address}{Los Alamitos, CA, USA}, \bibinfo{pages}{210--219}.
\newblock
\urldef\tempurl%
\url{https://doi.org/10.1109/ICDM.2019.00031}
\showDOI{\tempurl}


\bibitem[Feyisetan and Kasiviswanathan(2021)]%
        {Feyisetan-2021-private}
\bibfield{author}{\bibinfo{person}{Oluwaseyi Feyisetan} {and} \bibinfo{person}{Shiva Kasiviswanathan}.} \bibinfo{year}{2021}\natexlab{}.
\newblock \showarticletitle{Private Release of Text Embedding Vectors}. In \bibinfo{booktitle}{\emph{Proc. of the First Workshop on Trustworthy Natural Language Processing}}. \bibinfo{publisher}{Association for Computational Linguistics}, \bibinfo{address}{Online}, \bibinfo{pages}{15--27}.
\newblock
\urldef\tempurl%
\url{https://doi.org/10.18653/v1/2021.trustnlp-1.3}
\showDOI{\tempurl}


\bibitem[Floudas(2009)]%
        {Floudas2009}
\bibfield{author}{\bibinfo{person}{Christodoulos~A. Floudas}.} \bibinfo{year}{2009}\natexlab{}.
\newblock \bibinfo{booktitle}{\emph{Generalized benders decompositionGeneralized Benders Decomposition}}.
\newblock \bibinfo{publisher}{Springer US}, \bibinfo{address}{Boston, MA}, \bibinfo{pages}{1162--1175}.
\newblock
\showISBNx{978-0-387-74759-0}
\urldef\tempurl%
\url{https://doi.org/10.1007/978-0-387-74759-0_201}
\showDOI{\tempurl}


\bibitem[Hagen and Kahng(1992a)]%
        {Hagen-TCADICS1992}
\bibfield{author}{\bibinfo{person}{L. Hagen} {and} \bibinfo{person}{A.B. Kahng}.} \bibinfo{year}{1992}\natexlab{a}.
\newblock \showarticletitle{New spectral methods for ratio cut partitioning and clustering}.
\newblock \bibinfo{journal}{\emph{IEEE Transactions on Computer-Aided Design of Integrated Circuits and Systems}} \bibinfo{volume}{11}, \bibinfo{number}{9} (\bibinfo{year}{1992}), \bibinfo{pages}{1074--1085}.
\newblock
\urldef\tempurl%
\url{https://doi.org/10.1109/43.159993}
\showDOI{\tempurl}


\bibitem[Hagen and Kahng(1992b)]%
        {Hagen1992}
\bibfield{author}{\bibinfo{person}{L. Hagen} {and} \bibinfo{person}{A.B. Kahng}.} \bibinfo{year}{1992}\natexlab{b}.
\newblock \showarticletitle{New spectral methods for ratio cut partitioning and clustering}.
\newblock \bibinfo{journal}{\emph{IEEE Transactions on Computer-Aided Design of Integrated Circuits and Systems}} \bibinfo{volume}{11}, \bibinfo{number}{9} (\bibinfo{year}{1992}), \bibinfo{pages}{1074--1085}.
\newblock
\urldef\tempurl%
\url{https://doi.org/10.1109/43.159993}
\showDOI{\tempurl}


\bibitem[Hillier(2008)]%
        {Linear&Nonlinear}
\bibfield{author}{\bibinfo{person}{Frederick~S. Hillier}.} \bibinfo{year}{2008}\natexlab{}.
\newblock \bibinfo{booktitle}{\emph{Linear and Nonlinear Programming}}.
\newblock \bibinfo{publisher}{Stanford University}.
\newblock


\bibitem[Imola et~al\mbox{.}(2022)]%
        {ImolaUAI2022}
\bibfield{author}{\bibinfo{person}{Jacob Imola}, \bibinfo{person}{Shiva Kasiviswanathan}, \bibinfo{person}{Stephen White}, \bibinfo{person}{Abhinav Aggarwal}, {and} \bibinfo{person}{Nathanael Teissier}.} \bibinfo{year}{2022}\natexlab{}.
\newblock \showarticletitle{Balancing utility and scalability in metric differential privacy}. In \bibinfo{booktitle}{\emph{Proc. of UAI 2022}}.
\newblock
\urldef\tempurl%
\url{https://www.amazon.science/publications/balancing-utility-and-scalability-in-metric-differential-privacy}
\showURL{%
\tempurl}


\bibitem[L{\"u}tkepohl(1996)]%
        {lutkepohl1996handbook}
\bibfield{author}{\bibinfo{person}{H. L{\"u}tkepohl}.} \bibinfo{year}{1996}\natexlab{}.
\newblock \bibinfo{booktitle}{\emph{Handbook of Matrices}}.
\newblock \bibinfo{publisher}{Springer}.
\newblock
\showISBNx{9789810227197}
\urldef\tempurl%
\url{https://books.google.com/books?id=sGbVPgAACAAJ}
\showURL{%
\tempurl}


\bibitem[McSherry and Talwar(2007)]%
        {McSherry-FOCS2007}
\bibfield{author}{\bibinfo{person}{Frank McSherry} {and} \bibinfo{person}{Kunal Talwar}.} \bibinfo{year}{2007}\natexlab{}.
\newblock \showarticletitle{Mechanism Design via Differential Privacy}. In \bibinfo{booktitle}{\emph{48th Annual IEEE Symposium on Foundations of Computer Science (FOCS'07)}}. \bibinfo{pages}{94--103}.
\newblock
\urldef\tempurl%
\url{https://doi.org/10.1109/FOCS.2007.66}
\showDOI{\tempurl}


\bibitem[Ng et~al\mbox{.}(2001)]%
        {Ng-NIPS2001}
\bibfield{author}{\bibinfo{person}{Andrew Ng}, \bibinfo{person}{Michael Jordan}, {and} \bibinfo{person}{Yair Weiss}.} \bibinfo{year}{2001}\natexlab{}.
\newblock \showarticletitle{On Spectral Clustering: Analysis and an algorithm}. In \bibinfo{booktitle}{\emph{Advances in Neural Information Processing Systems}}, \bibfield{editor}{\bibinfo{person}{T.~Dietterich}, \bibinfo{person}{S.~Becker}, {and} \bibinfo{person}{Z.~Ghahramani}} (Eds.), Vol.~\bibinfo{volume}{14}. \bibinfo{publisher}{MIT Press}.
\newblock
\urldef\tempurl%
\url{https://proceedings.neurips.cc/paper_files/paper/2001/file/801272ee79cfde7fa5960571fee36b9b-Paper.pdf}
\showURL{%
\tempurl}


\bibitem[Palomar and Chiang(2006)]%
        {Palomar-JSAC2006}
\bibfield{author}{\bibinfo{person}{D.~P. Palomar} {and} \bibinfo{person}{Mung Chiang}.} \bibinfo{year}{2006}\natexlab{}.
\newblock \showarticletitle{A tutorial on decomposition methods for network utility maximization}.
\newblock \bibinfo{journal}{\emph{IEEE Journal on Selected Areas in Communications}} \bibinfo{volume}{24}, \bibinfo{number}{8} (\bibinfo{date}{Aug} \bibinfo{year}{2006}), \bibinfo{pages}{1439--1451}.
\newblock
\showISSN{0733-8716}
\urldef\tempurl%
\url{https://doi.org/10.1109/JSAC.2006.879350}
\showDOI{\tempurl}


\bibitem[Pappachan et~al\mbox{.}(2023)]%
        {Pappachan-EDBT2023}
\bibfield{author}{\bibinfo{person}{P. Pappachan}, \bibinfo{person}{C. Qiu}, \bibinfo{person}{A. Squicciarini}, {and} \bibinfo{person}{V. Manjunath}.} \bibinfo{year}{2023}\natexlab{}.
\newblock \showarticletitle{User Customizable and Robust Geo-Indistinguishability for Location Privacy}. In \bibinfo{booktitle}{\emph{Proc. of International Conference on Extending Database Technology (EDBT)}}.
\newblock


\bibitem[{Qiu} and {Squicciarini}(2019)]%
        {Qiu-ICDCS2019}
\bibfield{author}{\bibinfo{person}{C. {Qiu}} {and} \bibinfo{person}{A.~C. {Squicciarini}}.} \bibinfo{year}{2019}\natexlab{}.
\newblock \showarticletitle{Location Privacy Protection in Vehicle-Based Spatial Crowdsourcing Via Geo-Indistinguishability}. In \bibinfo{booktitle}{\emph{Proc. of IEEE ICDCS}}. \bibinfo{pages}{1061--1071}.
\newblock


\bibitem[Qiu et~al\mbox{.}(2020)]%
        {Qiu-CIKM2020}
\bibfield{author}{\bibinfo{person}{C. Qiu}, \bibinfo{person}{A.~C. Squicciarini}, \bibinfo{person}{Z. Li}, \bibinfo{person}{C. Pang}, {and} \bibinfo{person}{L. Yan}.} \bibinfo{year}{2020}\natexlab{}.
\newblock \showarticletitle{Time-Efficient Geo-Obfuscation to Protect Worker Location Privacy over Road Networks in Spatial Crowdsourcing}. In \bibinfo{booktitle}{\emph{Proc. of ACM CIKM}}.
\newblock


\bibitem[{Qiu} et~al\mbox{.}(2022)]%
        {Qiu-TMC2022}
\bibfield{author}{\bibinfo{person}{C. {Qiu}}, \bibinfo{person}{A.~C. {Squicciarini}}, \bibinfo{person}{C. {Pang}}, \bibinfo{person}{N. {Wang}}, {and} \bibinfo{person}{B. {Wu}}.} \bibinfo{year}{2022}\natexlab{}.
\newblock \showarticletitle{Location Privacy Protection in Vehicle-Based Spatial Crowdsourcing via Geo-Indistinguishability}.
\newblock \bibinfo{journal}{\emph{IEEE Transactions on Mobile Computing}} (\bibinfo{year}{2022}), \bibinfo{pages}{1--1}.
\newblock
\urldef\tempurl%
\url{https://doi.org/10.1109/TMC.2020.3037911}
\showDOI{\tempurl}


\bibitem[Qiu et~al\mbox{.}(2022)]%
        {Qiu-SIGSPATIAL2022}
\bibfield{author}{\bibinfo{person}{C. Qiu}, \bibinfo{person}{L. Yan}, \bibinfo{person}{A. Squicciarini}, \bibinfo{person}{J. Zhao}, \bibinfo{person}{C. Xu}, {and} \bibinfo{person}{P. Pappachan}.} \bibinfo{year}{2022}\natexlab{}.
\newblock \showarticletitle{TrafficAdaptor: An Adaptive Obfuscation Strategy for Vehicle Location Privacy Against Vehicle Traffic Flow Aware Attacks}. In \bibinfo{booktitle}{\emph{Proc. of ACM SIGSPATIAL}}.
\newblock


\bibitem[Rahmaniani et~al\mbox{.}(2017)]%
        {Rahmaniani-EJOR2017}
\bibfield{author}{\bibinfo{person}{Ragheb Rahmaniani}, \bibinfo{person}{Teodor~Gabriel Crainic}, \bibinfo{person}{Michel Gendreau}, {and} \bibinfo{person}{Walter Rei}.} \bibinfo{year}{2017}\natexlab{}.
\newblock \showarticletitle{The Benders decomposition algorithm: A literature review}.
\newblock \bibinfo{journal}{\emph{European Journal of Operational Research}} \bibinfo{volume}{259}, \bibinfo{number}{3} (\bibinfo{year}{2017}), \bibinfo{pages}{801--817}.
\newblock
\showISSN{0377-2217}
\urldef\tempurl%
\url{https://doi.org/10.1016/j.ejor.2016.12.005}
\showDOI{\tempurl}


\bibitem[Shokri et~al\mbox{.}(2012)]%
        {Shokri-CCS2012}
\bibfield{author}{\bibinfo{person}{R. Shokri}, \bibinfo{person}{G. Theodorakopoulos}, \bibinfo{person}{C. Troncoso}, \bibinfo{person}{J. Hubaux}, {and} \bibinfo{person}{J.~L. Boudec}.} \bibinfo{year}{2012}\natexlab{}.
\newblock \showarticletitle{Protecting Location Privacy: Optimal Strategy Against Localization Attacks}. In \bibinfo{booktitle}{\emph{Proc. of ACM CCS}}. \bibinfo{pages}{617--627}.
\newblock


\bibitem[Stroock(2010)]%
        {probability}
\bibfield{author}{\bibinfo{person}{Daniel~W. Stroock}.} \bibinfo{year}{2010}\natexlab{}.
\newblock \bibinfo{booktitle}{\emph{Probability Theory: An Analytic View} (\bibinfo{edition}{2nd} ed.)}.
\newblock \bibinfo{publisher}{Cambridge University Press}.
\newblock
\showISBNx{0521132509}


\bibitem[von Luxburg(2007)]%
        {vonluxburg2007tutorial}
\bibfield{author}{\bibinfo{person}{Ulrike von Luxburg}.} \bibinfo{year}{2007}\natexlab{}.
\newblock \bibinfo{title}{A Tutorial on Spectral Clustering}.
\newblock
\newblock
\showeprint[arxiv]{0711.0189}~[cs.DS]


\bibitem[Wang et~al\mbox{.}(2017)]%
        {Wang-WWW2017}
\bibfield{author}{\bibinfo{person}{L. Wang}, \bibinfo{person}{D. Yang}, \bibinfo{person}{X. Han}, \bibinfo{person}{T. Wang}, \bibinfo{person}{D. Zhang}, {and} \bibinfo{person}{X. Ma}.} \bibinfo{year}{2017}\natexlab{}.
\newblock \showarticletitle{Location Privacy-Preserving Task Allocation for Mobile Crowdsensing with Differential Geo-Obfuscation}. In \bibinfo{booktitle}{\emph{Proc. of ACM WWW}}. \bibinfo{pages}{627--636}.
\newblock


\bibitem[Wang et~al\mbox{.}(2016)]%
        {Wang-CIDM2016}
\bibfield{author}{\bibinfo{person}{Leye Wang}, \bibinfo{person}{Daqing Zhang}, \bibinfo{person}{Dingqi Yang}, \bibinfo{person}{Brian~Y. Lim}, {and} \bibinfo{person}{Xiaojuan Ma}.} \bibinfo{year}{2016}\natexlab{}.
\newblock \showarticletitle{Differential Location Privacy for Sparse Mobile Crowdsensing}. In \bibinfo{booktitle}{\emph{2016 IEEE 16th International Conference on Data Mining (ICDM)}}. \bibinfo{pages}{1257--1262}.
\newblock
\urldef\tempurl%
\url{https://doi.org/10.1109/ICDM.2016.0169}
\showDOI{\tempurl}


\bibitem[Wilson(1986)]%
        {GraphTheory}
\bibfield{author}{\bibinfo{person}{Robin~J Wilson}.} \bibinfo{year}{1986}\natexlab{}.
\newblock \bibinfo{booktitle}{\emph{Introduction to Graph Theory}}.
\newblock \bibinfo{publisher}{John Wiley \& Sons, Inc.}, \bibinfo{address}{USA}.
\newblock
\showISBNx{0470206160}


\bibitem[Yu et~al\mbox{.}(2017)]%
        {Yu-NDSS2017}
\bibfield{author}{\bibinfo{person}{L. Yu}, \bibinfo{person}{L. Liu}, {and} \bibinfo{person}{C. Pu}.} \bibinfo{year}{2017}\natexlab{}.
\newblock \showarticletitle{Dynamic Differential Location Privacy with Personalized Error Bounds}. In \bibinfo{booktitle}{\emph{Proc. of IEEE NDSS}}.
\newblock


\end{thebibliography}


\clearpage 
\appendix

\DEL{
\begin{eqnarray}
\min && \sum_{l=1}^N \mathbf{c}_{\mathcal{Y}_l} \mathbf{z}_{\mathcal{Y}_l} + \sum_{l=1}^N \mathbf{c}_{\mathcal{X}_l} \mathbf{z}_{\mathcal{X}_l} \\
\mathrm{s.t.} && \left[\mathbf{z}_{\mathcal{Y}_1}, ...,  \mathbf{z}_{\mathcal{Y}_M}\right]^{\top }\in \mathcal{Y} \\
&& \mathbf{A}_{\mathcal{X}_l}  \mathbf{z}_{\mathcal{X}_l}  + \mathbf{B}_{\mathcal{Y}_l} \mathbf{z}_{\mathcal{Y}_l} \geq \mathbf{b}_{\mathcal{X}_l, \mathcal{Y}_l} ~ i = 1, ..., N  \\ 
&& \mathbf{z}_{\mathcal{X}_l} \geq \mathbf{0}, ~ i = 1, ..., N 
\end{eqnarray}
where 
\begin{equation}
\mathbf{A}_{\mathcal{X}_l} = \left[\begin{array}{c}
-\mathbf{A}^{\mathrm{dp}}_{\mathcal{X}_l} 
\\
-\mathbf{A}^{\mathrm{um}}_{\mathcal{X}_l}  \\
\mathbf{A}^{\mathrm{um}}_{\mathcal{X}_l} \\
\end{array}\right], \mathbf{B}_{\mathcal{Y}_l} = \left[\begin{array}{c}
-\mathbf{B}^{\mathrm{dp}}_{\mathcal{Y}_l} \\
 \mathbf{0}_{\mathcal{Y}_l}  \\
\mathbf{0}_{\mathcal{Y}_l} \\
\end{array}\right],
\end{equation}
and
\begin{equation}
\mathbf{b}_{\mathcal{X}_l, \mathcal{Y}_l} = \left[\begin{array}{c}
\mathbf{0}_{\mathcal{R}_l} \\
-\mathbf{1}_{\mathcal{Y}_l}  \\
\mathbf{1}_{\mathcal{Y}_l} 
\end{array}\right] 
\end{equation}

\DEL{
Let 
\begin{equation}
f\left(\mathbf{z}_{\mathcal{Y}_l}\right) = \min_{\mathbf{z}_{\mathcal{X}_l} \geq \mathbf{0}} \left\{\mathbf{c}_{\mathcal{X}_l} \mathbf{z}_{\mathcal{X}_l}\left|\mathbf{A}_{\mathcal{X}_l}  \mathbf{z}_{\mathcal{X}_l}  \geq \mathbf{b}_{\mathcal{X}_l, \mathcal{Y}_l} - \mathbf{B}_{\mathcal{Y}_l} \mathbf{z}_{\mathcal{Y}_l}\right.\right\}
\end{equation}
or 
\begin{equation}
f\left(\mathbf{z}_{\mathcal{Y}_l}\right) = \max_{\mathbf{a}_{\mathcal{X}_l} \geq \mathbf{0}} \left\{\left(\mathbf{b}_{\mathcal{X}_l, \mathcal{Y}_l} -  \mathbf{B}_{\mathcal{Y}_l} \mathbf{z}_{\mathcal{Y}_l}\right)^{\top} \mathbf{a}_{\mathcal{X}_l}\left|\mathbf{A}_{\mathcal{X}_l}^{\top} \mathbf{a} \leq \mathbf{c}_{\mathcal{X}_l}\right.\right\}
\end{equation}}

\begin{equation}
\left[\begin{array}{cc}
\mathbf{A}^{\mathrm{dp}}_{\mathbf{z}_{\mathcal{X}_l} }
& \mathbf{B}^{\mathrm{dp}}_{\mathbf{z}_{\mathcal{Y}_l}} \\
\mathbf{A}^{\mathrm{um}}_{\mathbf{z}_{\mathcal{X}_l} } & \mathbf{0}_{\mathbf{z}_{\mathcal{Y}_l} } \\
\mathbf{0}_{\mathbf{z}_{\mathcal{X}_l} } & \mathbf{B}^{\mathrm{um}}_{\mathbf{z}_{\mathcal{Y}_l} } \\
-\mathbf{A}^{\mathrm{um}}_{\mathbf{z}_{\mathcal{X}_l} } & \mathbf{0}_{\mathbf{z}_{\mathcal{Y}_l} }\\
\mathbf{0}_{\mathbf{z}_{\mathcal{X}_l} } & -\mathbf{B}^{\mathrm{um}}_{\mathbf{z}_{\mathcal{Y}_l} }
\end{array}\right] \left[\begin{array}{l}
\mathbf{z}_{\mathcal{R}_l} \\ 
\mathbf{z}_{\mathcal{Y}_l} 
\end{array}\right] \leq \left[\begin{array}{l}
\mathbf{0}_{\mathcal{R}_l} \\
\mathbf{1}_{\mathbf{z}_{\mathcal{X}_l} } \\
\mathbf{1}_{\mathbf{z}_{\mathcal{Y}_l} } \\
-\mathbf{1}_{\mathbf{z}_{\mathcal{X}_l} } \\
-\mathbf{1}_{\mathbf{z}_{\mathcal{Y}_l} }
\end{array}\right] 
\end{equation}}

\section{Detailed Explanation of the Notations}
\subsection{Main notations used throughout the paper}
\vspace{-0.15in}
\begin{table}[h]
\caption{Main notations and their explanations.}
\vspace{-0.10in}
\label{Tb:Notationmodel}
\centering
\small 
\begin{tabular}{l l}
\hline
\hline
Sym.                  & Description \\
\hline
$N$                     & The number of records in the secret data set $\mathcal{R}$ \\ 
$K$                     & The number of records in the perturbed data set $\mathcal{O}$ \\ 
$M$                     & The number of subproblems \\ 
$\mathcal{R}$           & The secret dataset $\mathcal{R} = \left\{r_1, ..., r_N\right\}$, where $r_i$ denotes \\ 
& the $i$th record \\ 
$\mathcal{O}$           & The perturbed dataset $\mathcal{O} = \left\{o_1, ..., o_K\right\}$ \\ 
$\mathcal{E}$ & The set of neighboring records in $\mathcal{R}$ \\ 
$\mathcal{G}$                     & The mDP graph $\mathcal{G} = (\mathcal{R}, \mathcal{E})$ \\
$Q(\cdot)$              & The perturbation function \\
$\mathbf{Z}$            & The perturbation matrix \\ 
$z_{i,k}$               & The probability of selecting $o_k$ as the perturbed record  \\ 
                        & given the real record $r_i$ \\ 
$\epsilon$              & The privacy budget of mDP\\
$\eta$                  & The mDP neighbor threshold\\
$p_i$                   & The prior probability that a secret record is $r_i$ \\ 
$c_{i,k}$               & The data utility loss caused by the perturbed data $o_k$ \\ 
& given the real record $r_i$ \\
$d_{i,j}$               & The distance between the two records $r_i$ and $r_j$ \\
$\mathcal{G}_l$ & The $l$th subgraph of the mDP graph $\mathcal{G}_l = (\mathcal{R}_l, \mathcal{E}_l)$ \\ 
$\mathcal{R}_l$         & The set of secret records managed by subproblem $l$ \\ 
$\mathcal{X}_l$         & The set of internal records in $\mathcal{R}_l$ \\ 
$\mathcal{Y}_l$         & The set of boundary records in $\mathcal{R}_l$ \\ 
$\mathbf{z}_{\mathcal{R}_l}$    & The perturbation vectors of all the records in $\mathcal{R}_l$ \\ 
$\mathbf{z}_{\mathcal{X}_l}$    & The perturbation vectors of the internal records in $\mathcal{X}_l$ \\
$\mathbf{z}_{\mathcal{Y}_l}$    & The perturbation vectors of the boundary records in $\mathcal{Y}_l$ \\
\hline
\end{tabular}
\normalsize
\end{table}
\vspace{-0.00in}
\vspace{-0.1in}
\subsection{Detailed Explanation of the Notations in the PMO Reformulation}
\label{subsec:PMONotationExp}
\begin{itemize}
\item The vector $\mathbf{c}_{\mathcal{Y}_l} = \left\{\mathbf{c}_i \left|r_i \in \mathcal{Y}_l\right.\right\}$ includes the utility loss coefficients of all the records $r_i \in \mathcal{Y}_l$, where each $\mathbf{c}_i = [c_{i, 1}, ..., c_{i,K}]$ includes the utility loss coefficients of the perturbation vector $\mathbf{z}_i = [z_{i, 1}, ..., z_{i,K}]$; Similarly, the vector $\mathbf{c}_{\mathcal{X}_l} = \left\{\mathbf{c}_i \left|r_i \in \mathcal{X}_l\right.\right\}$ includes the utility loss coefficients of all the records $r_i \in \mathcal{X}_l$. 
\item The coefficient matrices $\mathbf{A}_{\mathcal{X}_l}$ and $\mathbf{B}_{\mathcal{Y}_l}$ can be written as
\vspace{-0.00in}
\begin{equation}
\mathbf{A}_{\mathcal{X}_l}  = -\left[\begin{array}{c}\mathbf{A}_{\mathcal{X}_l}^{\mathrm{mDP}} \\ \mathbf{A}_{\mathcal{X}_l}^{\mathrm{um}} \\
-\mathbf{A}_{\mathcal{X}_l}^{\mathrm{um}}\end{array}\right], 
\mathbf{B}_{\mathcal{Y}_l}  = -\left[\begin{array}{c}\mathbf{B}_{\mathcal{Y}_l}^{\mathrm{mDP}} \\ \mathbf{B}_{\mathcal{Y}_l}^{\mathrm{um}} \\
-\mathbf{B}_{\mathcal{Y}_l}^{\mathrm{um}}\end{array}\right]
\end{equation}
The coefficient vector $\mathbf{b}_{\mathcal{X}_l, \mathcal{Y}_l}$ can be written as 
\begin{equation}
\mathbf{b}_{\mathcal{X}_l, \mathcal{Y}_l}  = -\left[\begin{array}{c}\mathbf{b}_{\mathcal{X}_l, \mathcal{Y}_l}^{\mathrm{mDP}} \\ \mathbf{b}_{\mathcal{X}_l, \mathcal{Y}_l}^{\mathrm{um}} \\ -\mathbf{b}_{\mathcal{X}_l, \mathcal{Y}_l}^{\mathrm{um}} \end{array}\right],
\end{equation}
In the above notations, the supscripts ``mDP'' and ``um'' mean the ``mDP constraints'' and the ``unit measure constraints'', respectively. They form the constraint matrix for perturbation vectors $\mathbf{z}_{\mathcal{X}_l}$ and $\mathbf{z}_{\mathcal{Y}_l}$: 
\begin{equation}
\left[\begin{array}{cc} \mathbf{A}_{\mathcal{X}_l}^{\mathrm{mDP}} & \mathbf{B}_{\mathcal{Y}_l}^{\mathrm{mDP}} \\ \mathbf{A}_{\mathcal{X}_l}^{\mathrm{um}} & \mathbf{B}_{\mathcal{Y}_l}^{\mathrm{um}} \\
-\mathbf{A}_{\mathcal{X}_l}^{\mathrm{um}} & -\mathbf{B}_{\mathcal{Y}_l}^{\mathrm{um}}\end{array}\right] \left[\begin{array}{c} \mathbf{z}_{\mathcal{X}_l} \\ \mathbf{z}_{\mathcal{Y}_l} \end{array}\right] \leq \left[\begin{array}{c}\mathbf{b}_{\mathcal{X}_l, \mathcal{Y}_l}^{\mathrm{mDP}} \\ \mathbf{b}_{\mathcal{X}_l, \mathcal{Y}_l}^{\mathrm{um}} \\ -\mathbf{b}_{\mathcal{X}_l, \mathcal{Y}_l}^{\mathrm{um}} \end{array}\right].
\end{equation}
Specifically, the block matrix $\left[\mathbf{A}_{\mathcal{X}_l}^{\mathrm{mDP}}, \mathbf{B}_{\mathcal{Y}_l}^{\mathrm{mDP}}\right]$ includes the mDP constraints between the boundary records in $\mathcal{R}_{l}$: 
\begin{eqnarray}
\label{eq:}
\scriptsize   
\nonumber && \left[\mathbf{A}_{\mathcal{X}_l}^{\mathrm{mDP}}, \mathbf{B}_{\mathcal{Y}_l}^{\mathrm{mDP}}\right] \\ \nonumber 
&=& \scriptsize  
\left[\begin{array}{lcccr} 
\ddots & \cdots     & \cdots  & \cdots                              & \iddots \\ 
\cdots & 1   & \cdots  & -e^{{\epsilon d_{i,j}}}    & \cdots \\ 
\cdots & -e^{{\epsilon d_{i,j}}}    & \cdots  & 1    & \cdots \\
\iddots & \cdots     & \cdots  & \cdots                              & \ddots \\ \end{array}\right] 
\hspace{-0.06in}
\begin{array}{l}
\left\}
\begin{array}{l}
\forall r_i, r_j \in \mathcal{R}_{l} \\
\mbox{s.t.}~d_{i, j} \leq \eta
\end{array}
\right.
\\
\end{array}
\end{eqnarray}
$\left[\mathbf{A}_{\mathcal{X}_l}^{\mathrm{um}}, \mathbf{B}_{\mathcal{Y}_l}^{\mathrm{um}}\right]$ includes $|\mathcal{R}_{l}|$ rows, where each row corresponds to the unit measure constraint of a record $r_i \in \mathcal{R}_{l}$. $\mathbf{A}_{\mathcal{X}_l}^{\mathrm{mDP}}$ and $\mathbf{A}_{\mathcal{X}_l}^{\mathrm{um}}$ include the coefficients of $\mathbf{z}_{\mathcal{X}_l}$ in the constraint matrices. $\mathbf{B}_{\mathcal{Y}_l}^{\mathrm{mDP}}$ and $\mathbf{B}_{\mathcal{Y}_l}^{\mathrm{um}}$ include the coefficients of $\mathbf{z}_{\mathcal{Y}_l}$ in the constraint matrices. 

$\mathbf{b}_{\mathcal{X}_l, \mathcal{Y}_l}^{\mathrm{mDP}}$ is an all-zeros vector, which corresponds to the right-hand side coefficients of the constraint matrix $\left[\mathbf{A}_{\mathcal{X}_l}^{\mathrm{mDP}}, \mathbf{B}_{\mathcal{Y}_l}^{\mathrm{mDP}}\right]$ in the LP formulation. 

$\mathbf{b}_{\mathcal{X}_l, \mathcal{Y}_l}^{\mathrm{um}}$ is an all-ones vector, which corresponds to the right-hand side coefficients of the constraint matrix $\left[\mathbf{A}_{\mathcal{X}_l}^{\mathrm{um}}, \mathbf{B}_{\mathcal{Y}_l}^{\mathrm{um}}\right]$ in the LP formulation.
Note that 
\begin{equation}
\left[\begin{array}{cc} \mathbf{A}_{\mathcal{X}_l}^{\mathrm{um}} & \mathbf{B}_{\mathcal{Y}_l}^{\mathrm{um}} \\
-\mathbf{A}_{\mathcal{X}_l}^{\mathrm{um}} & -\mathbf{B}_{\mathcal{Y}_l}^{\mathrm{um}}\end{array}\right] \left[\begin{array}{c} \mathbf{z}_{\mathcal{X}_l} \\ \mathbf{z}_{\mathcal{Y}_l} \end{array}\right] \leq \left[\begin{array}{c}\mathbf{b}_{\mathcal{X}_l, \mathcal{Y}_l}^{\mathrm{um}} \\ -\mathbf{b}_{\mathcal{X}_l, \mathcal{Y}_l}^{\mathrm{um}} \end{array}\right].
\end{equation}
is equivalent to
\vspace{-0.00in}
\begin{equation}
\left[\begin{array}{cc} \mathbf{A}_{\mathcal{X}_l}^{\mathrm{um}} & \mathbf{B}_{\mathcal{Y}_l}^{\mathrm{um}} \end{array}\right] \left[\begin{array}{c} \mathbf{z}_{\mathcal{X}_l} \\ \mathbf{z}_{\mathcal{Y}_l} \end{array}\right] = \mathbf{b}_{\mathcal{X}_l, \mathcal{Y}_l}^{\mathrm{um}}.
\end{equation}

\item $\mathbf{B}_{\mathcal{Y}_1, ..., \mathcal{Y}_M}$ includes the constraints of all the boundary records in $\mathcal{Y}_1, ..., \mathcal{Y}_M$, composed of three parts
\begin{equation}
\mathbf{B}_{\mathcal{Y}_1, ..., \mathcal{Y}_M}  = -\left[\begin{array}{c}\mathbf{B}_{\mathcal{Y}_1, ..., \mathcal{Y}_M}^{\mathrm{mDP}} \\ \mathbf{B}_{\mathcal{Y}_1, ..., \mathcal{Y}_M}^{\mathrm{um}} \\ 
-\mathbf{B}_{\mathcal{Y}_1, ..., \mathcal{Y}_M}^{\mathrm{um}}\end{array}\right]
\end{equation}
where $\mathbf{B}_{\mathcal{Y}_1, ..., \mathcal{Y}_M}^{\mathrm{mDP}}$ includes the mDP constraints between any pair of boundary records in $\mathcal{Y}_1, ..., \mathcal{Y}_M$: 
\normalsize 
\small
\begin{eqnarray}
&& \mathbf{B}_{\mathcal{Y}_1, ..., \mathcal{Y}_M}^{\mathrm{mDP}}  \\ \nonumber 
&=& \scriptsize  
\left[\begin{array}{lcccr} 
\ddots & \cdots     & \cdots  & \cdots                              & \iddots \\ 
\cdots & 1   & \cdots  & -e^{{\epsilon d_{i,j}}}    & \cdots \\ 
\cdots & -e^{{\epsilon d_{i,j}}}    & \cdots  & 1    & \cdots \\
\iddots & \cdots     & \cdots  & \cdots                              & \ddots \\ \end{array}\right] 
\hspace{-0.06in}
\begin{array}{l}
\left\}
\begin{array}{l}
\forall r_i, r_j \in \cup_l\mathcal{Y}_{l} \\
\mbox{s.t.}~d_{i, j} \leq \eta
\end{array}
\right.
\\
\end{array}
\end{eqnarray} 
\normalsize 
and $\mathbf{B}_{\mathcal{Y}_1, ..., \mathcal{Y}_M}^{\mathrm{um}}$ includes $|\cup_l\mathcal{Y}_{l}|$ rows, where each row corresponds to the unit measure constraint of a record $r_i \in \cup_l \mathcal{Y}_{l}$. 
\item $\mathbf{b}_{\mathcal{Y}_1, ..., \mathcal{Y}_M}$ has three components: 
$$
\mathbf{b}_{\mathcal{Y}_1, ..., \mathcal{Y}_M}  = -\left[\begin{array}{c}\mathbf{b}_{\mathcal{Y}_1, ..., \mathcal{Y}_M}^{\mathrm{mDP}} \\ \mathbf{b}_{\mathcal{Y}_1, ..., \mathcal{Y}_M}^{\mathrm{um}} \\ -\mathbf{b}_{\mathcal{Y}_1, ..., \mathcal{Y}_M}^{\mathrm{um}} \end{array}\right],$$
where 
$\mathbf{b}_{\mathcal{Y}_1, ..., \mathcal{Y}_M}^{\mathrm{mDP}}$ is an all-zeros vector, which corresponds to the right-hand side coefficients of the constraint matrix $\mathbf{B}_{\mathcal{Y}_1, ..., \mathcal{Y}_M}^{\mathrm{mDP}}$, and $\mathbf{b}_{\mathcal{Y}_1, ..., \mathcal{Y}_M}^{\mathrm{um}}$ is an all-ones vector, which corresponds to the right-hand side coefficients of the constraint matrix $\mathbf{B}_{\mathcal{Y}_1, ..., \mathcal{Y}_M}^{\mathrm{um}}$. 
Note that 
\begin{equation}
\left[\begin{array}{c}\mathbf{B}_{\mathcal{Y}_1, ..., \mathcal{Y}_M}^{\mathrm{um}} \\ 
-\mathbf{B}_{\mathcal{Y}_1, ..., \mathcal{Y}_M}^{\mathrm{um}}\end{array}\right] \left[\begin{array}{c} \mathbf{z}_{\mathcal{Y}_1} \\ \vdots \\ \mathbf{z}_{\mathcal{Y}_M} \end{array}\right] \leq \left[\begin{array}{c} \mathbf{b}_{\mathcal{Y}_1, ..., \mathcal{Y}_M}^{\mathrm{um}} \\ -\mathbf{b}_{\mathcal{Y}_1, ..., \mathcal{Y}_M}^{\mathrm{um}} \end{array}\right]
\end{equation} 
is equivalent to 
\begin{equation}
\mathbf{B}_{\mathcal{Y}_1, ..., \mathcal{Y}_M}^{\mathrm{um}} \left[\begin{array}{c} \mathbf{z}_{\mathcal{Y}_1} \\ \vdots \\ \mathbf{z}_{\mathcal{Y}_M} \end{array}\right] = \mathbf{b}_{\mathcal{Y}_1, ..., \mathcal{Y}_M}^{\mathrm{um}}.
\end{equation} 
\end{itemize}

\DEL{
\begin{equation}
\left[\begin{array}{c}
-\mathbf{A}^{\mathrm{dp}}_{\mathbf{z}_{\mathcal{X}_l} }
\\
-\mathbf{A}^{\mathrm{um}}_{\mathbf{z}_{\mathcal{X}_l} } \\
\mathbf{0}_{\mathbf{z}_{\mathcal{X}_l} } \\
\mathbf{A}^{\mathrm{um}}_{\mathbf{z}_{\mathcal{X}_l} }\\
\mathbf{0}_{\mathbf{z}_{\mathcal{X}_l} }
\end{array}\right] 
\mathbf{z}_{\mathcal{X}_l}  + \left[\begin{array}{ll}
-\mathbf{B}^{\mathrm{dp}}_{\mathbf{z}_{\mathcal{Y}_l}} \\
 \mathbf{0}_{\mathbf{z}_{\mathcal{Y}_l} } \\
-\mathbf{B}^{\mathrm{um}}_{\mathbf{z}_{\mathcal{Y}_l} } \\
\mathbf{0}_{\mathbf{z}_{\mathcal{Y}_l} }\\
\mathbf{B}^{\mathrm{um}}_{\mathbf{z}_{\mathcal{Y}_l} }
\end{array}\right] 
\mathbf{z}_{\mathcal{Y}_l} 
\geq \left[\begin{array}{l}
\mathbf{0}_{\mathcal{R}_l} \\
-\mathbf{1}_{\mathbf{z}_{\mathcal{X}_l} } \\
-\mathbf{1}_{\mathbf{z}_{\mathcal{Y}_l} } \\
\mathbf{1}_{\mathbf{z}_{\mathcal{X}_l} } \\
\mathbf{1}_{\mathbf{z}_{\mathcal{Y}_l} }
\end{array}\right] 
\end{equation}}

\section{Omitted Proofs}
\label{sec:proofs}
\subsection{Proof of Property \ref{property:mDPcompo} \\
(Independent Computation of mDP Components)}
\label{subsec:proof:mDPcompo}
\begin{proof}
Given the mDP graph has $m$ components $\mathcal{C}_1, ..., \mathcal{C}_m$, the corresponding LP formulation of PMO is 
\normalsize
\small 
\vspace{-0.0in}
\begin{eqnarray}
\label{eq:independentObj}
\min && \textstyle  \sum_{l=1}^m \mathbf{c}_{\mathcal{R}_l} \mathbf{z}_{\mathcal{R}_l} \\ 
\mathrm{s.t.} && \mbox{mDP constraints of each $\mathbf{z}_{i}$ in $\mathcal{R}_l$ are satisfied,}~\forall l \\ 
&& \mbox{Unit measure of each $\mathbf{z}_{i}$ in $\mathcal{R}_l$ is satisfied,}~\forall l\\  \label{eq:independentconstr}
&& 0 \leq z_{i,k} \leq 1, \forall (r_i, o_k) \in \mathcal{R}_l \times \mathcal{O},~\forall l. 
\end{eqnarray}
\normalsize
Since both mDP constraints and unit measure constraints are disjoint between the perturbation vectors across different subsets $\mathcal{R}_1$, ..., $\mathcal{R}_m$, the LP formulated in Equ. (\ref{eq:independentObj})--(\ref{eq:independentconstr}) is equivalent to solving the following $\mathrm{Sub}_l$ independently ($l=1, ..., m$):  
\normalsize
\small 
\vspace{-0.05in}
\begin{eqnarray}
\min && \textstyle  \mathbf{c}_{\mathcal{R}_l} \mathbf{z}_{\mathcal{R}_l} \\ 
\mathrm{s.t.} && \mbox{mDP constraints of each $\mathbf{z}_{i}$ in $\mathcal{R}_l$ are satisfied} \\  
&& \mbox{Unit measure of each $\mathbf{z}_{i}$ in $\mathcal{R}_l$ is satisfied} \\ 
&& 0 \leq z_{i,k} \leq 1, \forall (r_i, o_k) \in \mathcal{R}_l \times \mathcal{O}. 
\end{eqnarray}
\normalsize
The proof is completed. 
\end{proof}

\subsection{Proof of Property \ref{property:chainrule}} 
\label{subsec:proof:chainrule}
\begin{proof}
We let $\{r_{i}, r_{l_1}, r_{l_2}, ..., r_{l_{n-1}}, r_{l_n}, r_{j}\}$ represent the sequence of records in the shortest path between $r_i$ and $r_j$. Since each pair of adjacent records satisfy mDP, for each $o_k \in \mathcal{O}$, we have 
\vspace{-0.00in}
\normalsize
\begin{eqnarray}
&& \frac{z_{i,k}}{z_{l_1,k}} \leq e^{\epsilon d_{i, {l_1}}}, \frac{z_{l_n,k}}{z_{j,k}} \leq e^{\epsilon d_{{l_n}, {j}}},\\
&& \frac{z_{l_m,k}}{z_{l_{m+1}},k} \leq e^{\epsilon d_{{l_m}, {l_{m+1}}}}~(m = 1, ..., n-1), 
\end{eqnarray}
\normalsize
from which we can derive that 
\begin{eqnarray} 
\nonumber \frac{z_{i,k}}{z_{j,k}} &=& \frac{z_{i,k}}{z_{l_1,k}} \prod_{m=1}^{n-1} \frac{z_{l_m,k}}{z_{l_{m+1}},k} \frac{z_{l_n,k}}{z_{j,k}} \\ 
&\leq&  e^{\epsilon d_{i, {l_1}}} \prod_{m=1}^{n-1} e^{\epsilon d_{{l_m}, {l_{m+1}}}} e^{\epsilon d_{{l_n}, {j}}} \\   \nonumber 
&=& e^{\epsilon \left(d_{i, {l_1}} + \sum_{m=1}^{n-1} d_{{l_m}, {l_{m+1}}} + d_{{l_n}, j}\right)} 
\\
&=& e^{\epsilon D_{i,j}}
\end{eqnarray}
\normalsize
The proof is completed. 
\end{proof}

\section{Benchmarks of Secret Dataset Partitioning}
\label{sec:benchmarks}
In this section, we introduce the details of the three benchmarks of dataset partitioning, ``k-m-rec'' (\textbf{Section \ref{subsec:benchmarks_kmean_rec}}), ``k-m-adj'' (\textbf{Section \ref{subsec:benchmarks_kmean_adj}}), and ``BSC'' (\textbf{Section \ref{subsec:benchmarks_BSC}}). 

\subsection{K-means based on records (k-m-rec)}
\label{subsec:benchmarks_kmean_rec}
Each record $r_i$ is represented by vector $\mathbf{v}_i = [v_{i,1}, ..., v_{i, L}]$, where $L$ is the dimension of the vectors. For instance, a geo-location record in the road network can be represented as a 2-dimensional vector (longitude, latitude), and a word can be represented as a 300-dimensional vector using the \texttt{word2vec} function \cite{word2vec}. To determine the partitions $\mathcal{R}_1, \ldots, \mathcal{R}_M$, we formulate the following \emph{k-means clustering} formulation:
\vspace{-0.00in}
\begin{equation}
\label{eq:kmean-rec}
\min ~ \sum_{l=1}^M \sum_{r_i\in \mathcal{R}_l} \|\mathbf{v}_i-\boldsymbol{\mu}_l\|_2
\end{equation}
where $\boldsymbol{\mu}_l$ represents the \emph{centroid} of the vectors in $\mathcal{R}_l$, i.e., $\boldsymbol{\mu}_l = \frac{\sum_{r_j \in \mathcal{R}_l}\mathbf{v}_j}{|\mathcal{R}_l|}$ ($l = 1, \ldots, M$). 

\subsection{K-means based on adjacency matrix (k-m-adj)}
\label{subsec:benchmarks_kmean_adj}

Each record $r_i$ is embedded by its adjacency vector to the neighbors $\mathbf{w}_i = [w_{i,1}, ..., w_{i, N}]$, where each indicator variable $w_{i,j}$ represents whether the record $r_j$ is a neighbor of $r_i$ in the mDP graph, i.e., 
\begin{equation}
\label{eq:w}
w_{i,j} = \left\{\begin{array}{ll} 1 & \mbox{if $r_j$ is a neighbor of $r_i$} \\
0 & \mbox{if $r_j$ is not a neighbor of $r_i$} \end{array}\right.
\end{equation}
We formulate the following \emph{k-means clustering} problem to determine the partitions $\mathcal{R}_1, \ldots, \mathcal{R}_M$:
\vspace{-0.00in}
\begin{equation}
\label{eq:kmean-adj}
\min ~ \sum_{l=1}^M \sum_{r_i\in \mathcal{R}_l} \|\mathbf{w}_i-\boldsymbol{\mu}_l\|_2
\end{equation}
where $\boldsymbol{\mu}_l$ represents the \emph{centroid} of the adjacency vectors in $\mathcal{R}_l$, i.e., $\boldsymbol{\mu}_l = \frac{\sum_{r_j \in \mathcal{R}_l}\mathbf{w}_j}{|\mathcal{R}_l|}$ ($l = 1, \ldots, M$).

Note that, similar to the objective function in Equ. (\ref{eq:kmean}), the objective functions in Equ. (\ref{eq:kmean-rec}) and Equ. (\ref{eq:kmean-adj}) attain a lower value when the records within each subset $\mathcal{R}_l$ are strongly connected and the sizes of $\mathcal{R}_1, ..., \mathcal{R}_M$ are well-balanced.

\subsection{Balanced Spectral Clustering (BSC)}
\label{subsec:benchmarks_BSC}
The problem of minimizing the number of mDP constraints (edges) across subsets $\mathcal{R}_1$, ..., $\mathcal{R}_M$, and balancing the number of internal records of all the subsets can be formulated as the following \emph{RaioCut problem} \cite{Ng-NIPS2001}: 
\vspace{-0.05in}
\begin{equation}
\label{eq:partitionobj}
\min~ 
\sum_{l=1}^M \frac{\mathrm{CR}\left(\mathcal{R}_l, \overline{\mathcal{R}}_l\right)}{|\mathcal{R}_l|} 
\end{equation}
where $\overline{\mathcal{R}}_l$ denotes the complement of $\mathcal{R}_l$ (i.e. $\overline{\mathcal{R}}_l = \mathcal{R}\backslash \mathcal{R}_l$), and $\mathrm{CR}\left(\mathcal{R}_l, \overline{\mathcal{R}}_l\right)$ denote the number of edges (cuts) across $\mathcal{R}_l$ and other subsets $\overline{\mathcal{R}}_l$. \emph{RaioCut problem} is known NP-hard and its relaxed versions can be solved by \emph{Spectral clustering} \cite{vonluxburg2007tutorial}. Given the adjacence matrix $\mathbf{W} = \left\{w_{i,j}\right\}_{N \times N}$ of the mDP graph $\mathcal{G}$, the \emph{degree matrix} of $\mathbf{W}$ is denoted by $\textbf{D}$, which is a diagonal matrix with the degrees of nodes $r_1$, ..., $r_n$ on the diagonal. Then, the \emph{unnormalized graph Laplacian matrix} $\mathbf{L}$ is given by $\mathbf{L} = \mathbf{D} - \mathbf{W}$. 

We define $M$ indicator vectors
$\mathbf{f}_l = [f_{1,l}, . . . , f_{N, l}]^{\top}$, where 
\begin{equation}
\label{eq:f}
f_{i,l} = \left\{\begin{array}{ll}
\frac{1}{\sqrt{|\mathcal{R}_l|}} & r_i \in \mathcal{R}_l  \\
0 & r_i \notin \mathcal{R}_l\end{array}\right. ~(i = 1, ..., N)
\end{equation}
and let the matrix $\mathbf{F} = [\mathbf{f}_1, ..., \mathbf{f}_M]$ contain the $M$ indicator vectors. Since the columns in $\mathbf{F}$ are orthonormal to each other,  $\mathbf{F}^{\top}\mathbf{F} = \mathbf{I}_M$, where $\mathbf{I}_M$ is an identity matrix. Moreover, the objective function in Equ. (\ref{eq:partitionobj}) can be rewritten as $
\sum_{l=1}^M \frac{\mathrm{CR}\left(\mathcal{R}_l, \overline{\mathcal{R}}_l\right)}{|\mathcal{R}_l|} = \mathrm{Trace}\left(\mathbf{F}^{\top}\mathbf{L}\mathbf{F}\right)$ \cite{vonluxburg2007tutorial}.

Accordingly, relaxing $\mathbf{F}$ to the continuous region $\mathbb{R}^{N\times M}$, denoted by $\tilde{\mathbf{F}} = \left\{\tilde{f}_{i,j}\right\}_{N\times M}$, we can formulate the relaxed problem of the dataset partitioning (in Equ. (\ref{eq:partitionobj})) as: 
\begin{equation}
\label{eq:traceminrel}
\min_{\tilde{\mathbf{F}} \in \mathbb{R}^{N\times M}}~ 
\mathrm{Trace}\left(\tilde{\mathbf{F}}^{\top}\mathbf{L}\tilde{\mathbf{F}}\right)~\mathrm{s.t.}~ \tilde{\mathbf{F}}^{\top}\tilde{\mathbf{F}} = \mathbf{I}_M,  
\end{equation}
which is the standard \emph{trace minimization problem}. 
As per the \emph{Rayleigh-Ritz Theorem} \cite{lutkepohl1996handbook}, the optimal $\tilde{\mathbf{F}}$ is attained when its columns encompass the first $M$ eigenvectors of $\mathbf{L}$. To revert $\tilde{\mathbf{F}}$ to the discrete form $\mathbf{F}$ defined in Equ. (\ref{eq:f}), we apply k-means on the row of $\tilde{\mathbf{F}}$, i.e., 
\begin{equation}
\label{eq:kmean-BSC}
\min ~ \sum_{l=1}^M \sum_{r_i\in \mathcal{R}_l} \|\tilde{\mathbf{f}}^i-\boldsymbol{\mu}_l\|_2
\end{equation}
where $\tilde{\mathbf{f}}^i$ denotes the $i$th row of $\tilde{\mathbf{F}}$ and  $\boldsymbol{\mu}_l$ represents the \emph{centroid} of $\tilde{\mathbf{f}}^i$ in $\mathcal{R}_l$, i.e., $\boldsymbol{\mu}_l = \frac{\sum_{r_j \in \mathcal{R}_l}\tilde{\mathbf{f}}^j}{|\mathcal{R}_l|}$ ($l = 1, \ldots, M$). 

\DEL{
We can also obtain $|\mathcal{Y}_l|$ and $|\mathcal{X}_l|$ ($l = 1, ..., M$) 
\normalsize
\small 
\begin{eqnarray}
&& \textstyle |\mathcal{Y}_l| = \sum_{r_i \in \mathcal{R}_l}H\left(\textstyle \sum_{r_j \in \overline{\mathcal{R}}_l} w_{i,j}-0.5\right), \\
&& \textstyle |\mathcal{X}_l| = \sum_{r_i \in \mathcal{R}_l} H\left(\textstyle 0.5 - \sum_{r_j \in \overline{\mathcal{R}}_l} w_{i,j}\right)
\end{eqnarray}
\normalsize
where $H(\cdot)$ is the \emph{Heaviside step function}, i.e., $H(n) = 1$ if $n\geq 0$ and $H(n) = 0$ if $n < 0$.

\begin{equation}
\label{eq:frel}
\small \left\{\begin{array}{ll}
\frac{1 - \frac{|\mathcal{Y}_l|}{\Lambda|\mathcal{X}_l|}}{\Lambda|\mathcal{X}_l|} = \tilde{f}^2_{i,l} & r_i \in \mathcal{R}_l \mbox{ and } r_i \in \mathcal{X}_l  \\
\frac{1}{\Lambda|\mathcal{X}_l|} = \tilde{f}^2_{i,l} & r_i \in \mathcal{R}_l \mbox{ and } r_i \in \mathcal{Y}_l  \\
0 & r_i \notin \mathcal{R}_l\end{array}\right.
\end{equation}

\begin{equation}
\tilde{f}^2_{i,l}\Lambda^2|\mathcal{X}_l|^2 - \Lambda|\mathcal{X}_l| + |\mathcal{Y}_l| = 0 
\end{equation}
\begin{equation}
|\mathcal{X}_l| = \frac{1}{\Lambda \tilde{f}^2_{i,l} }
\end{equation}
implying that 
\begin{eqnarray}
|\mathcal{X}_l| &=& \frac{1}{\Lambda \tilde{f}^2_{i,l}} \\
|\mathcal{Y}_l| &=& \frac{1}{\tilde{f}^2_{i,l}} - \tilde{f}^2_{i,l}\Lambda^2|\mathcal{X}_l|^2
\end{eqnarray}}


\section{Additional Experimental Results}
\label{sec:addExp}

In this section, we provide the additional experimental results, including the data utility loss (\textbf{Section \ref{subsec:UL}}), the convergence speed of subproblems (\textbf{Section \ref{subsec:convergencespeed}}), the computation time of MP components and subproblems (\textbf{Section \ref{subsec:computetime}}), the size of decomposed problems (\textbf{Section \ref{subsec:datasetpartition}}), and examples of BD convergence (\textbf{Section \ref{subsec:BDconvergence}}). 

\subsection{Data Utility Loss}
\label{subsec:UL}

\textbf{Benchmark}. In Section \ref{sec:performance}, we have shown that our approach, labeled as ``BD'', can achieve a much lower computation time than the traditional LP-based methods (in this part, we only use the dataset partitioning algorithm ``k-m-DV'' in the BD framework). Notably, data perturbation using the \emph{exponential mechanism} also exhibits lower time complexity but at the expense of data utility. To demonstrate, in this part, we assess the expected data utility loss of our method by comparing it with an exponential mechanism \cite{McSherry-FOCS2007}, labeled as ``ExpMech''. Specifically, given a secret data set $\mathcal{R}$ and its distance measure $d$, ExpMech selects the perturbed record based on the following probability distribution:
\begin{equation}
\small \mathbb{P}\mathrm{r}\left\{Q(r_i) = o_k\right\} = \frac{e^{-\epsilon d_{i, k}/2}}{\sum_{l} e^{-\epsilon d_{i, l}/2}}.
\end{equation}

\noindent \textbf{Expected data utility loss}. Given the perturbation matrix $\mathbf{Z}$, the \emph{expected data utility loss} is calculated by $\sum_{i=1}^N \mathbf{c}_i \mathbf{z}_{i}$, where $\mathbf{c}_i = [p_i c_{i, 1}, ..., p_i c_{i,K}]$ and $p_i$ denotes the prior probability of the record being $r_i$. Here, for the geo-location data in the road network and grip maps, we consider the location-based services where the users are recommended to travel to a designated location based on their perturbed locations (e.g., hotel/restaurant recommendation). In this case, the utility loss can be measured by the discrepancy between the estimated travel cost and the actual travel cost to reach the designated locations.

Specifically, when applying the road network and the grid map datasets, we randomly deploy 100 destinations over the target region. For each destination located at $r_{\mathrm{dest}}$, we calculate the utility loss caused by a perturbed record $o_k$ given the real record $r_i$ by $\left|pd(r_i,r_{\mathrm{dest}}) - pd(o_k,r_{\mathrm{dest}})\right|$, where $pd(r_i,r_{\mathrm{dest}})$ (or $pd(o_k,r_{\mathrm{dest}})$) denotes the \emph{path distance} from the location $r_i$ (or $o_k$) to the destination $r_{\mathrm{dest}}$. The expected data utility loss $c_{i,k}$ is calculated by 
\begin{equation}
\textstyle 
\small c_{i,k} = \sum_{r_{\mathrm{dest}} \in \mathcal{R}}p_{\mathrm{dest}}\left|pd(r_i,r_{\mathrm{dest}}) - pd(o_k,r_{\mathrm{dest}})\right|.
\end{equation}
where $p_{\mathrm{dest}}$ is the prior probability of the destination being at location $r_{\mathrm{dest}}$. 

\begin{figure}[t]
\centering
\hspace{0.00in}
\begin{minipage}{0.225\textwidth}
  \subfigure[\small Road network]{
\includegraphics[width=1.00\textwidth, height = 0.13\textheight]{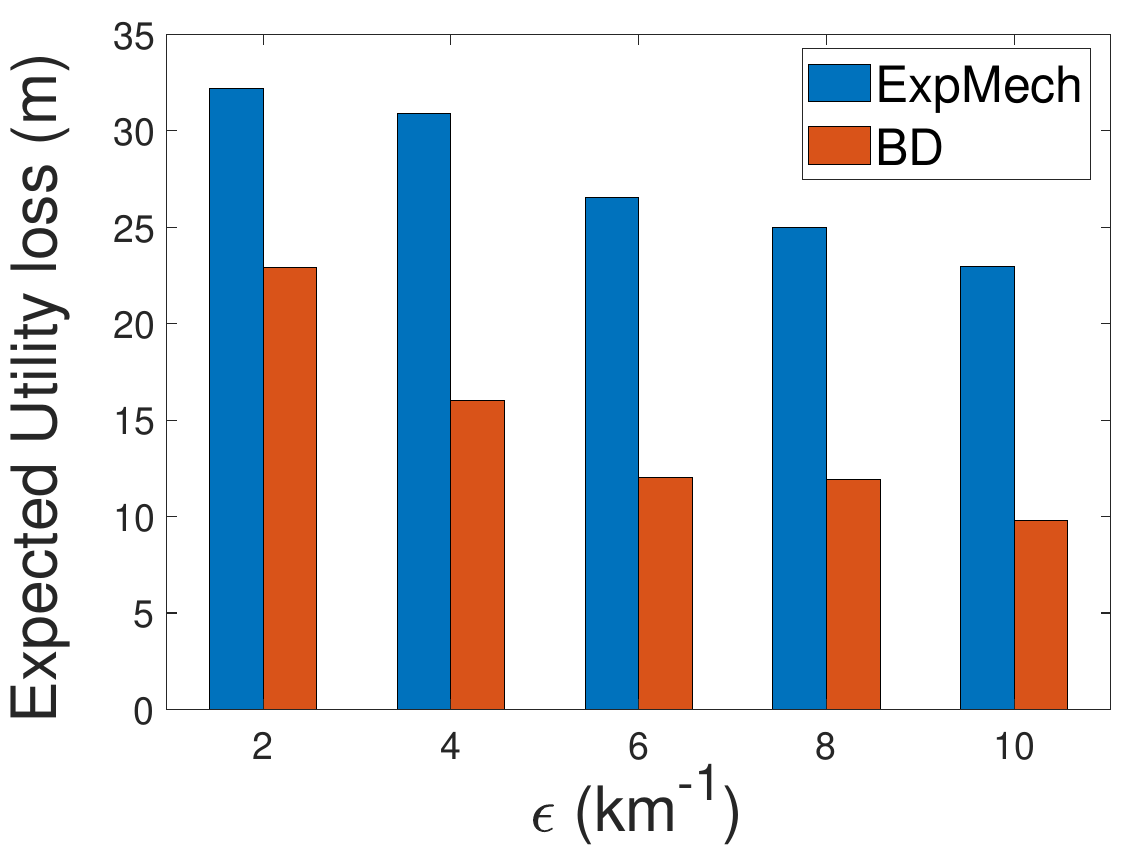}}
\end{minipage}
\begin{minipage}{0.225\textwidth}
  \subfigure[\small Grid map]{
\includegraphics[width=1.00\textwidth, height = 0.13\textheight]{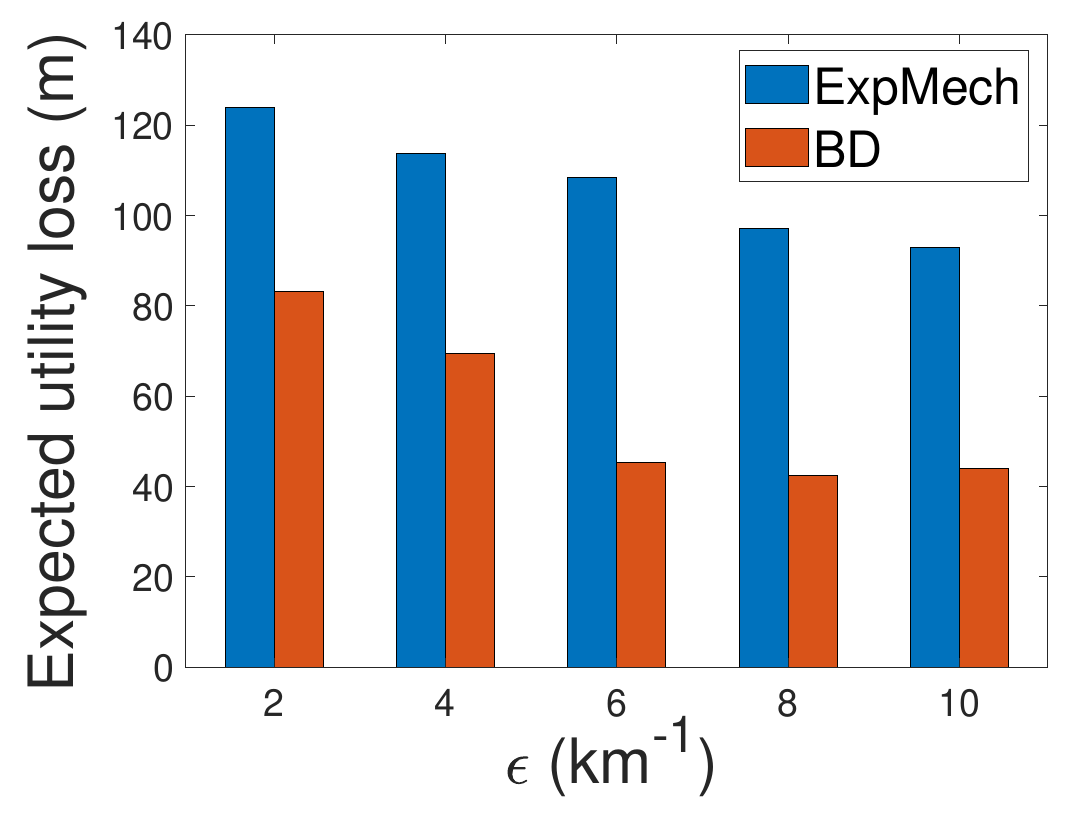}}
\vspace{-0.00in}
\end{minipage}
\begin{minipage}{0.225\textwidth}
  \subfigure[\small Text data]{
\includegraphics[width=1.00\textwidth, height = 0.13\textheight]{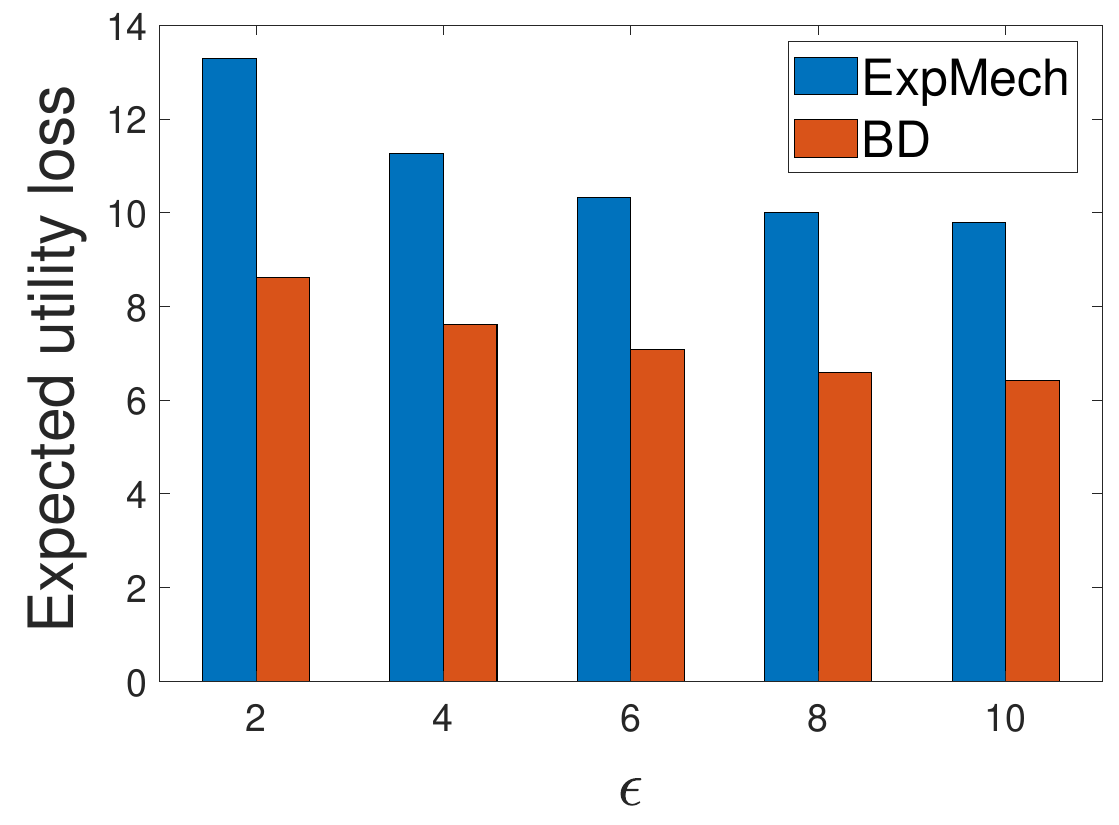}}
\end{minipage}
\begin{minipage}{0.225\textwidth}
  \subfigure[\small Synthetic data]{
\includegraphics[width=1.00\textwidth, height = 0.13\textheight]{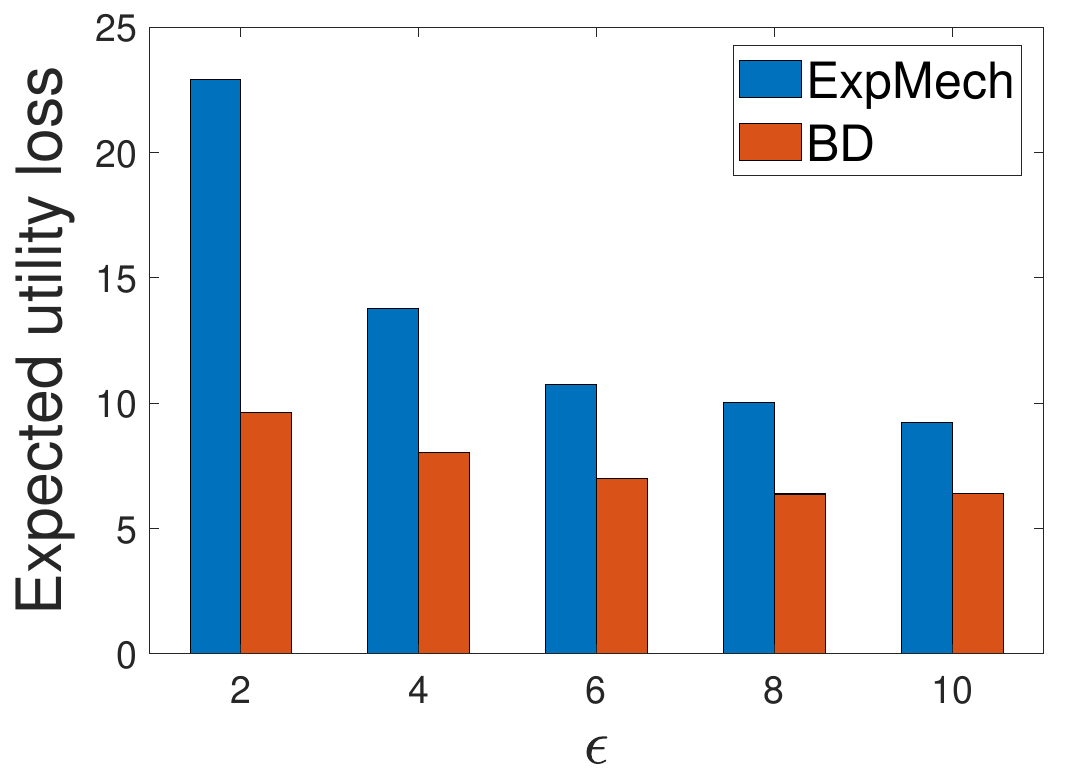}}
\end{minipage}
\vspace{-0.00in}
\caption{\normalsize Expected utility loss.}
\label{fig:utilityloss}
\vspace{-0.00in}
\end{figure}

In the context of both the text dataset and the synthetic dataset, we define the utility loss coefficient $c_{i,k}$ as the Euclidean distance between the original record $r_i$ and the perturbed data $o_k$. Note that our framework is not confined to particular downstream data processing tasks and can easily be extended to diverse applications where the accurate measurement of $c_{i,k}$ is applicable.


In Fig \ref{fig:utilityloss}(a)(b)(c)(d), we depict the expected data utility loss of BD and ExpMech using the four datasets, with the privacy budget increased from 2 to 10. From the figures, we find that on average, the data utility of BD is 47.19\%, 46.99\%, 22.49\%, and 24.05\% lower than that of ExpMech when using the road network, grid map, text, and synthetic datasets, respectively. ExpMech demonstrates a higher level of utility loss, as it doesn't capture the correlation between utility loss and perturbation while determining the perturbation vector for each record. Additionally, we note that the expected data utility loss decreases as the privacy budget $\epsilon$ increases for both methods. This trend is attributed to a higher value of $\epsilon$ imposing fewer restrictions on perturbation vectors, allowing the algorithms to select perturbed data that is closer to the original records, resulting in a lower utility loss.

\subsection{Convergence Speed of Subproblems}
\label{subsec:convergencespeed}

\begin{figure}[t]
\centering
\hspace{0.00in}
\begin{minipage}{0.23\textwidth}
  \subfigure[\small Road network]{
\includegraphics[width=1.00\textwidth, height = 0.13\textheight]{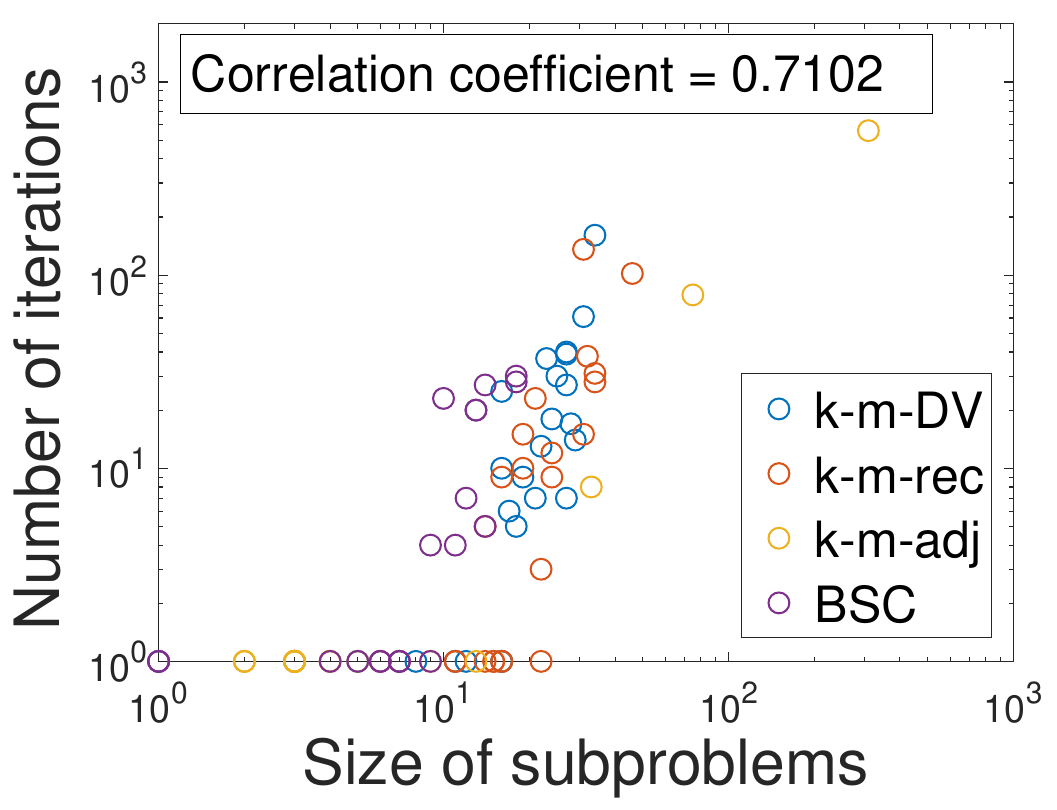}}
\vspace{-0.00in}
\end{minipage}
\hspace{0.00in}
\begin{minipage}{0.23\textwidth}
  \subfigure[\small Grid map]{
\includegraphics[width=1.00\textwidth, height = 0.13\textheight]{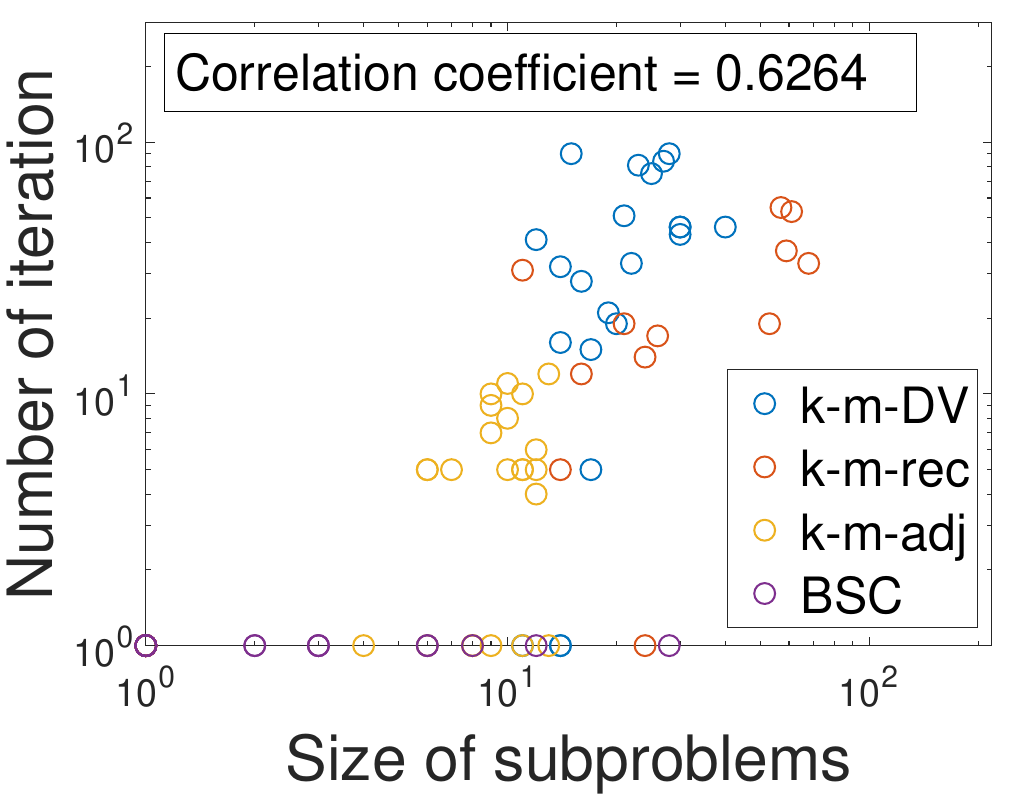}}
\vspace{-0.00in}
\end{minipage}
\hspace{0.00in}
\begin{minipage}{0.23\textwidth}
  \subfigure[\small Text data]{
\includegraphics[width=1.00\textwidth, height = 0.13\textheight]{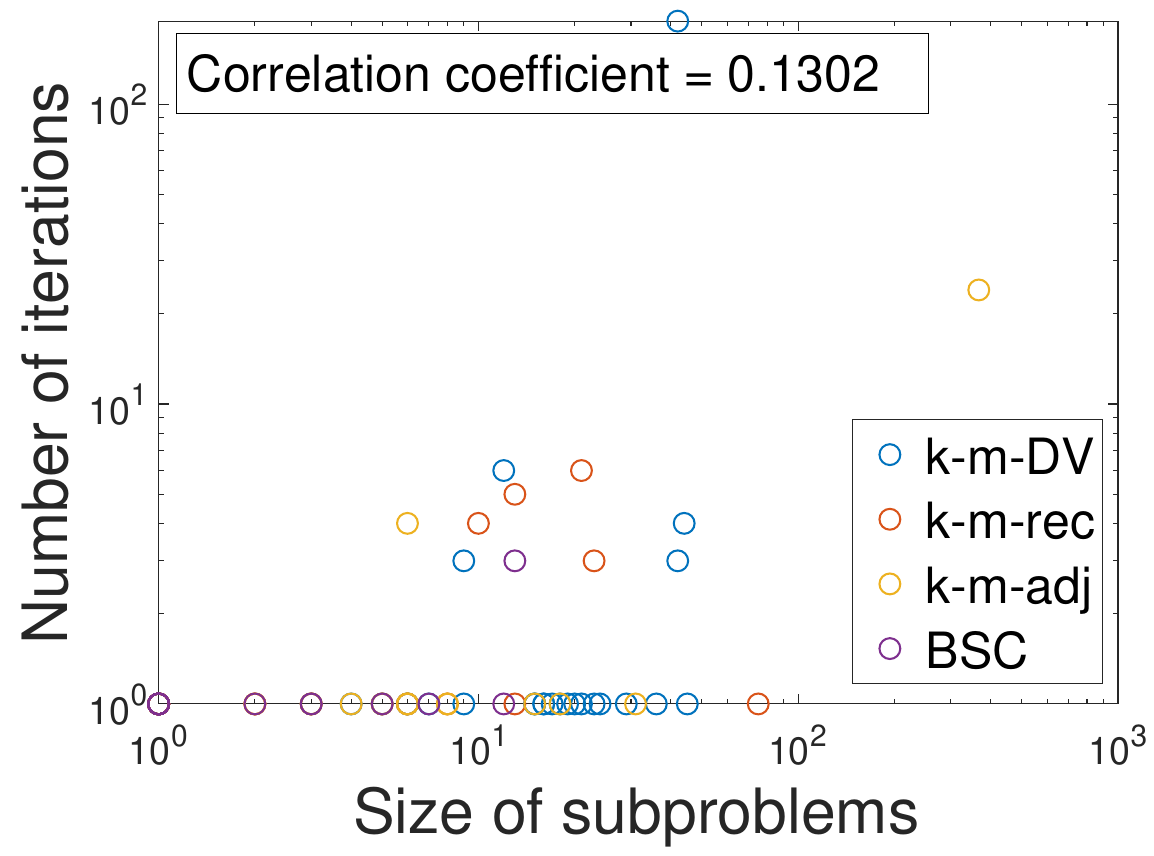}}
\vspace{-0.00in}
\end{minipage}
\hspace{0.00in}
\begin{minipage}{0.23\textwidth}
  \subfigure[\small Synthetic data]{
\includegraphics[width=1.00\textwidth, height = 0.13\textheight]{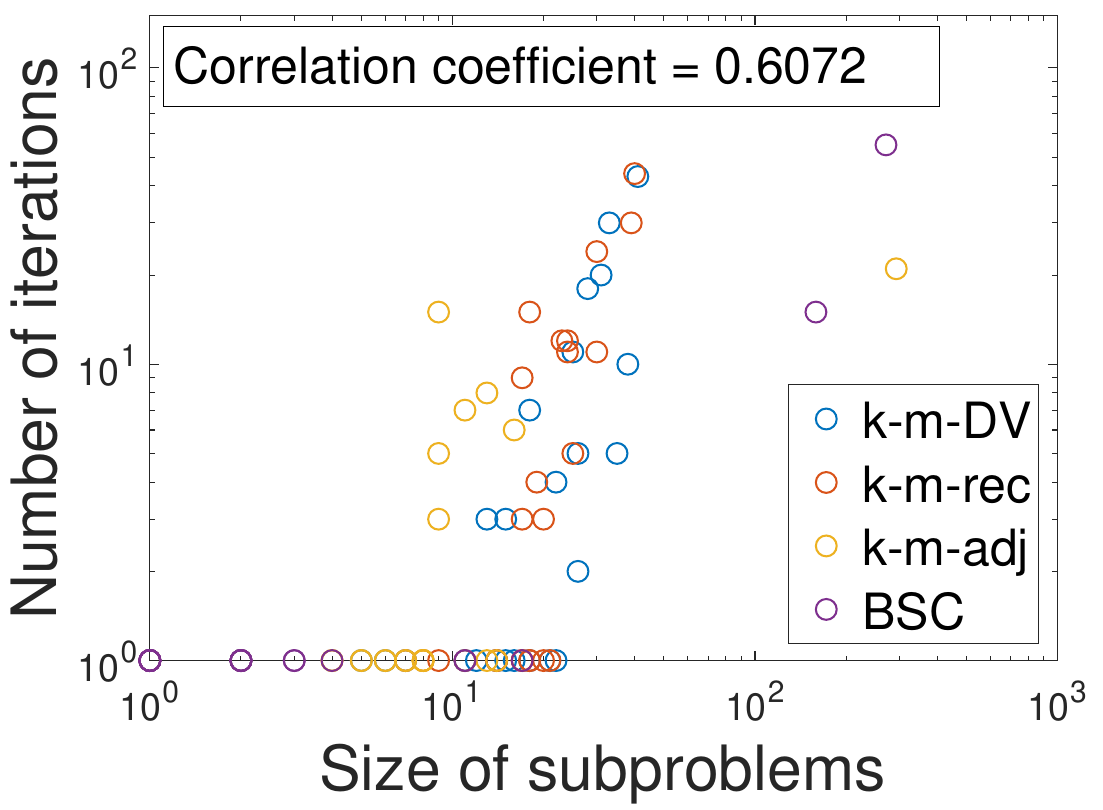}}
\end{minipage}
\vspace{-0.00in}
\caption{\normalsize Subproblem size vs. convergence.}
\label{fig:subproblemvsconvergence1}
\vspace{-0.00in}
\end{figure}

In Table \ref{Tb:exp:time}, our findings indicate that the convergence speed of ``k-m-DV'' surpasses that of the other three dataset partitioning methods. A key contributing factor is the effective balancing of decomposed subproblems by ``k-m-DV'', resulting in the minimization of the size of the ``largest'' subproblem. This reduction in subproblem size plays a crucial role in accelerating convergence, as smaller subproblems require fewer iterations, on average, to reach a feasible solution. In this section, we present additional statistical results to illustrate the specific correlations between subproblem size and their respective convergence time. For simplicity, in what follows, when mentioning the ``convergence time'' of a subproblem, we mean the number of iterations for the subproblem to attain a feasible solution. 

Fig. \ref{fig:subproblemvsconvergence1}(a)(b)(c)(d) depicts the convergence time of subproblems versus their varying problem sizes in the four datasets. As observed in the figures, the correlation coefficients between the convergence time and the subproblem sizes are 0.712, 0.6264, 0.1302, and 0.6072, respectively, suggesting a positive correlation between convergence time and subproblem sizes. This is because larger subproblems, on average, encompass more linear constraints to satisfy, leading to an extended duration in finding a feasible solution. 

In addition, Fig. \ref{fig:subproblemvsconvergence2}(a)(b)(c)(d) shows the correlations between the convergence time of subproblems and the number of boundary records in those subproblems in the four datasets. Not surprisingly, we observe a positive correlation between convergence time and the number of boundary records from the figures. This is because a higher number of boundary records fixed by the MP, on average, imposes stronger constraints on the internal records of the subproblems, which results in smaller feasible regions and, consequently, more iterations for subproblems to find feasible solutions.

\begin{figure}[t]
\centering
\hspace{0.00in}
\begin{minipage}{0.23\textwidth}
  \subfigure[\small Road network]{
\includegraphics[width=1.00\textwidth, height = 0.13\textheight]{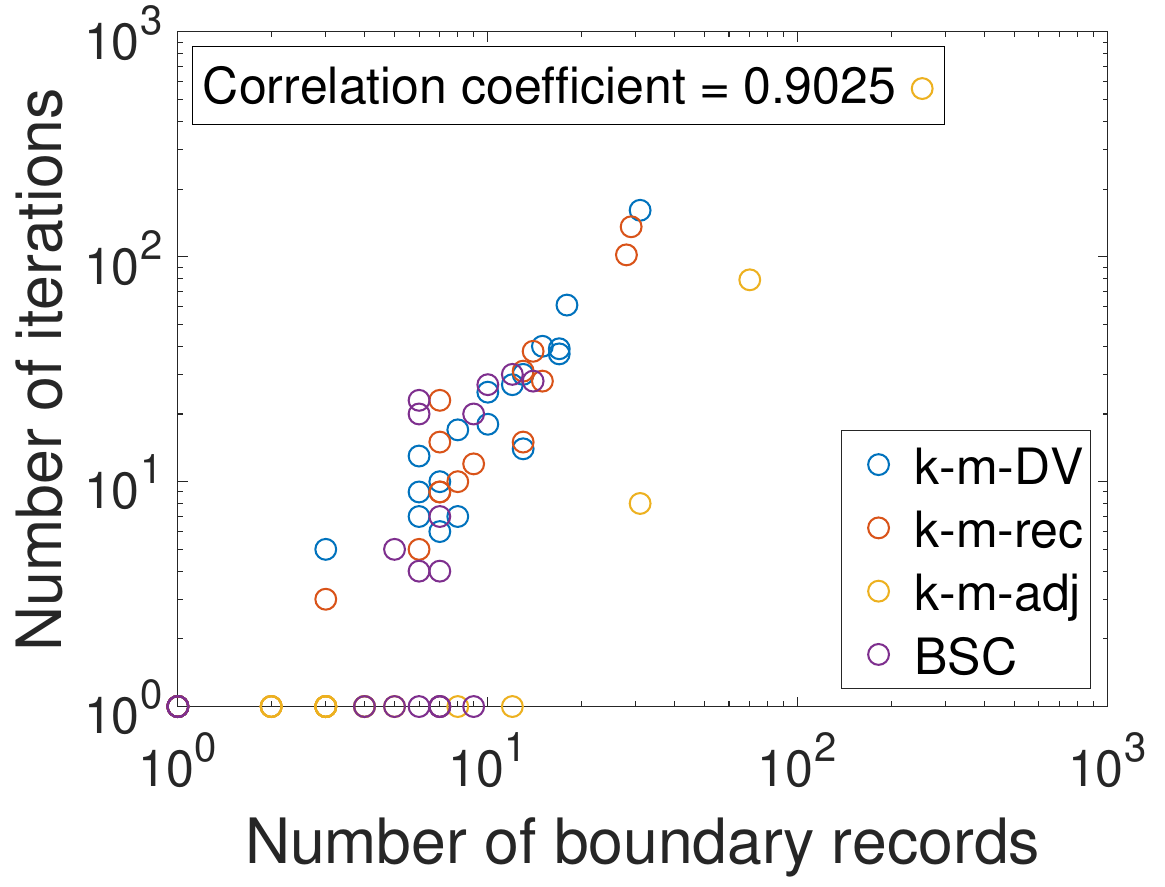}}
\vspace{-0.00in}
\end{minipage}
\hspace{0.00in}
\begin{minipage}{0.23\textwidth}
  \subfigure[\small Grid map]{
\includegraphics[width=1.00\textwidth, height = 0.13\textheight]{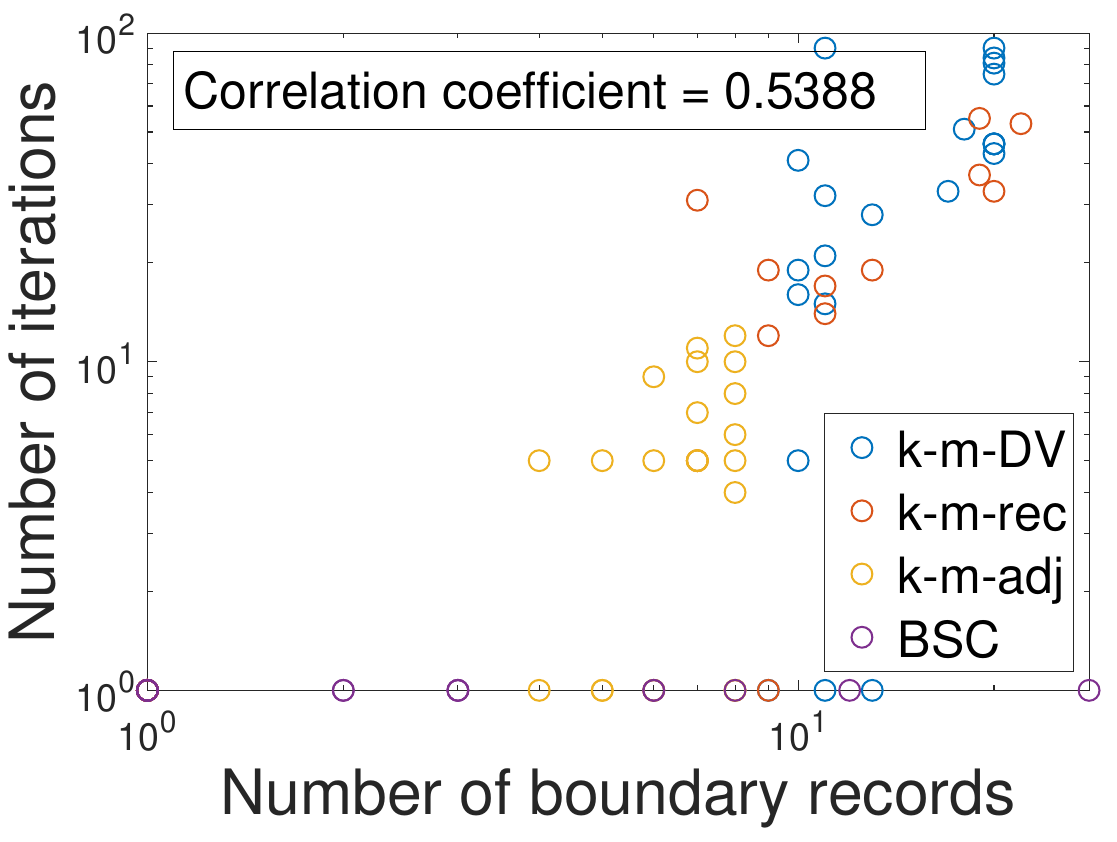}}
\vspace{-0.00in}
\end{minipage}
\hspace{0.00in}
\begin{minipage}{0.23\textwidth}
  \subfigure[\small Text data]{
\includegraphics[width=1.00\textwidth, height = 0.13\textheight]{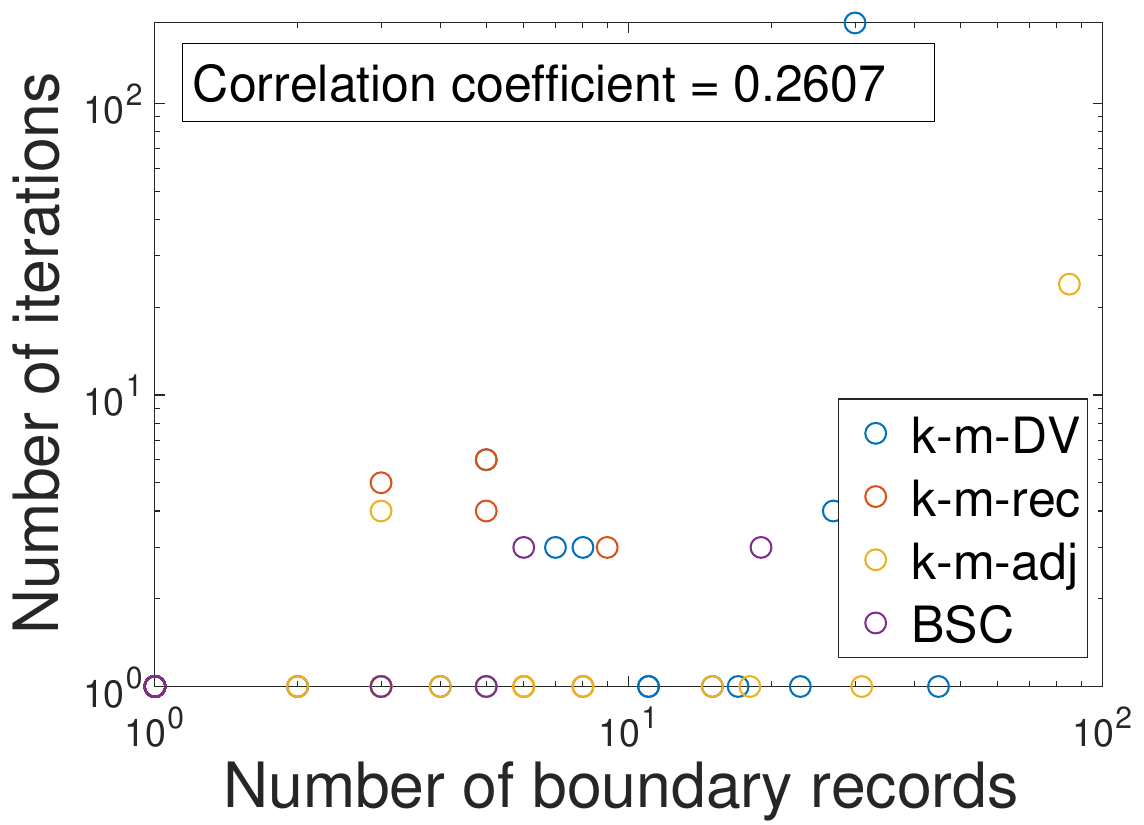}}
\vspace{-0.00in}
\end{minipage}
\hspace{0.00in}
\begin{minipage}{0.23\textwidth}
  \subfigure[\small Synthetic data]{
\includegraphics[width=1.00\textwidth, height = 0.13\textheight]{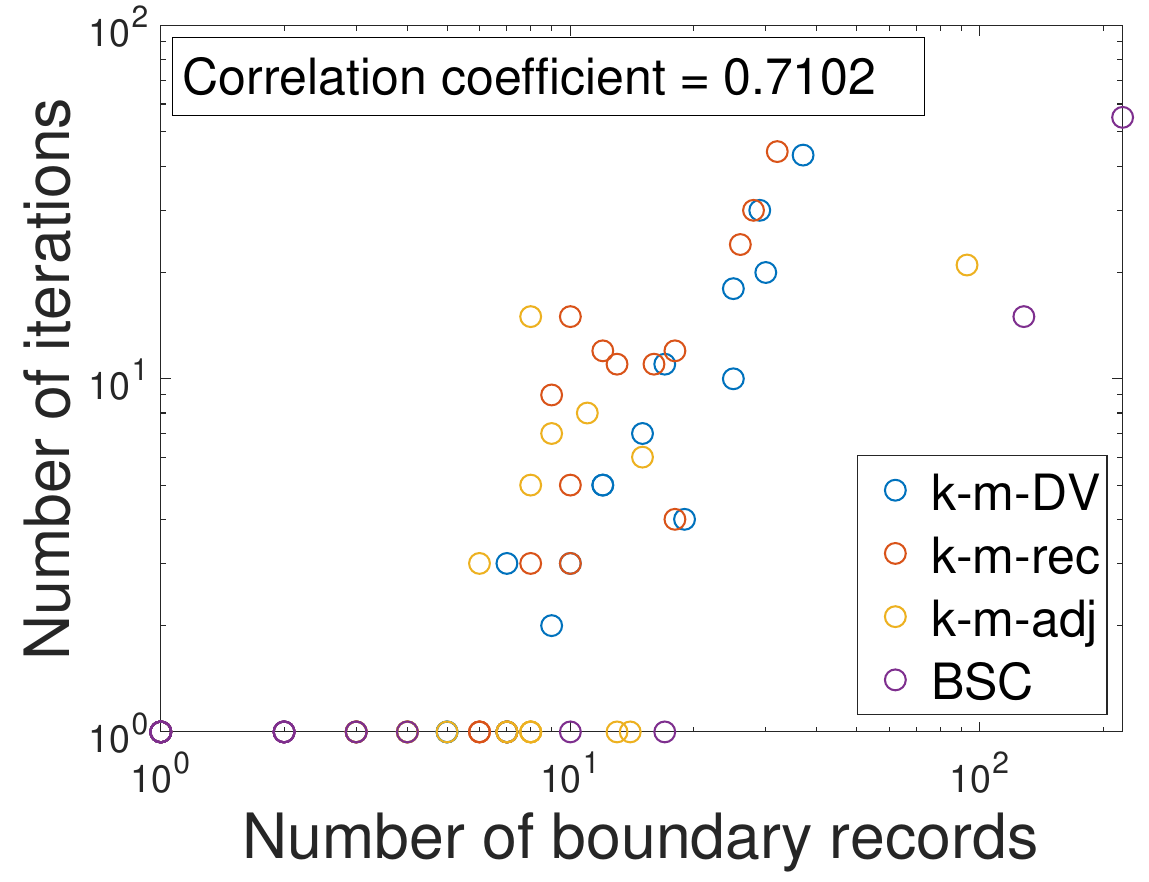}}
\end{minipage}
\vspace{-0.00in}
\caption{\normalsize Boundary record size in subproblems vs. convergence.}
\label{fig:subproblemvsconvergence2}
\vspace{-0.00in}
\end{figure}

\subsection{Computation Time of MP Components and Subproblems}
\label{subsec:computetime}

\begin{figure}[h]
\centering
\hspace{0.00in}
\begin{minipage}{0.23\textwidth}
  \subfigure[\small Road network]{
\includegraphics[width=1.00\textwidth, height = 0.13\textheight]{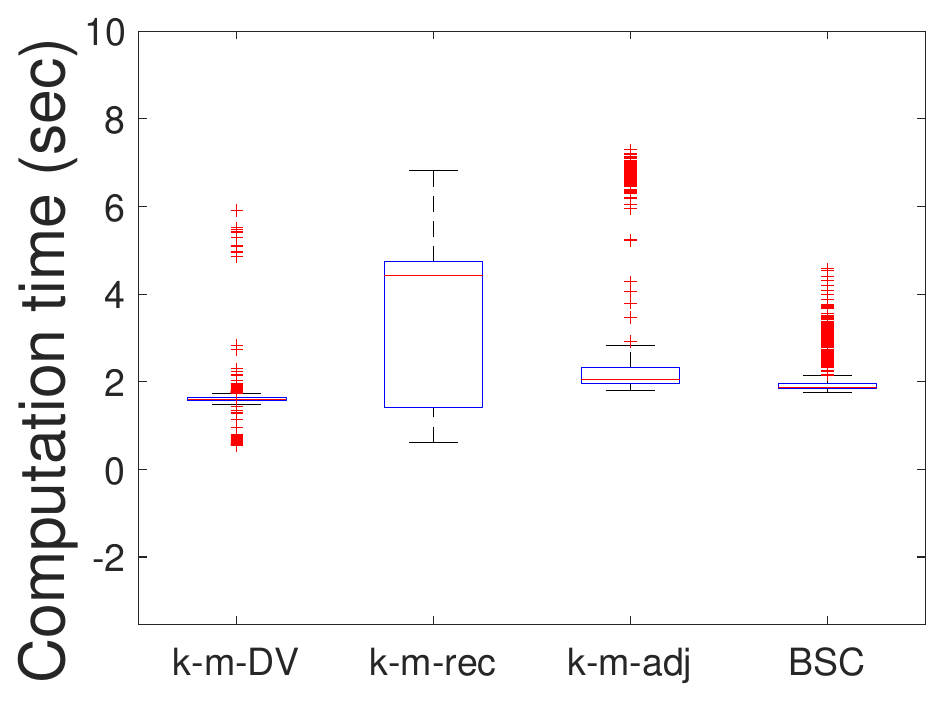}}
\vspace{-0.00in}
\end{minipage}
\hspace{0.00in}
\begin{minipage}{0.23\textwidth}
  \subfigure[\small Grid maps]{
\includegraphics[width=1.00\textwidth, height = 0.13\textheight]{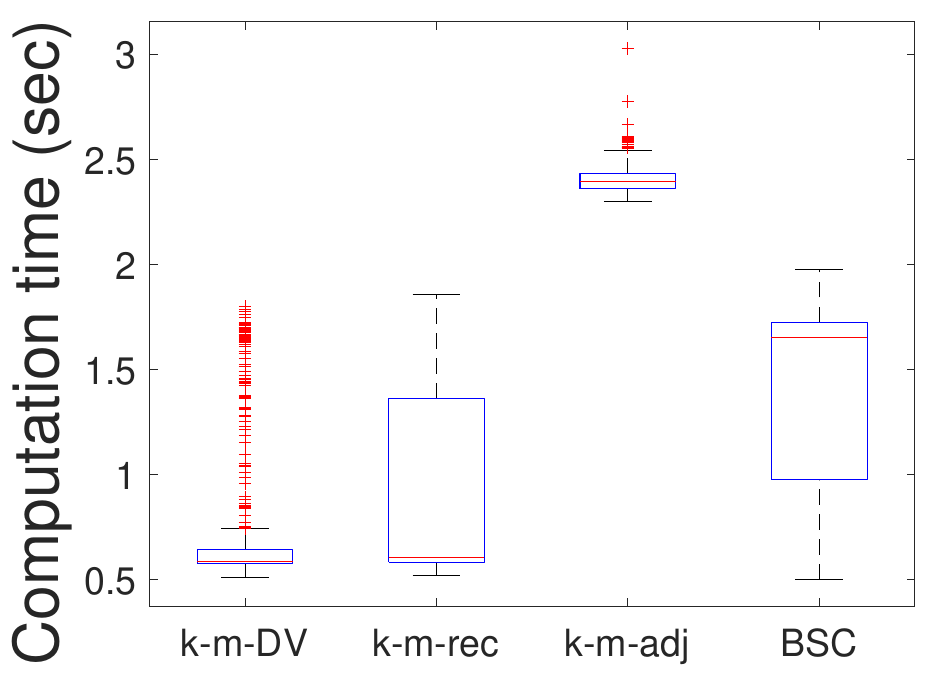}}
\vspace{-0.00in}
\end{minipage}
\begin{minipage}{0.23\textwidth}
  \subfigure[\small Text data]{
\includegraphics[width=1.00\textwidth, height = 0.13\textheight]{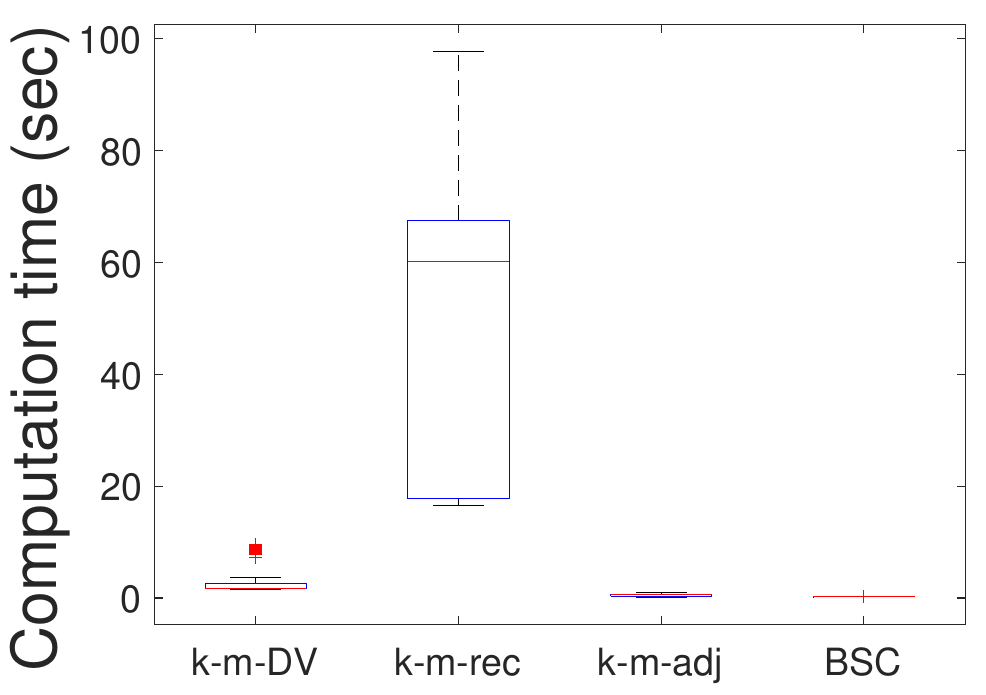}}
\vspace{-0.00in}
\end{minipage}
\hspace{0.00in}
\begin{minipage}{0.23\textwidth}
  \subfigure[\small Synthetic data]{
\includegraphics[width=1.00\textwidth, height = 0.13\textheight]{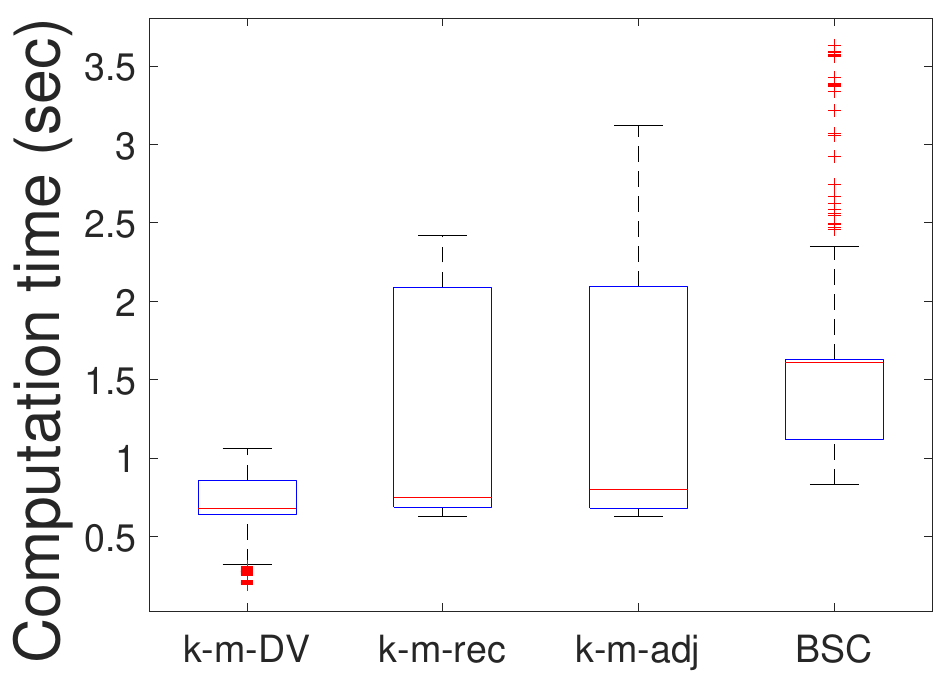}}
\vspace{-0.00in}
\end{minipage}
\caption{\normalsize Computation time of suproblems.}
\label{fig:Sub}
\vspace{-0.00in}
\end{figure}

\begin{figure}[h]
\centering
\hspace{0.00in}
\begin{minipage}{0.23\textwidth}
  \subfigure[\small Road network]{
\includegraphics[width=1.00\textwidth, height = 0.13\textheight]{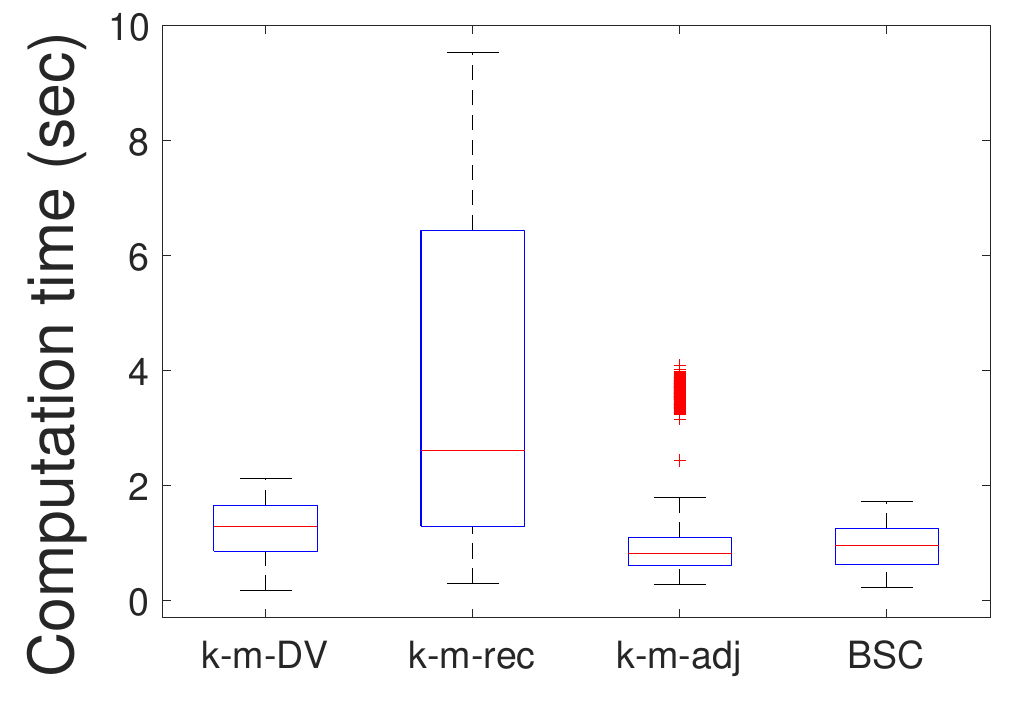}}
\vspace{-0.00in}
\end{minipage}
\hspace{0.00in}
\begin{minipage}{0.23\textwidth}
  \subfigure[\small Grid maps]{
\includegraphics[width=1.00\textwidth, height = 0.13\textheight]{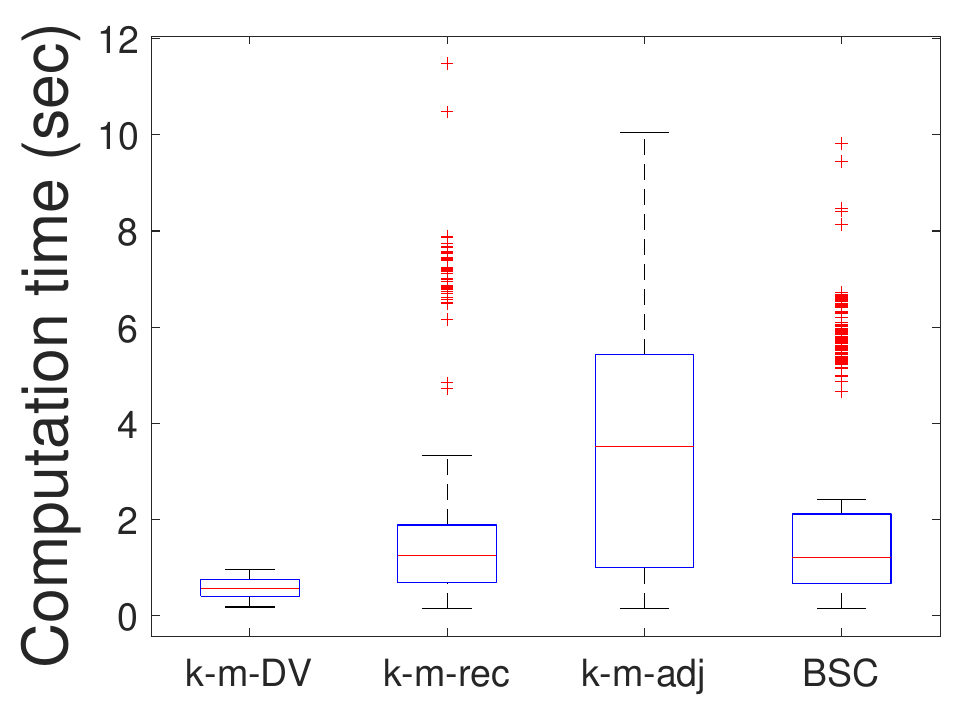}}
\vspace{-0.00in}
\end{minipage}
\begin{minipage}{0.23\textwidth}
  \subfigure[\small Text data]{
\includegraphics[width=1.00\textwidth, height = 0.13\textheight]{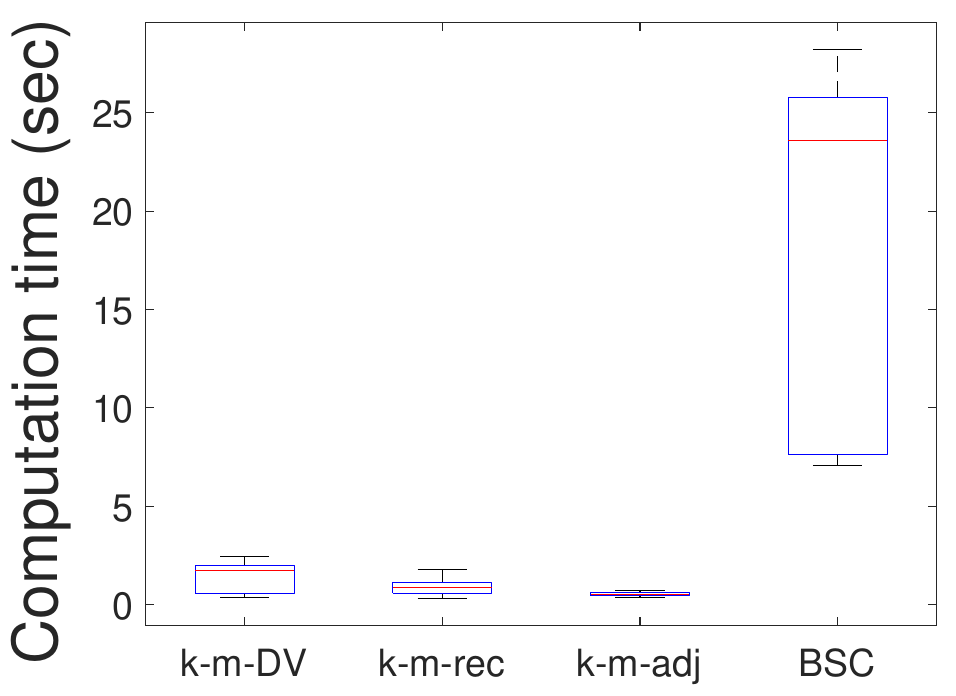}}
\vspace{-0.00in}
\end{minipage}
\hspace{0.00in}
\begin{minipage}{0.23\textwidth}
  \subfigure[\small Synthetic data]{
\includegraphics[width=1.00\textwidth, height = 0.13\textheight]{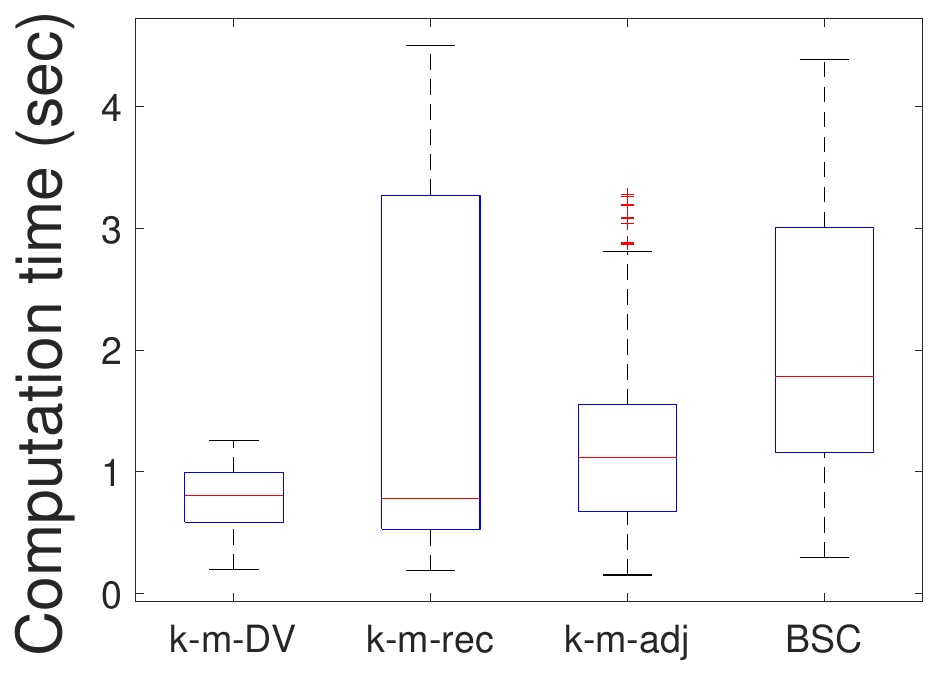}}
\end{minipage}
\vspace{-0.00in}
\caption{\normalsize Computation time of MPs.}
\label{fig:MPs}
\vspace{-0.00in}
\end{figure}

Notably, the computation time of BD is impacted by not only the convergence time but also the computation time of decomposed MP and subproblems. In this part, we depict the boxplot of the computation time of subproblems (in Fig. \ref{fig:Sub}(a)(b)(c)(d)) and MP (in Fig. \ref{fig:MPs}(a)(b)(c)(d)) when applying different dataset partitioning methods. 

From the figures, it is evident that ``k-m-DV'' has a lower computation time for both subproblems and MPs across the four datasets. Specifically, its average computation time for MPs and subproblems is 1.505 seconds and 1.194 seconds, respectively. This is because the subproblems and MP components partitioned by ``k-m-DV'' are sufficiently small problems, as illustrated in Table \ref{Tb:exp:partitionSub} and Table \ref{Tb:exp:partitionMP}. This partitioning results in shorter computation times for these decomposed problems.


\subsection{Dataset Partitioning}
\label{subsec:datasetpartition}
Recall that Table \ref{Tb:exp:partitionSub} and Table \ref{Tb:exp:partitionMP} in Section \ref{subsec:expBalance} present a comparison of the sizes of subproblems and MP components among the four dataset partitioning algorithms when the number of subsets $M$ is set by 25.  In Fig. \ref{fig:problemsizevsnr_agents}(a)--(d), we increase $M$ from 25 to 34 and evaluate how the average size of MP components and the average number of internal records in subproblems change in the four datasets. In Fig. \ref{fig:problemsizevsnr_agents}, we only use ``k-m-DV'' as the dataset partitioning algorithm. From the figures, we find that with the increase of $M$, the size of subproblems decreases. This is because, on average, with more subproblems partitioned, there is a lower number of records in each subproblem. However, the size of MP components does not necessarily decrease with the increase in $M$. Particularly, when using the geo-location dataset in grid maps, the size of MP components even increases when $M$ increases. This is because with more subsets partitioned, there are more boundary records across different subsets, leading to a larger MP and, consequently, a higher likelihood of generating larger MP components.

\begin{figure}[t]
\centering
\hspace{0.00in}
\begin{minipage}{0.23\textwidth}
  \subfigure[\small Road network]{
\includegraphics[width=1.00\textwidth, height = 0.13\textheight]{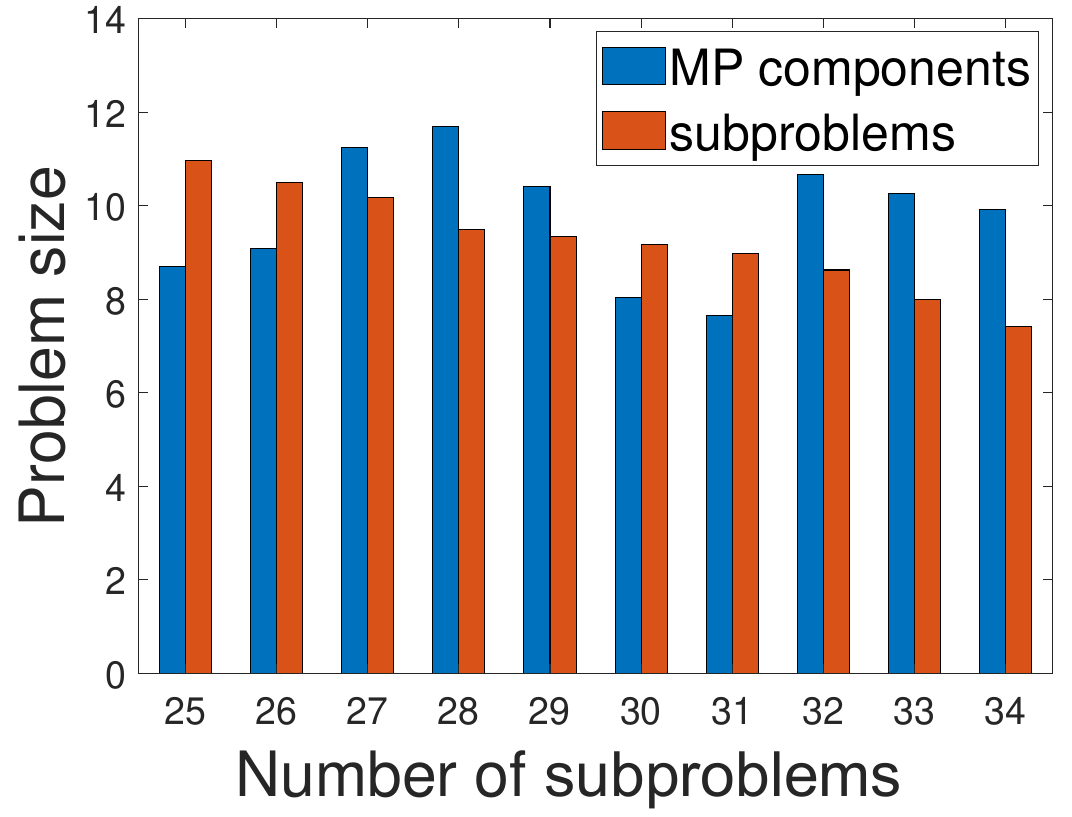}}
\vspace{-0.00in}
\end{minipage}
\hspace{0.00in}
\begin{minipage}{0.23\textwidth}
  \subfigure[\small Grid map]{
\includegraphics[width=1.00\textwidth, height = 0.13\textheight]{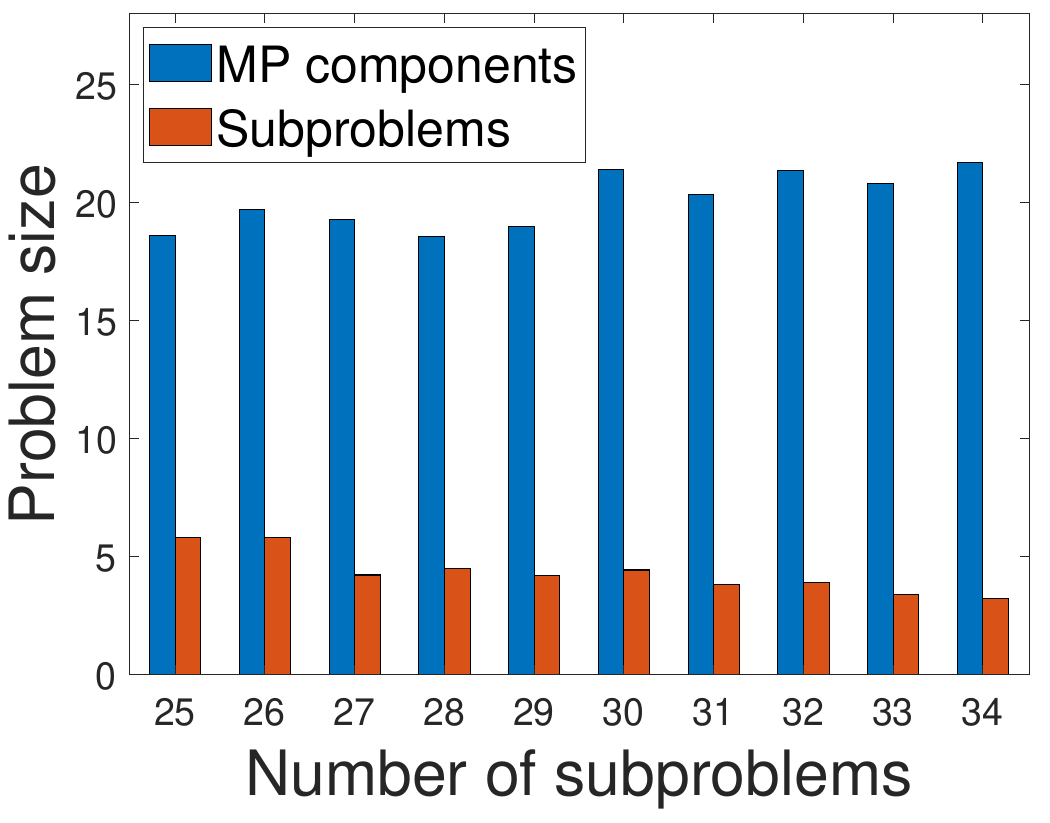}}
\vspace{-0.00in}
\end{minipage}
\begin{minipage}{0.23\textwidth}
  \subfigure[\small Text data]{
\includegraphics[width=1.00\textwidth, height = 0.13\textheight]{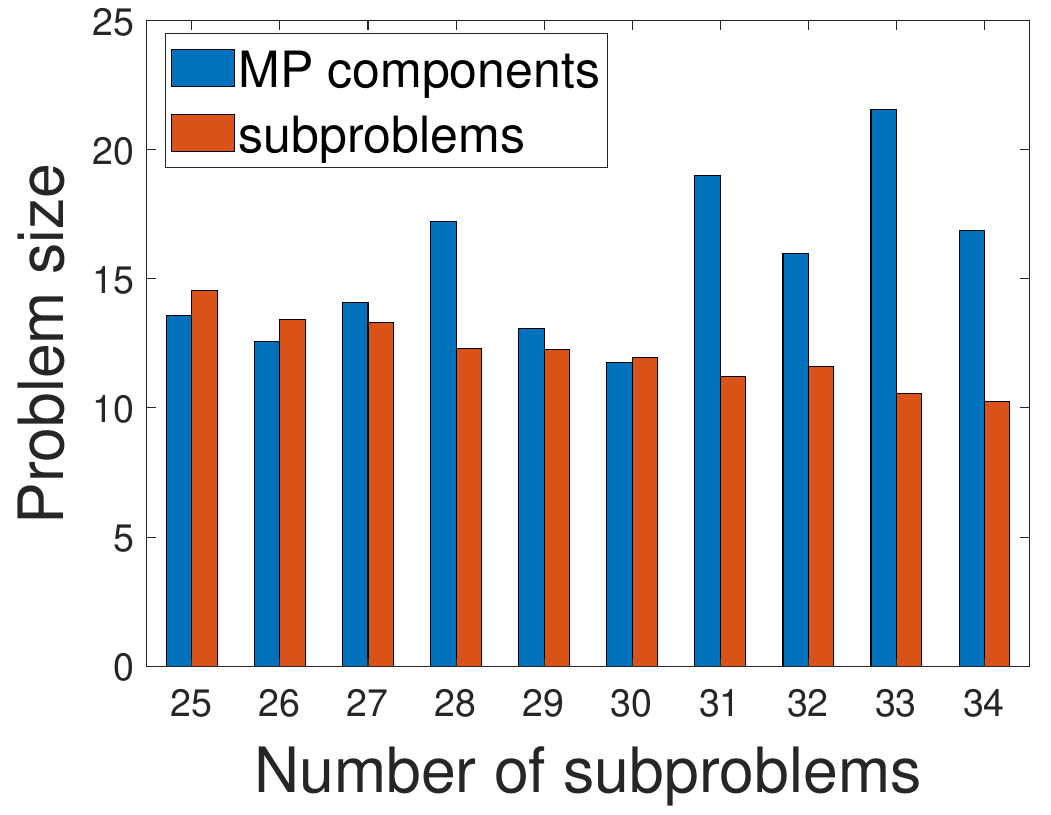}}
\vspace{-0.00in}
\end{minipage}
\hspace{0.00in}
\begin{minipage}{0.23\textwidth}
  \subfigure[\small Synthetic data]{
\includegraphics[width=1.00\textwidth, height = 0.13\textheight]{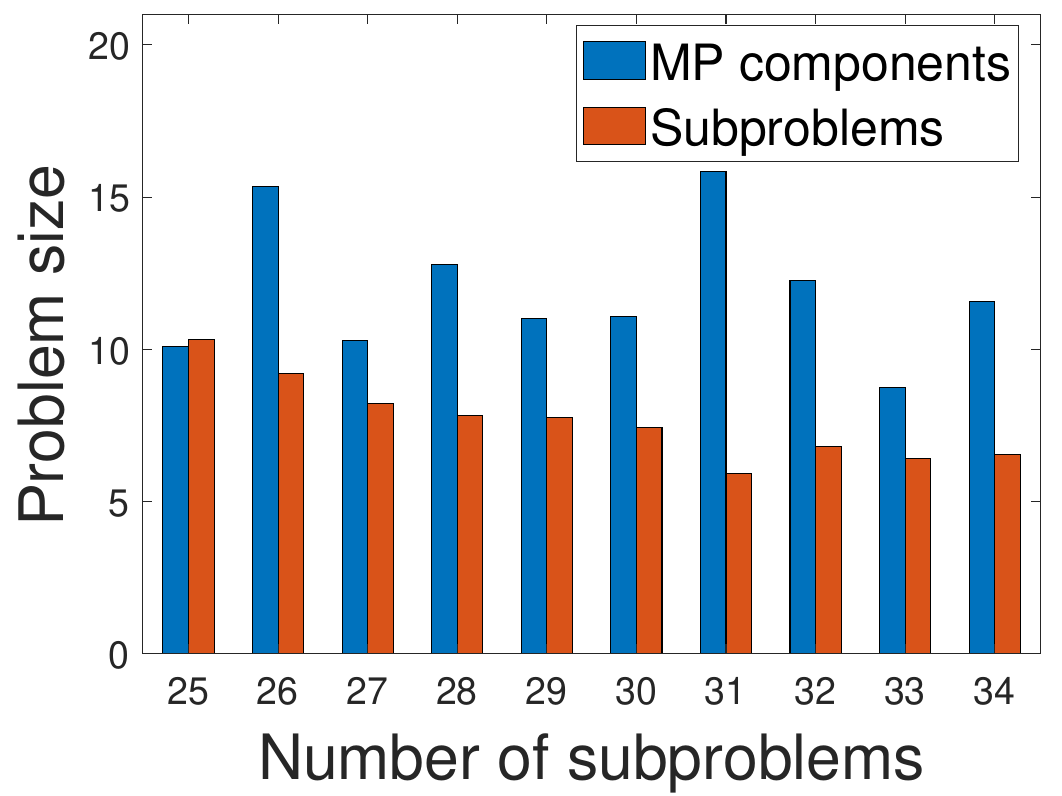}}
\vspace{-0.00in}
\end{minipage}
\vspace{-0.00in}
\caption{\normalsize Problem size vs. number of subproblems.}
\label{fig:problemsizevsnr_agents}
\vspace{-0.00in}
\end{figure}

In Fig. \ref{fig:cluster_rome_downtown}--Fig. \ref{fig:cluster_random}, we give examples to visually illustrate how the secret records are partitioned by various algorithms. Specifically, 

\begin{itemize}
\item Fig. \ref{fig:cluster_rome_downtown}(a)--(d) shows how \textbf{500 geo-locations in the road network} are partitioned using the four methods.
\item Fig. \ref{fig:cluster_grid}(a)--(d) shows how \textbf{500 geo-locations in a grid map} are partitioned using the four methods.
\item Fig. \ref{fig:cluster_rome_text}(a)--(d) shows how \textbf{500 text records} are partitioned using the four methods.
\item Fig. \ref{fig:cluster_random}(a)--(d) shows how \textbf{500 3-dimensional synthetic records} are partitioned using the four methods.
\end{itemize}

\begin{figure*}[t]
\centering
\hspace{0.00in}
\begin{minipage}{0.235\textwidth}
  \subfigure[\small k-m-DV]{
\includegraphics[width=1.00\textwidth, height = 0.14\textheight]{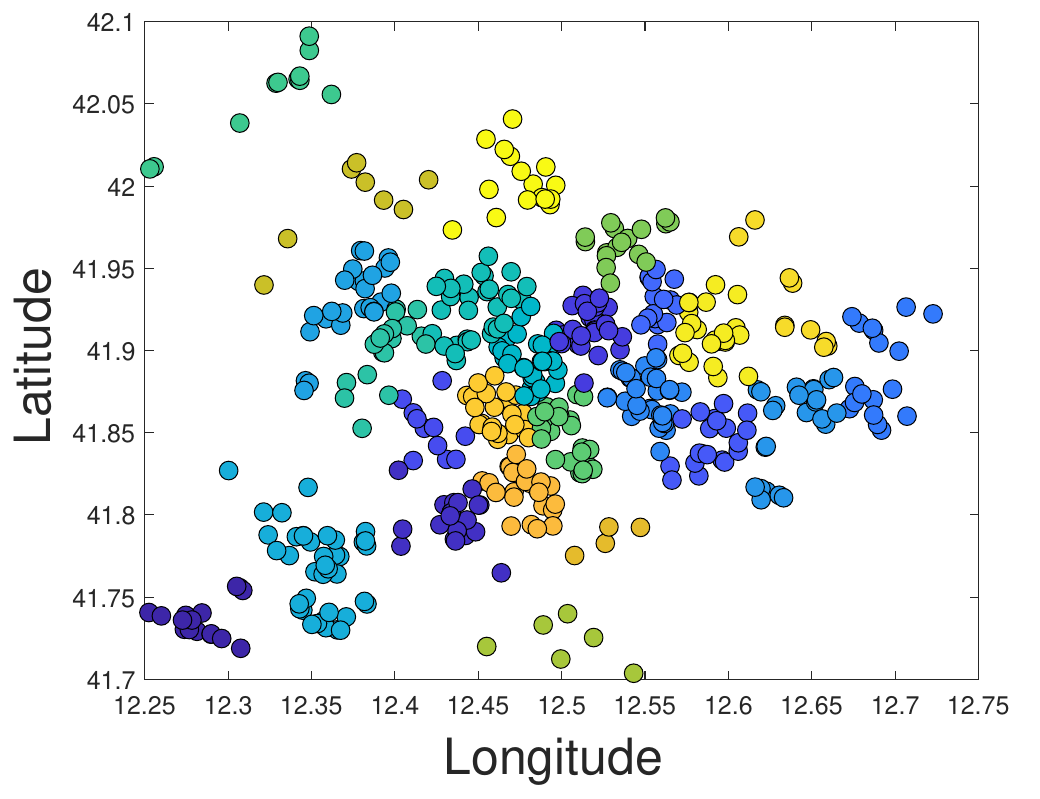}}
\vspace{-0.10in}
\end{minipage}
\hspace{0.00in}
\begin{minipage}{0.235\textwidth}
  \subfigure[\small
  k-m-rec]{
\includegraphics[width=1.00\textwidth, height = 0.14\textheight]{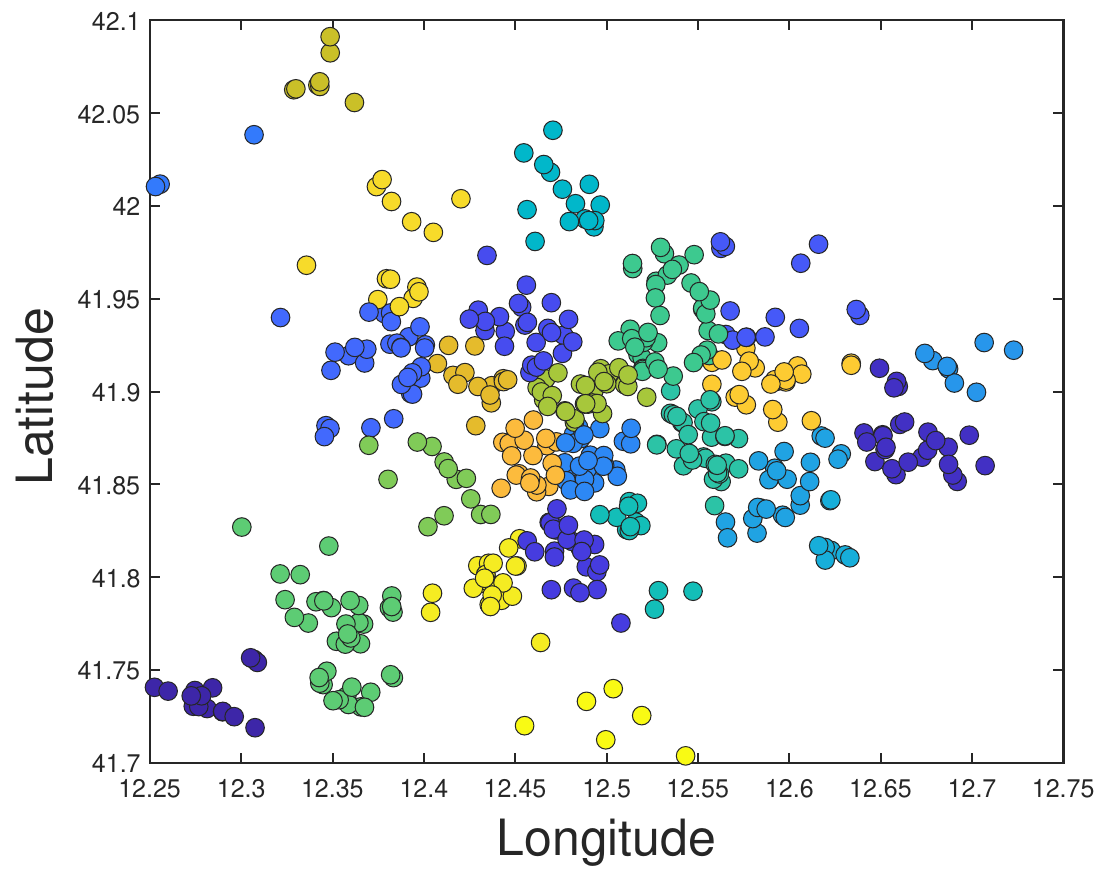}}
\vspace{-0.10in}
\end{minipage}
\hspace{0.00in}
\begin{minipage}{0.235\textwidth}
  \subfigure[\small k-m-adj]{
\includegraphics[width=1.00\textwidth, height = 0.14\textheight]{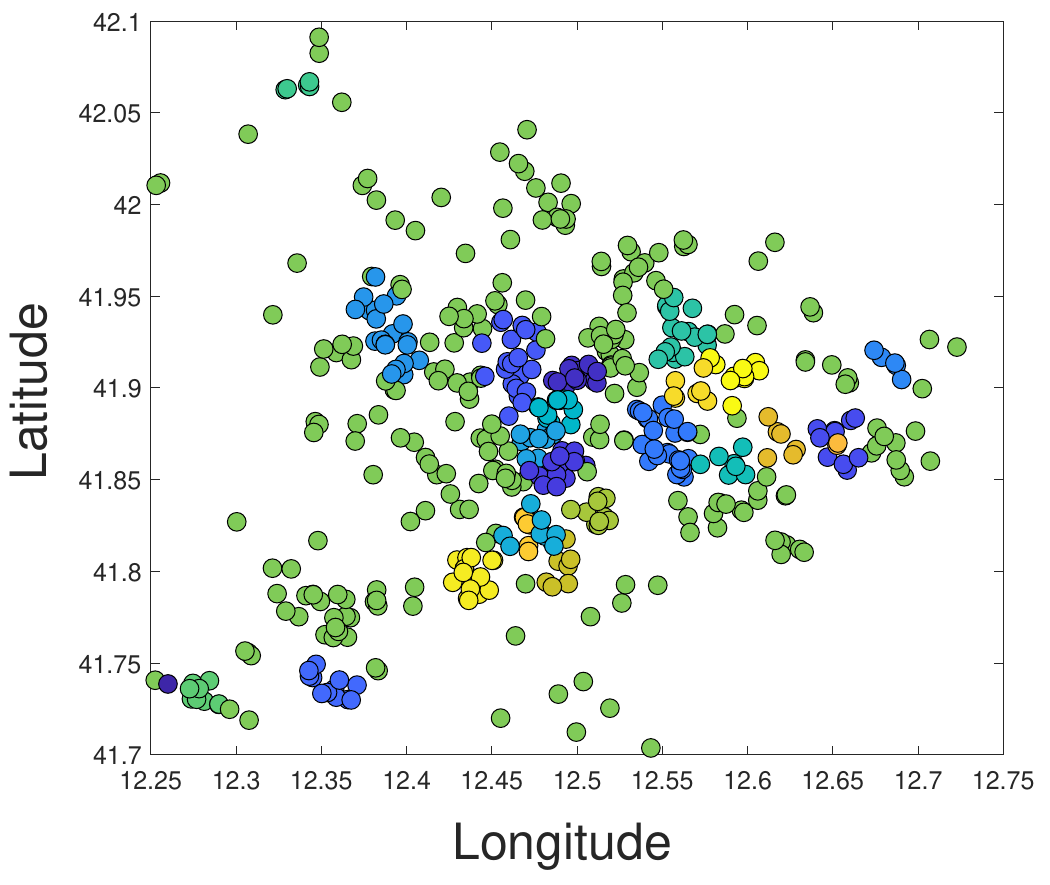}}
\vspace{-0.10in}
\end{minipage}
\hspace{0.00in}
\begin{minipage}{0.235\textwidth}
  \subfigure[\small BSC]{
\includegraphics[width=1.00\textwidth, height = 0.14\textheight]{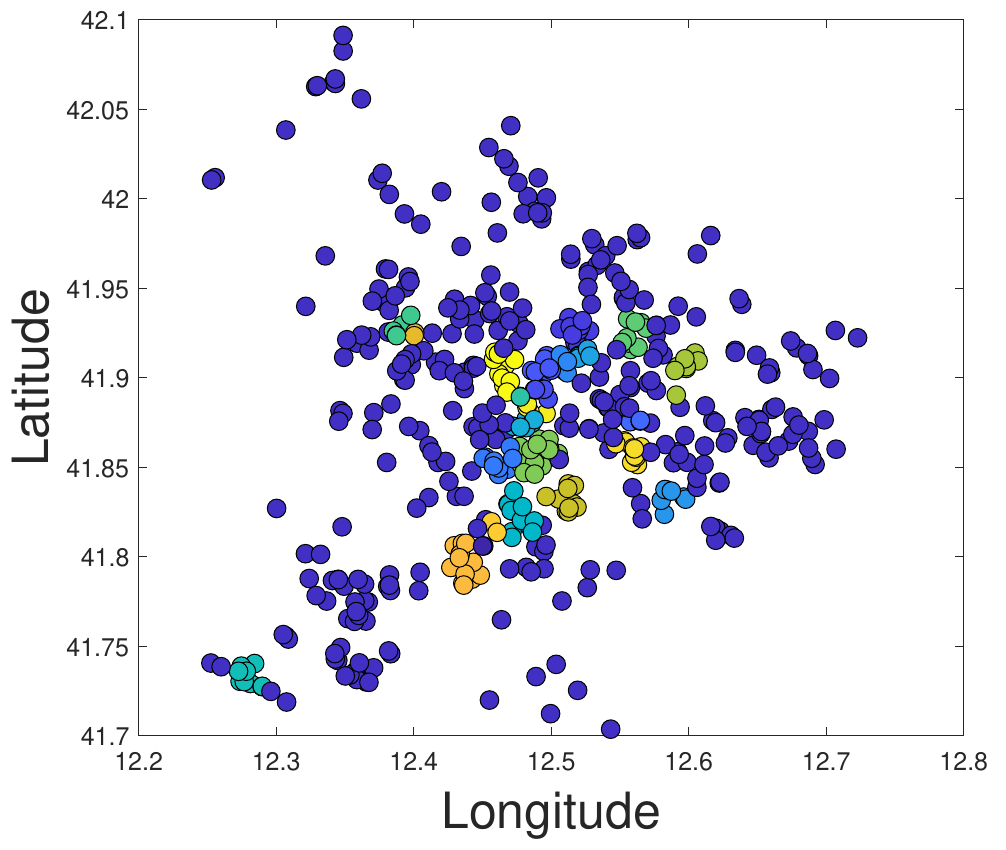}}
\vspace{-0.1in}
\end{minipage}
\caption{\normalsize Example: Dataset partitioning (Geo-location data in the road network).}
\label{fig:cluster_rome_downtown}
\vspace{-0.00in}
\vspace{0.10in}
\begin{minipage}{0.235\textwidth}
  \subfigure[\small k-m-DV]{
\includegraphics[width=1.00\textwidth, height = 0.14\textheight]{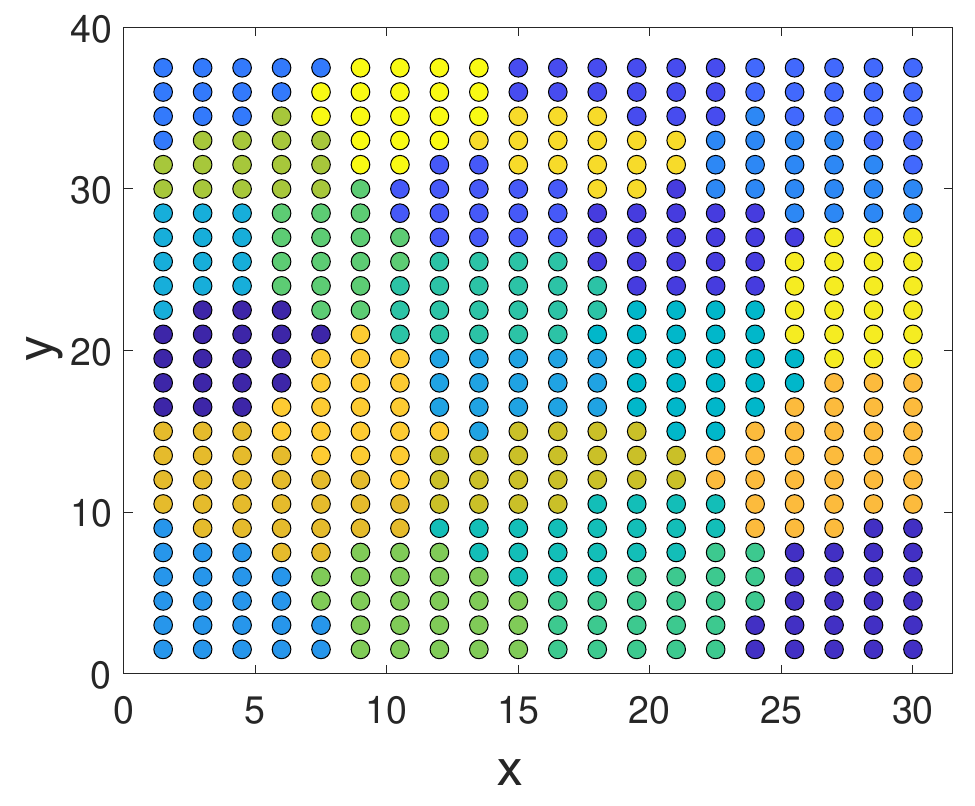}}
\vspace{-0.10in}
\end{minipage}
\hspace{0.00in}
\begin{minipage}{0.235\textwidth}
  \subfigure[\small k-m-rec]{
\includegraphics[width=1.00\textwidth, height = 0.14\textheight]{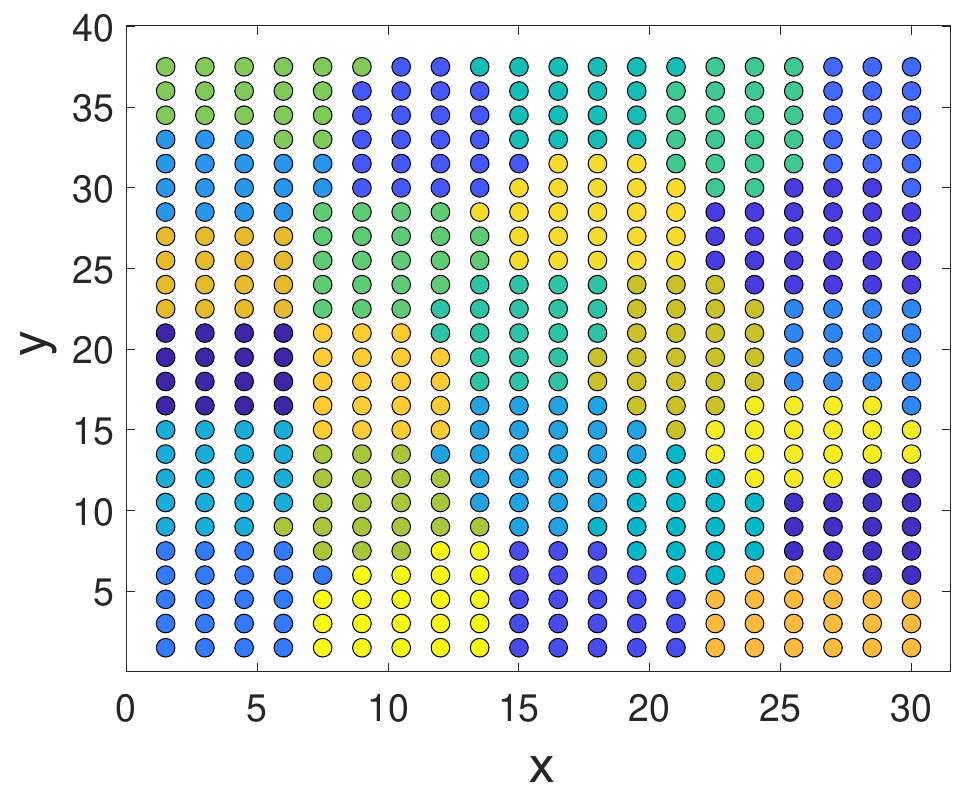}}
\vspace{-0.10in}
\end{minipage}
\hspace{0.00in}
\begin{minipage}{0.235\textwidth}
  \subfigure[\small k-m-adj]{
\includegraphics[width=1.00\textwidth, height = 0.14\textheight]{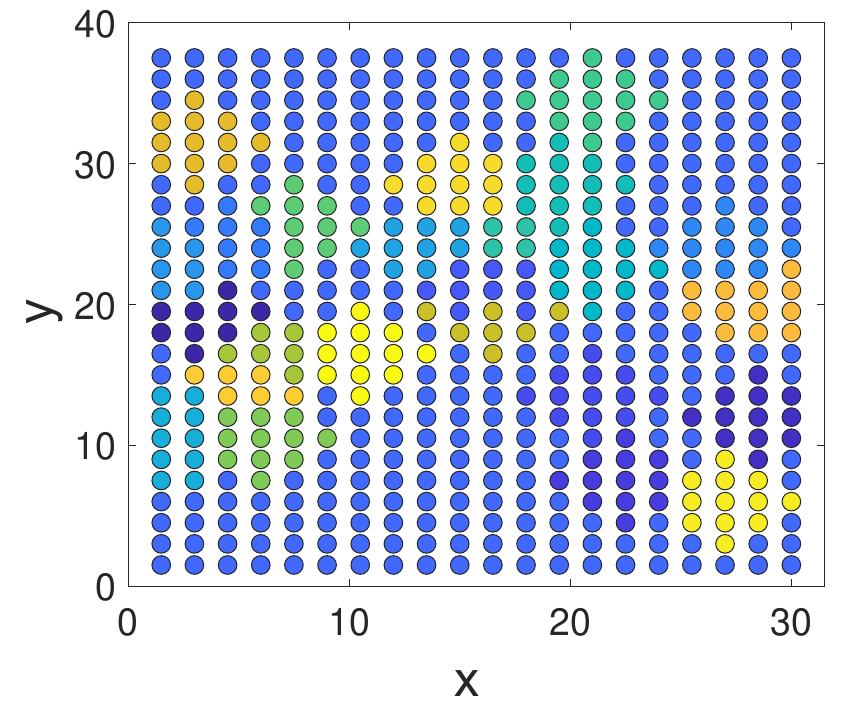}}
\vspace{-0.10in}
\end{minipage}
\hspace{0.00in}
\begin{minipage}{0.235\textwidth}
  \subfigure[\small BSC]{
\includegraphics[width=1.00\textwidth, height = 0.14\textheight]{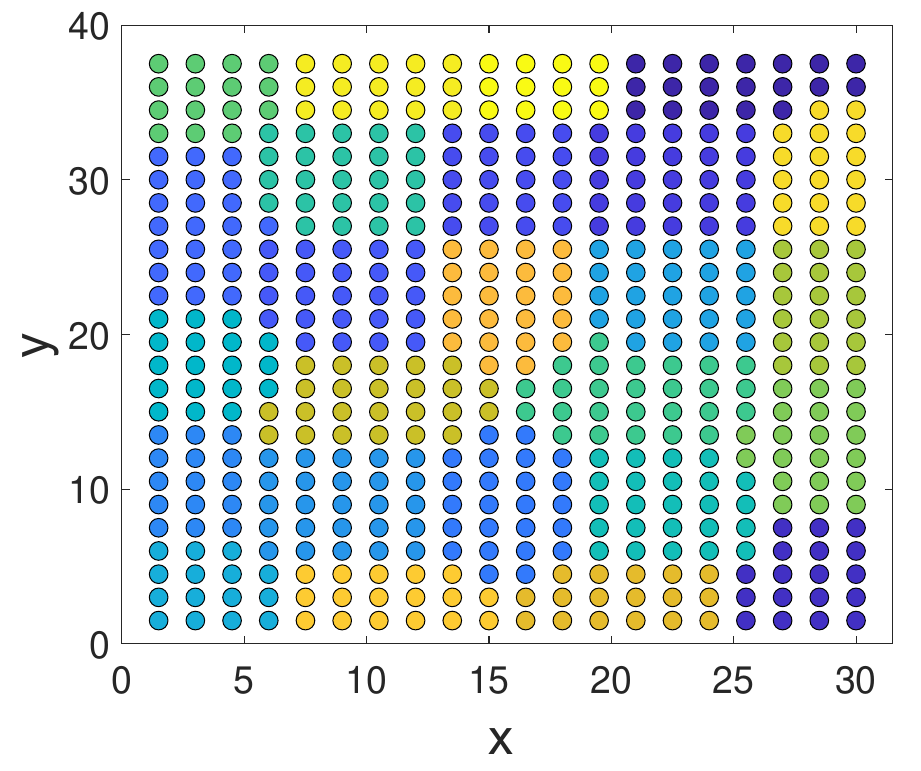}}
\vspace{-0.1in}
\end{minipage}
\caption{\normalsize Example: Dataset partitioning (Geo-location data in a grid map).}
\label{fig:cluster_grid}
\vspace{0.10in}
\begin{minipage}{0.235\textwidth}
  \subfigure[\small k-m-DV]{
\includegraphics[width=1.00\textwidth, height = 0.14\textheight]{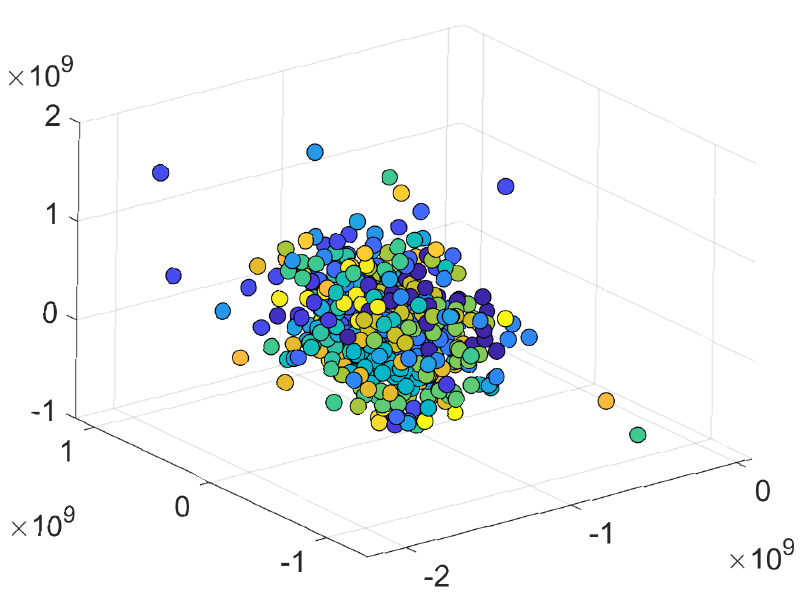}}
\vspace{-0.10in}
\end{minipage}
\hspace{0.00in}
\begin{minipage}{0.235\textwidth}
  \subfigure[\small k-m-rec]{
\includegraphics[width=1.00\textwidth, height = 0.14\textheight]{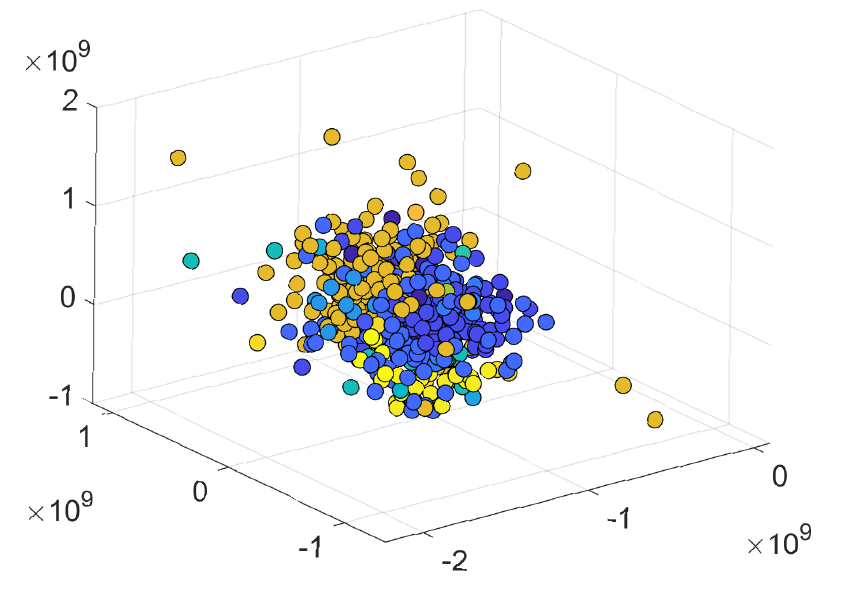}}
\vspace{-0.10in}
\end{minipage}
\hspace{0.00in}
\begin{minipage}{0.235\textwidth}
  \subfigure[\small k-m-adj]{
\includegraphics[width=1.00\textwidth, height = 0.14\textheight]{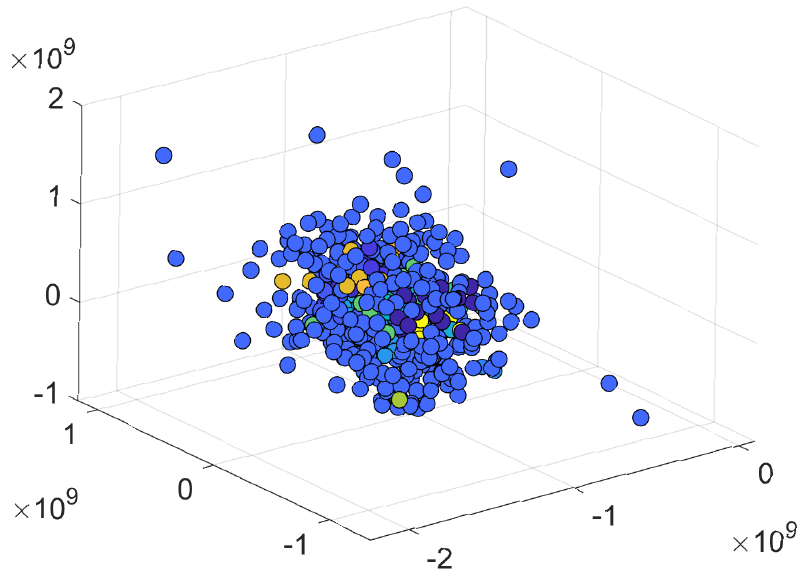}}
\vspace{-0.10in}
\end{minipage}
\hspace{0.00in}
\begin{minipage}{0.235\textwidth}
  \subfigure[\small BSC]{
\includegraphics[width=1.00\textwidth, height = 0.14\textheight]{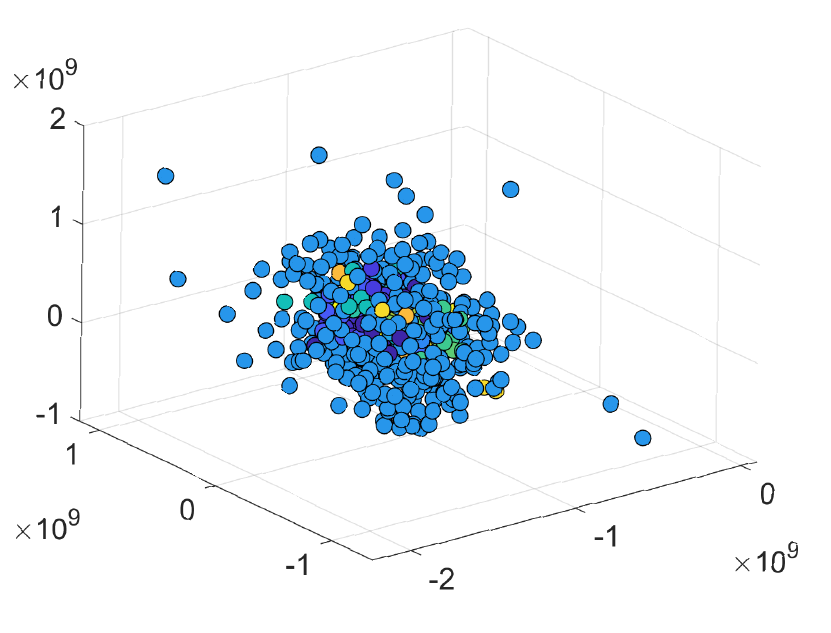}}
\vspace{-0.1in}
\end{minipage}
\caption{\normalsize Example: Dataset partitioning (Text dataset).}
\label{fig:cluster_rome_text}
\vspace{-0.00in}
\vspace{0.10in}
\begin{minipage}{0.235\textwidth}
  \subfigure[\small k-m-DV]{
\includegraphics[width=1.00\textwidth, height = 0.14\textheight]{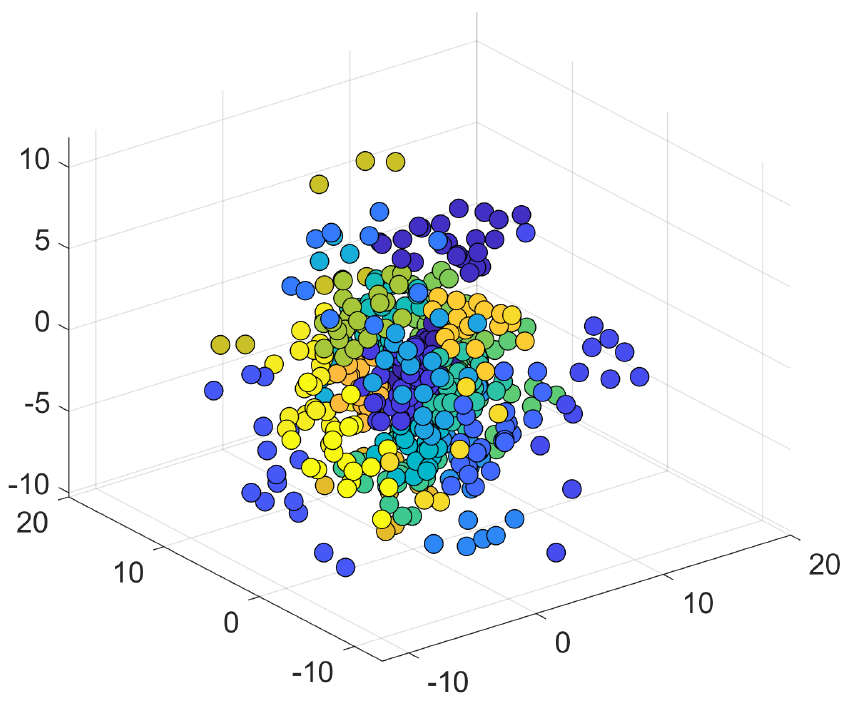}}
\vspace{-0.10in}
\end{minipage}
\hspace{0.00in}
\begin{minipage}{0.235\textwidth}
  \subfigure[\small k-m-rec]{
\includegraphics[width=1.00\textwidth, height = 0.14\textheight]{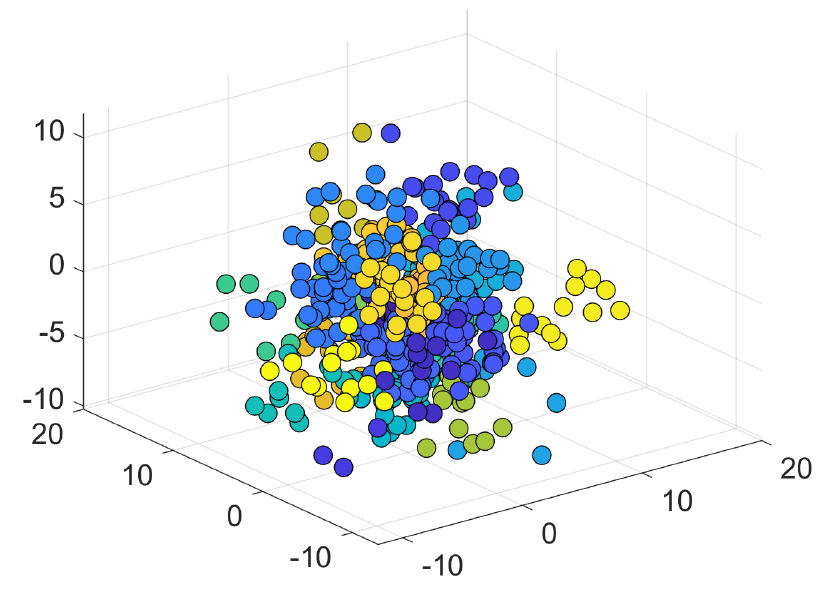}}
\vspace{-0.10in}
\end{minipage}
\hspace{0.00in}
\begin{minipage}{0.235\textwidth}
  \subfigure[\small k-m-adj]{
\includegraphics[width=1.00\textwidth, height = 0.14\textheight]{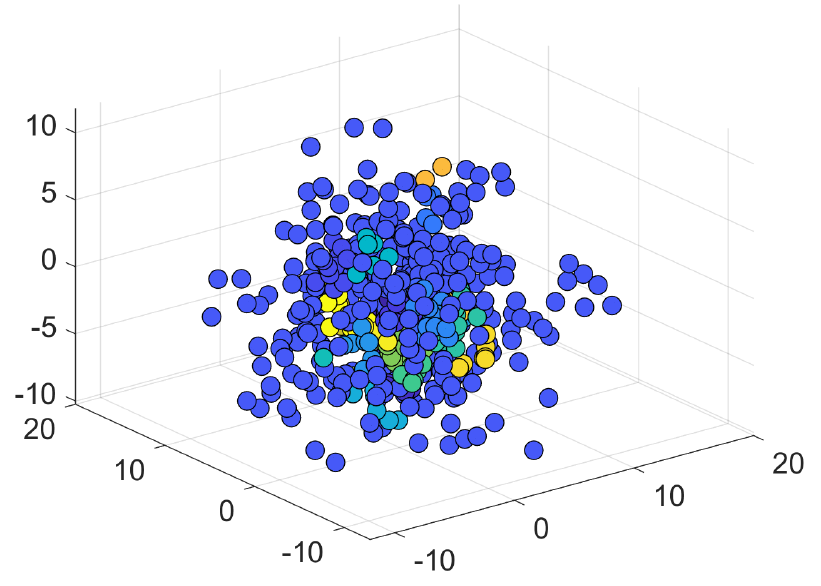}}
\vspace{-0.10in}
\end{minipage}
\hspace{0.00in}
\begin{minipage}{0.235\textwidth}
  \subfigure[\small BSC]{
\includegraphics[width=1.00\textwidth, height = 0.14\textheight]{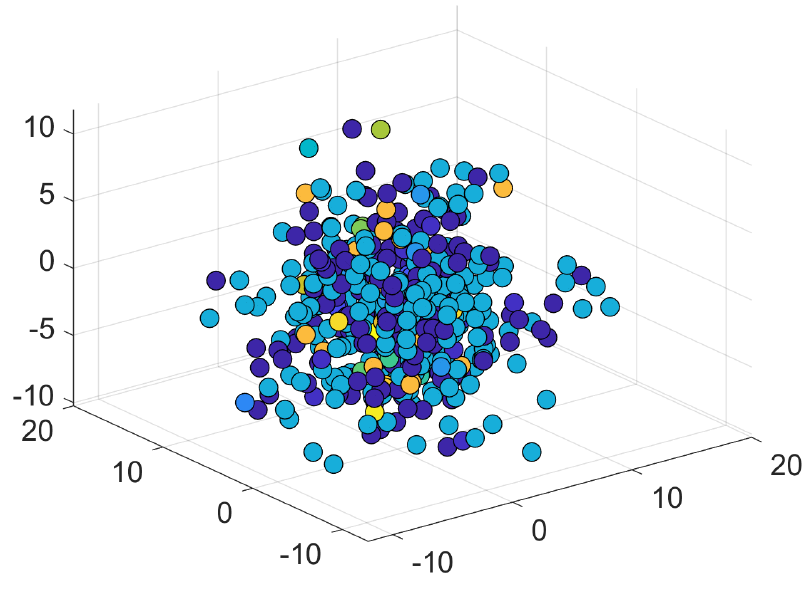}}
\vspace{-0.1in}
\end{minipage}
\caption{\normalsize Example: Dataset partitioning (Synthetic dataset).}
\label{fig:cluster_random}
\vspace{-0.00in}
\end{figure*}

\subsection{Examples of BD Convergence}
\label{subsec:BDconvergence}
Recall that Table \ref{Tb:exp:convergency} compares the number of iterations to converge to the optimal among the four dataset partitioning algorithms. In this part, we give examples to visually illustrate how the BD algorithm convergence over iterations when applying different dataset partitioning algorithms. Specifically, we test the convergence by comparing the upper bound and the lower bound in \textbf{Proposition \ref{prop:BDbound}}, as the optimal data utility loss must be within the upper and lower bounds. 
\begin{itemize}
\item Fig. \ref{fig:convergence_rome_downtown}(a)--(d) compares the convergence of BD when applying the four dataset partitioning methods in the \textbf{geo-location data (road network)}.
\item Fig. \ref{fig:convergence_grid}(a)--(d) compares the convergence of BD when applying the four dataset partitioning methods in the \textbf{geo-location data (grid map)}.
\item Fig. \ref{fig:convergence_text}(a)--(d) compares the convergence of BD when applying the four dataset partitioning methods in a \textbf{text dataset}.
\item Fig. \ref{fig:convergence_random}(a)--(d) compares the convergence of BD when applying the four dataset partitioning methods in a \textbf{synthetic dataset}.
\end{itemize}
When creating Figs. \ref{fig:convergence_rome_downtown} to \ref{fig:convergence_random}, we assigned an objective value of 9 to any subproblem that returned an unbounded solution, which helps visualize the number of unbounded solutions among the subproblems.

\begin{figure*}[h]
\centering
\hspace{0.00in}
\begin{minipage}{0.235\textwidth}
  \subfigure[\small k-m-DV]{
\includegraphics[width=1.00\textwidth, height = 0.13\textheight]{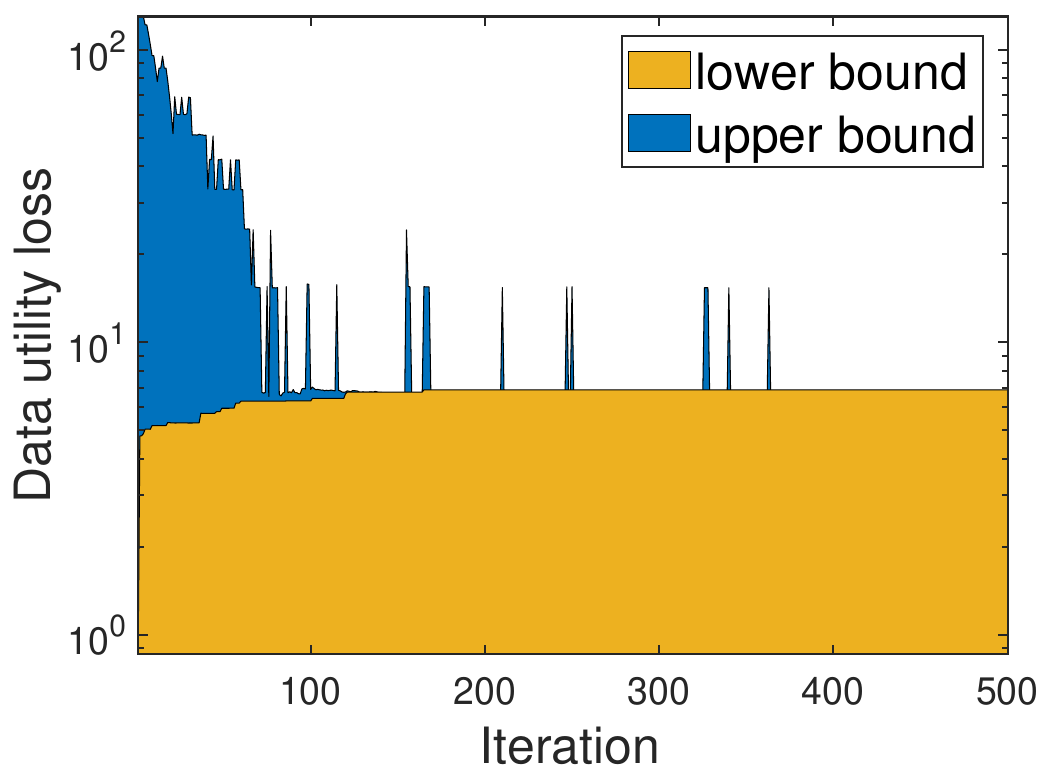}}
\vspace{-0.10in}
\end{minipage}
\hspace{0.00in}
\begin{minipage}{0.235\textwidth}
  \subfigure[\small
  k-m-rec]{
\includegraphics[width=1.00\textwidth, height = 0.13\textheight]{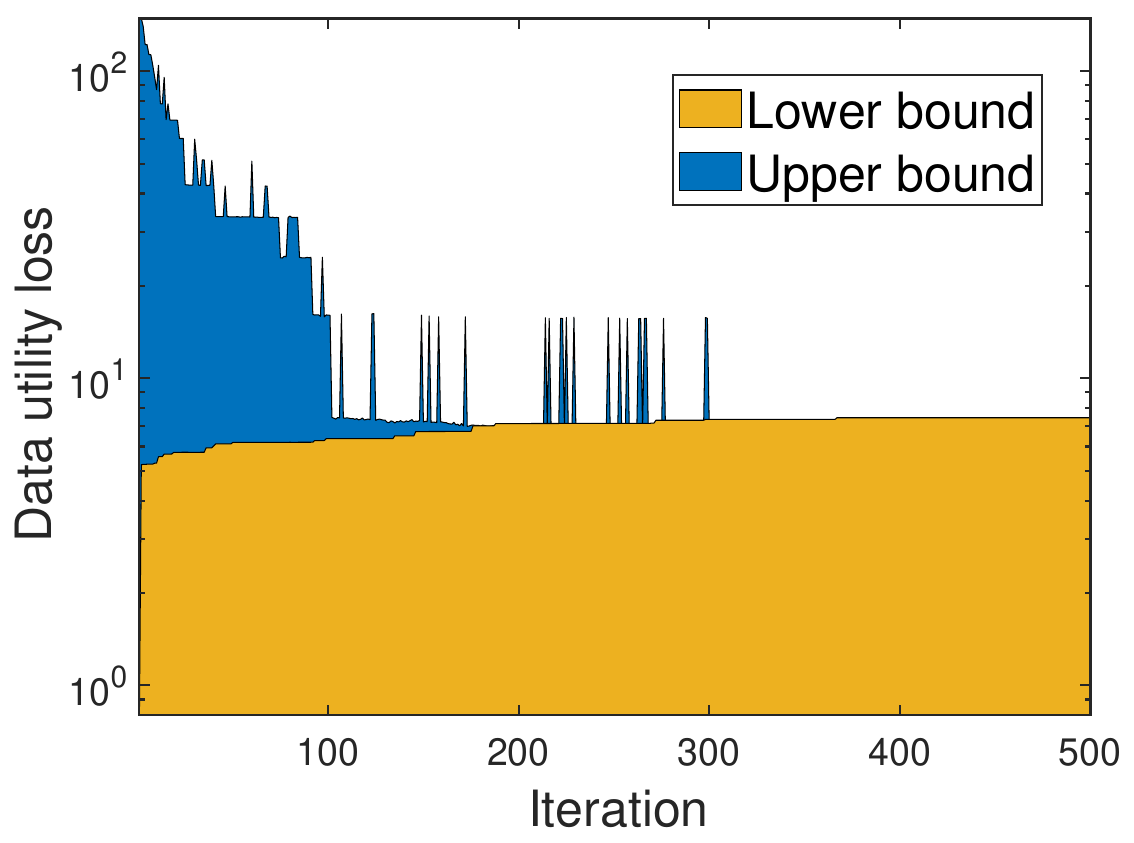}}
\vspace{-0.10in}
\end{minipage}
\hspace{0.00in}
\begin{minipage}{0.235\textwidth}
  \subfigure[\small k-m-adj]{
\includegraphics[width=1.00\textwidth, height = 0.13\textheight]{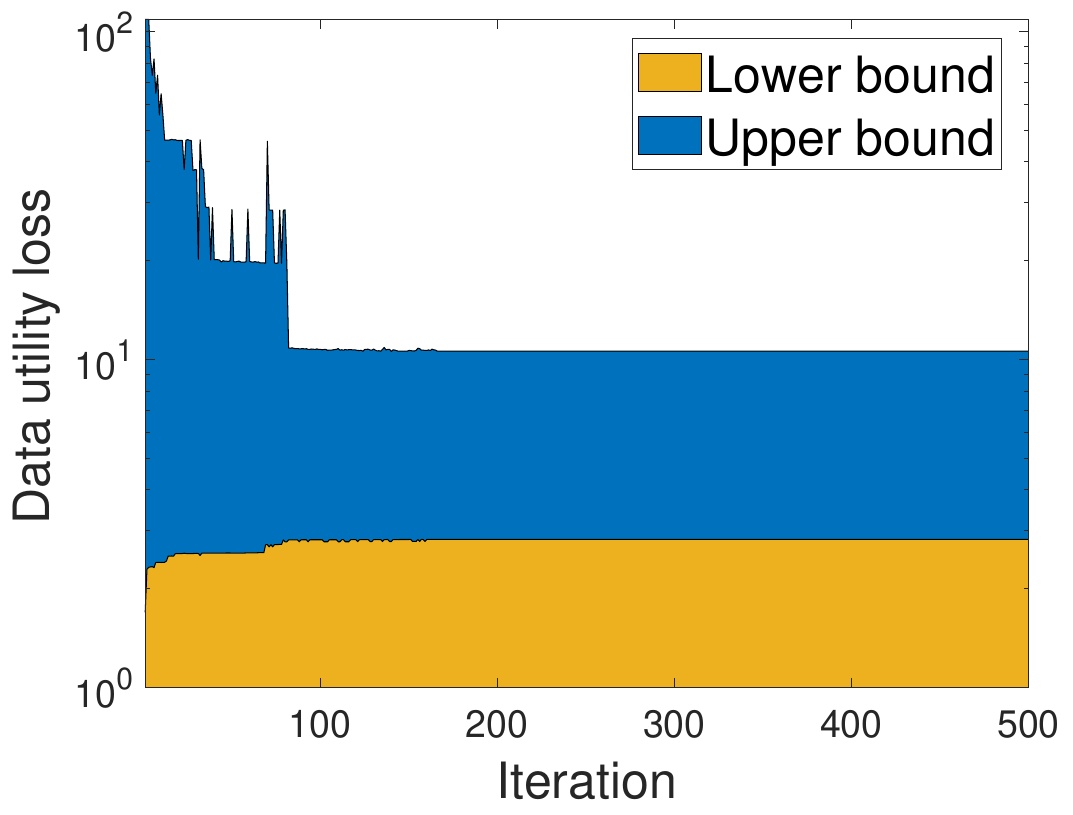}}
\vspace{-0.10in}
\end{minipage}
\hspace{0.00in}
\begin{minipage}{0.235\textwidth}
  \subfigure[\small BSC]{
\includegraphics[width=1.00\textwidth, height = 0.13\textheight]{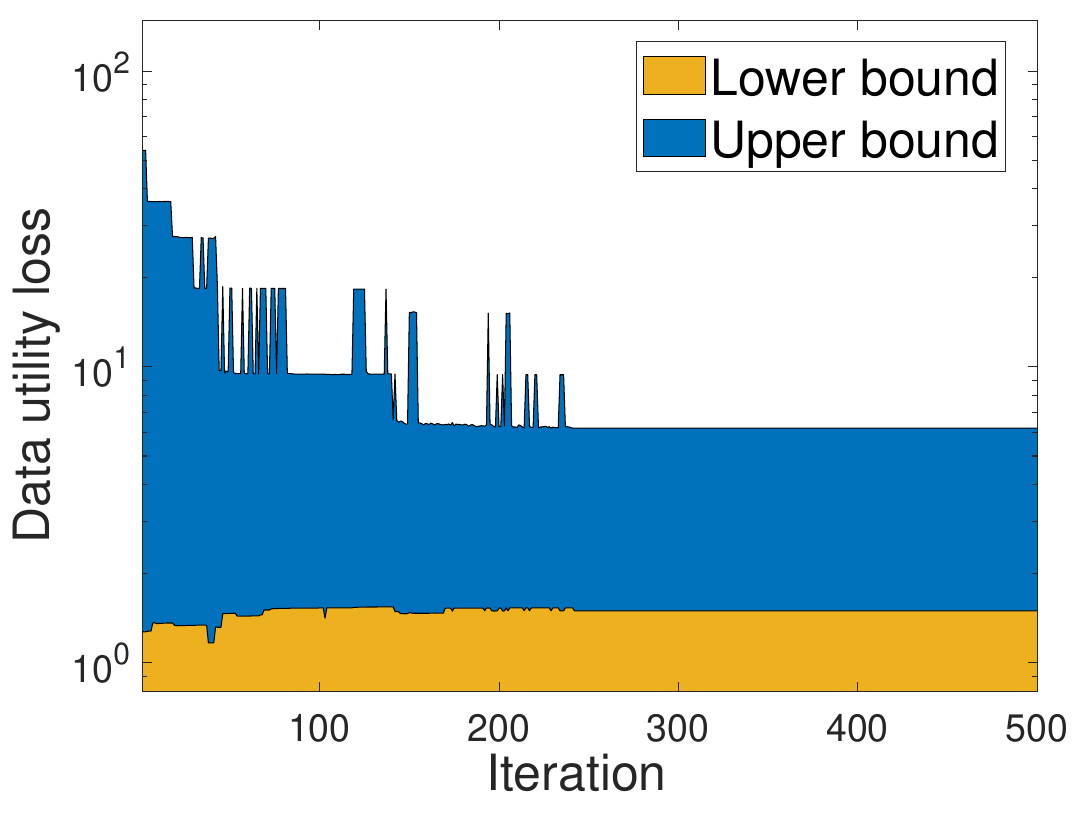}}
\vspace{-0.10in}
\end{minipage}
\caption{\normalsize Example: Convergence of Benders Decomposition (Geo-location data in the road network).}
\label{fig:convergence_rome_downtown}
\vspace{0.10in}
\centering
\hspace{0.00in}
\begin{minipage}{0.235\textwidth}
  \subfigure[\small k-m-DV]{
\includegraphics[width=1.00\textwidth, height = 0.13\textheight]{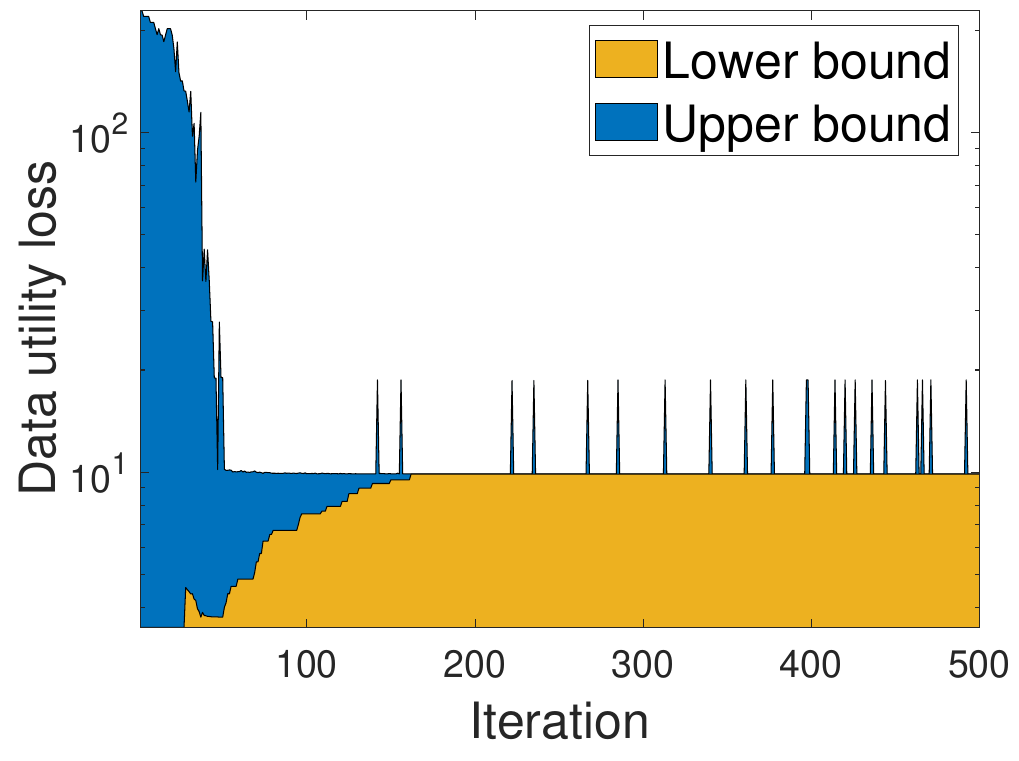}}
\vspace{-0.10in}
\end{minipage}
\hspace{0.00in}
\begin{minipage}{0.235\textwidth}
  \subfigure[\small
  k-m-rec]{
\includegraphics[width=1.00\textwidth, height = 0.13\textheight]{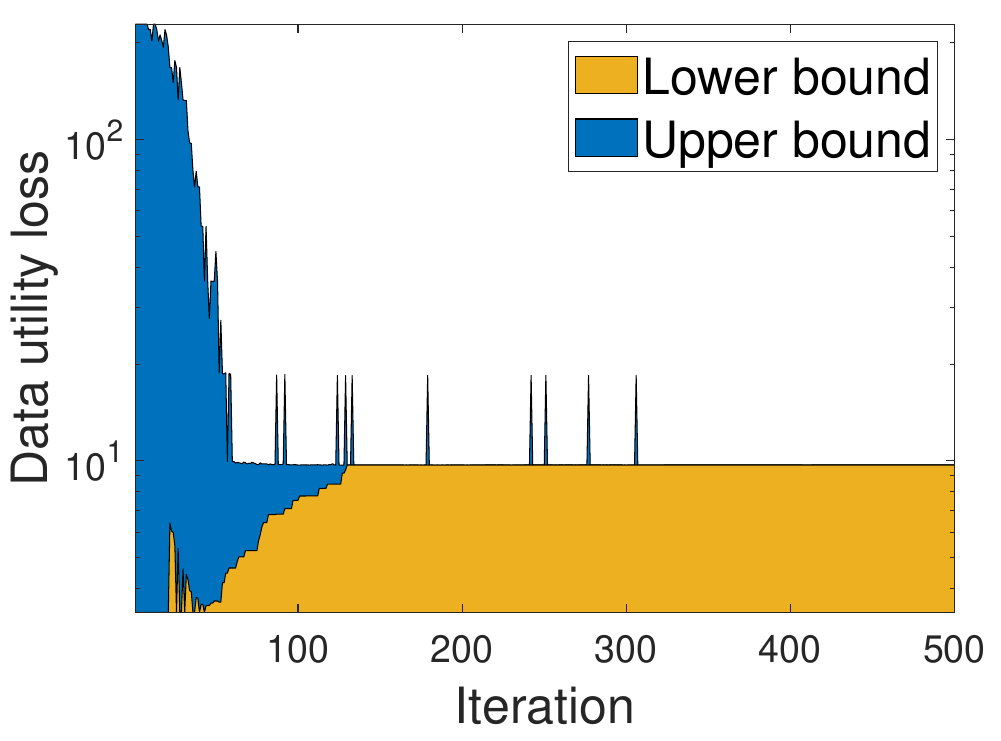}}
\vspace{-0.10in}
\end{minipage}
\hspace{0.00in}
\begin{minipage}{0.235\textwidth}
  \subfigure[\small k-m-adj]{
\includegraphics[width=1.00\textwidth, height = 0.13\textheight]{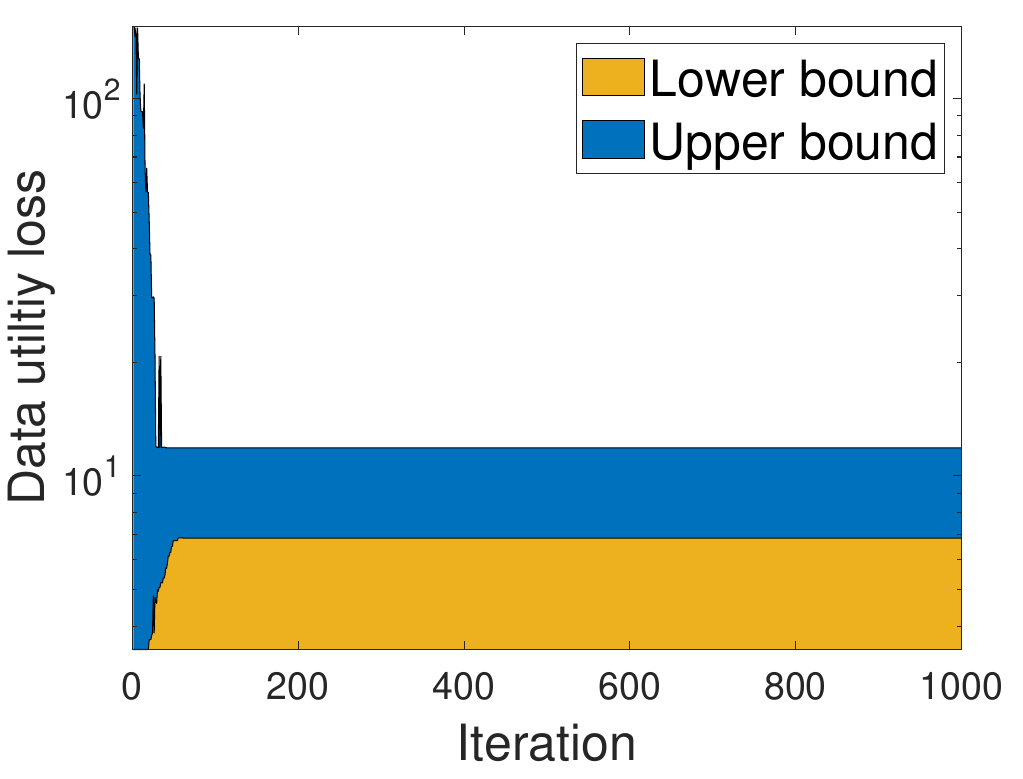}}
\vspace{-0.10in}
\end{minipage}
\hspace{0.00in}
\begin{minipage}{0.235\textwidth}
  \subfigure[\small BSC]{
\includegraphics[width=1.00\textwidth, height = 0.13\textheight]{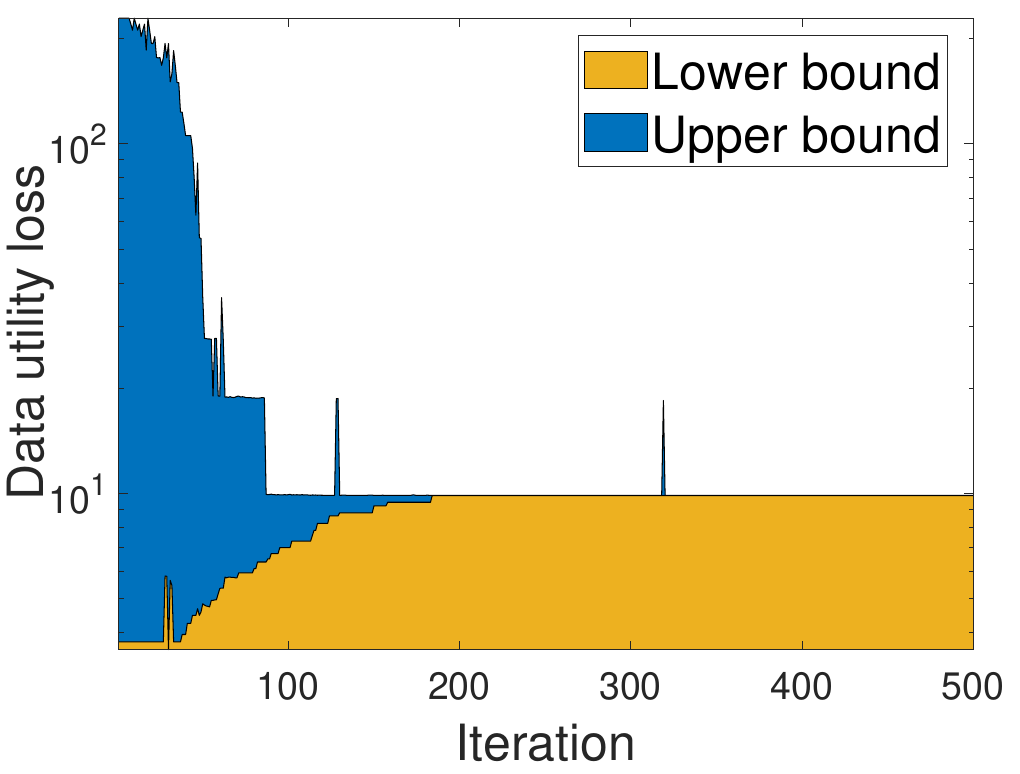}}
\vspace{-0.10in}
\end{minipage}
\caption{\normalsize Example: Convergence of Benders Decomposition (Geo-location data in a grid map).}
\label{fig:convergence_grid}
\vspace{0.10in}
\centering
\hspace{0.00in}
\begin{minipage}{0.235\textwidth}
  \subfigure[\small \gr k-m-DV]{
\includegraphics[width=1.00\textwidth, height = 0.13\textheight]{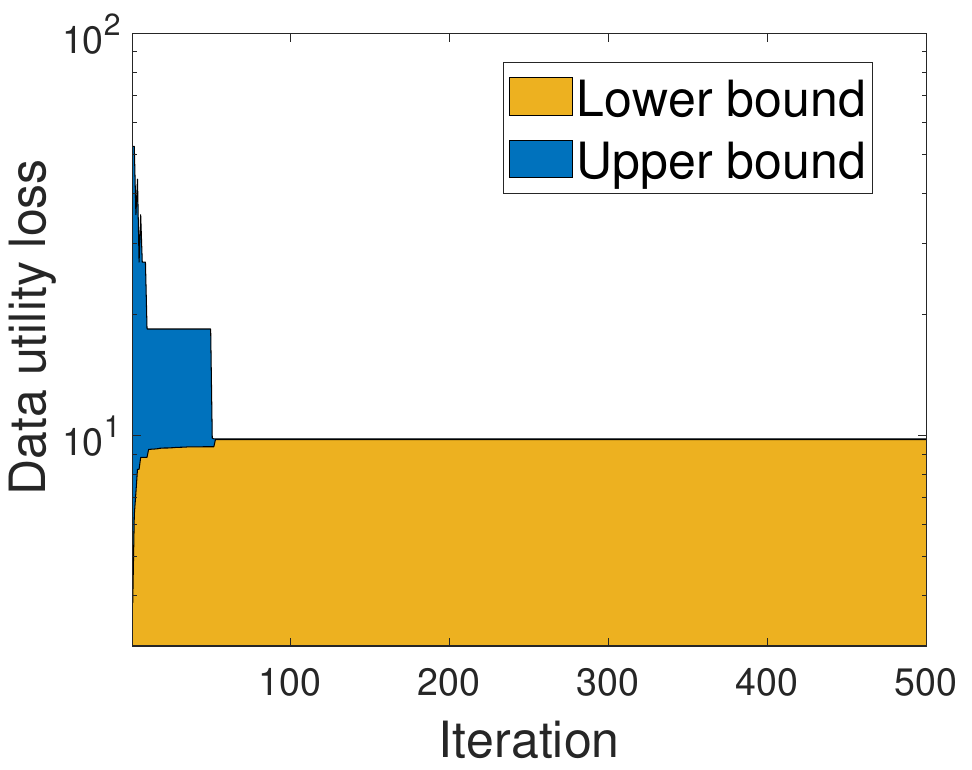}}
\vspace{-0.10in}
\end{minipage}
\hspace{0.00in}
\begin{minipage}{0.235\textwidth}
  \subfigure[\small
  k-m-rec]{
\includegraphics[width=1.00\textwidth, height = 0.13\textheight]{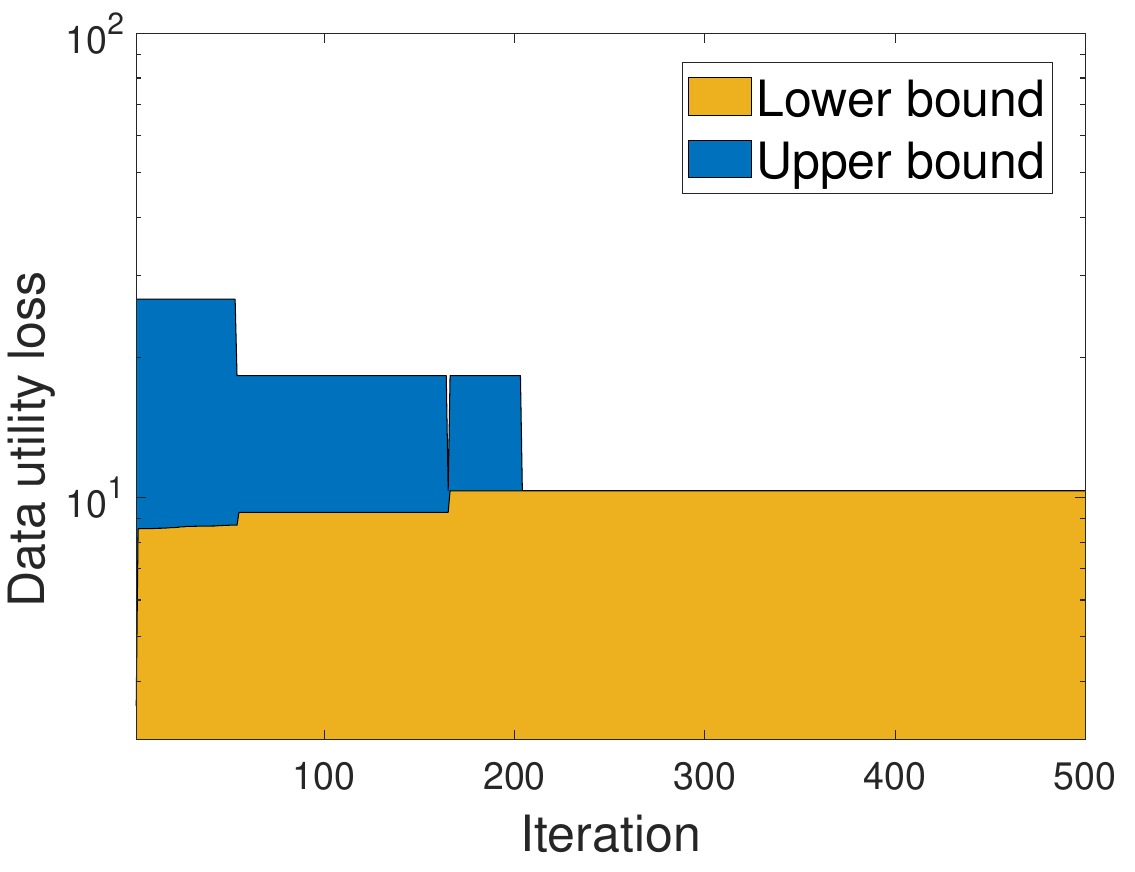}}
\vspace{-0.10in}
\end{minipage}
\hspace{0.00in}
\begin{minipage}{0.235\textwidth}
  \subfigure[\small k-m-adj]{
\includegraphics[width=1.00\textwidth, height = 0.13\textheight]{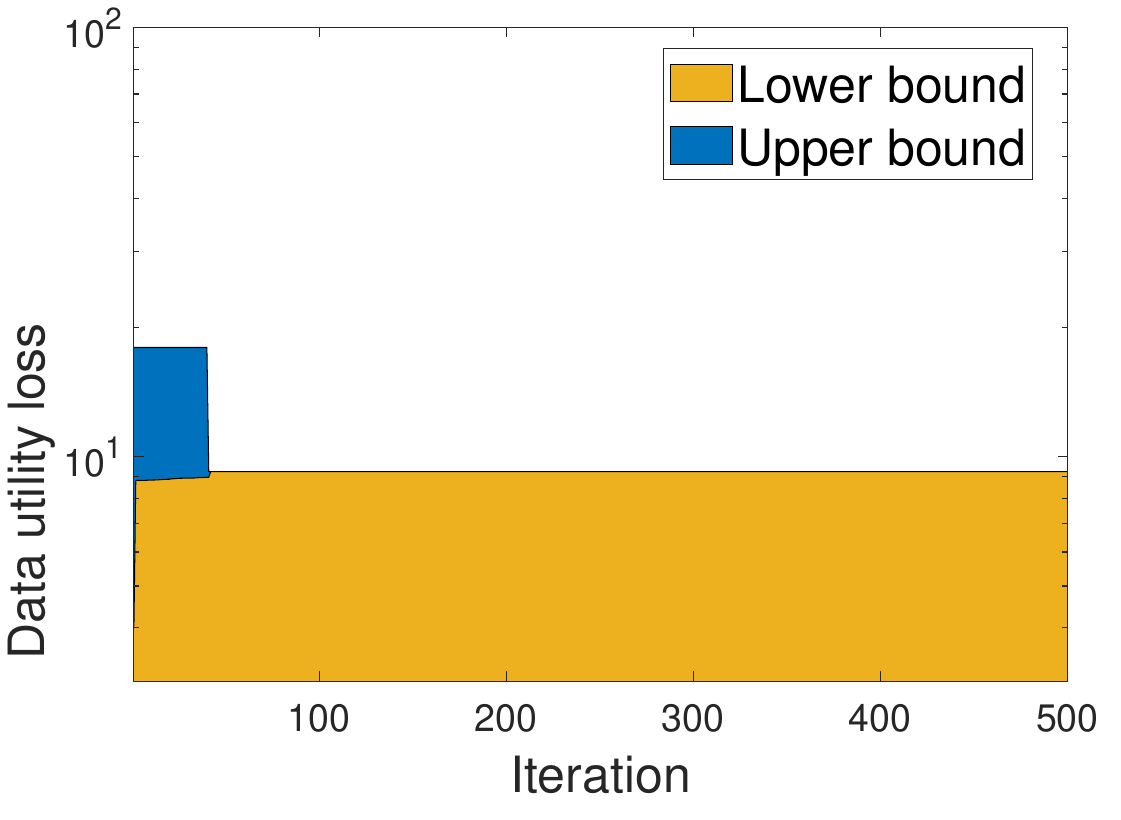}}
\vspace{-0.10in}
\end{minipage}
\hspace{0.00in}
\begin{minipage}{0.235\textwidth}
  \subfigure[\small BSC]{
\includegraphics[width=1.00\textwidth, height = 0.13\textheight]{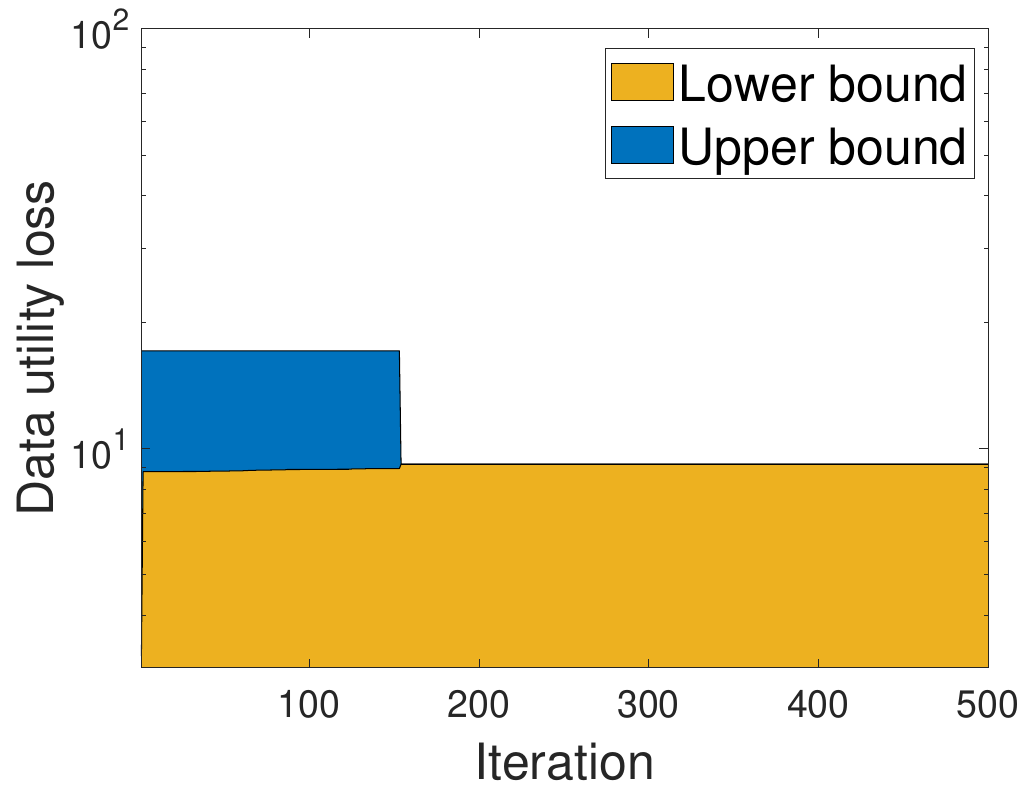}}
\vspace{-0.10in}
\end{minipage}
\caption{\normalsize Example: Convergence of Benders Decomposition (Text dataset).}
\label{fig:convergence_text}
\vspace{0.10in}

\centering
\hspace{0.00in}
\begin{minipage}{0.235\textwidth}
  \subfigure[\small k-m-DV]{
\includegraphics[width=1.00\textwidth, height = 0.13\textheight]{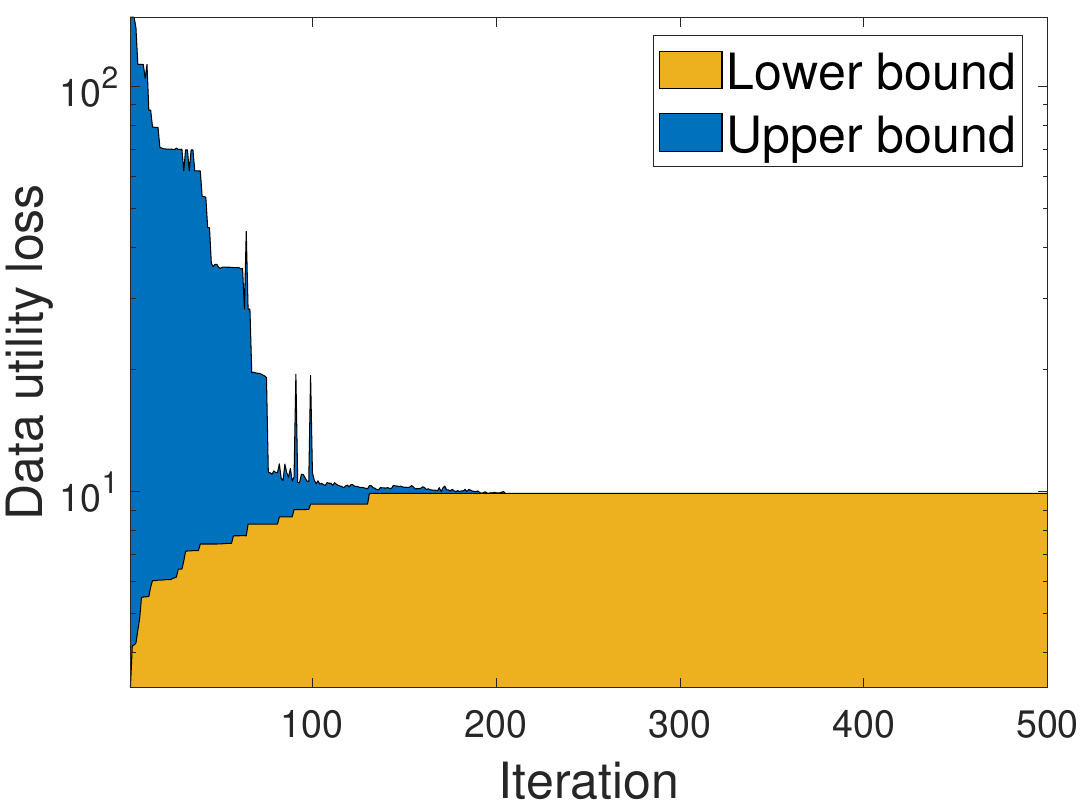}}
\vspace{-0.10in}
\end{minipage}
\hspace{0.00in}
\begin{minipage}{0.235\textwidth}
  \subfigure[\small
  k-m-rec]{
\includegraphics[width=1.00\textwidth, height = 0.13\textheight]{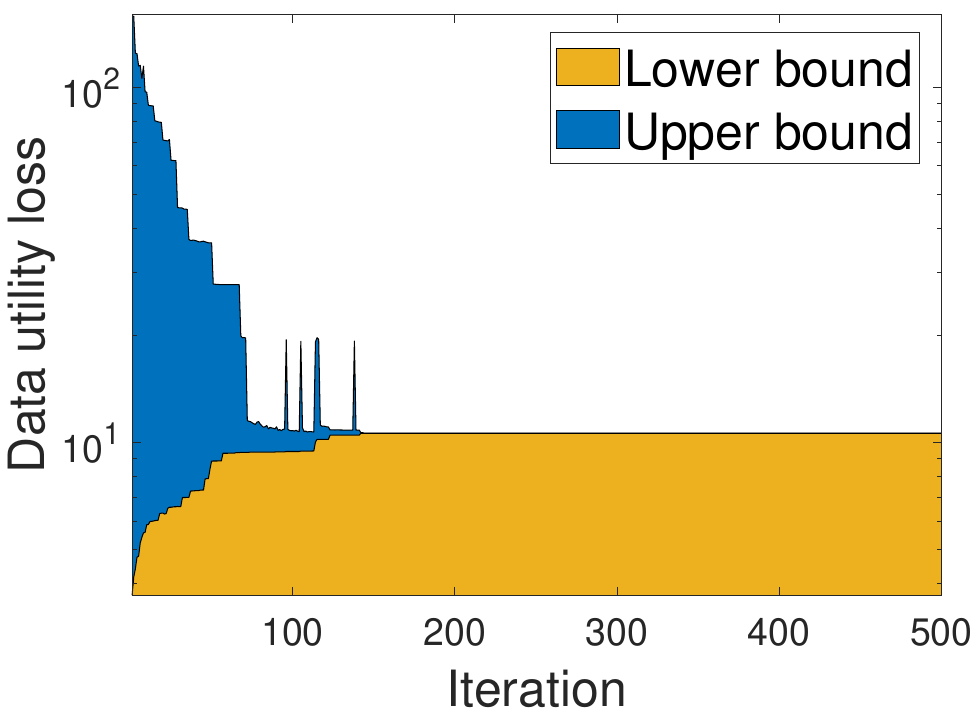}}
\vspace{-0.10in}
\end{minipage}
\hspace{0.00in}
\begin{minipage}{0.235\textwidth}
  \subfigure[\small k-m-adj]{
\includegraphics[width=1.00\textwidth, height = 0.13\textheight]{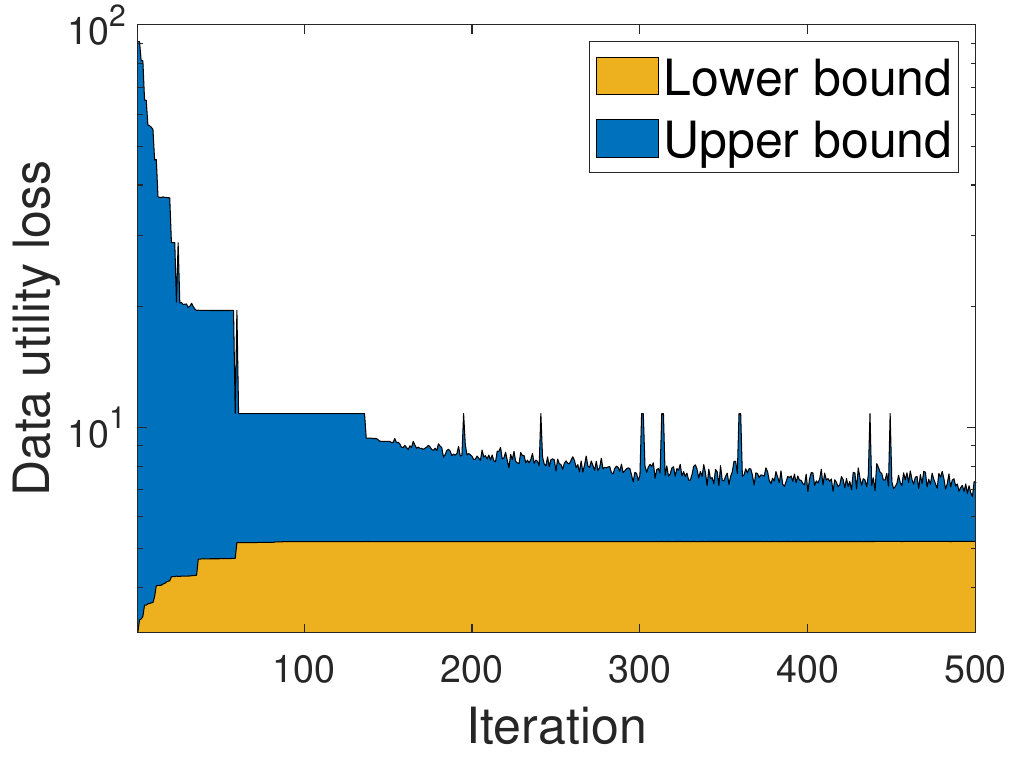}}
\vspace{-0.10in}
\end{minipage}
\hspace{0.00in}
\begin{minipage}{0.235\textwidth}
  \subfigure[\small BSC]{
\includegraphics[width=1.00\textwidth, height = 0.13\textheight]{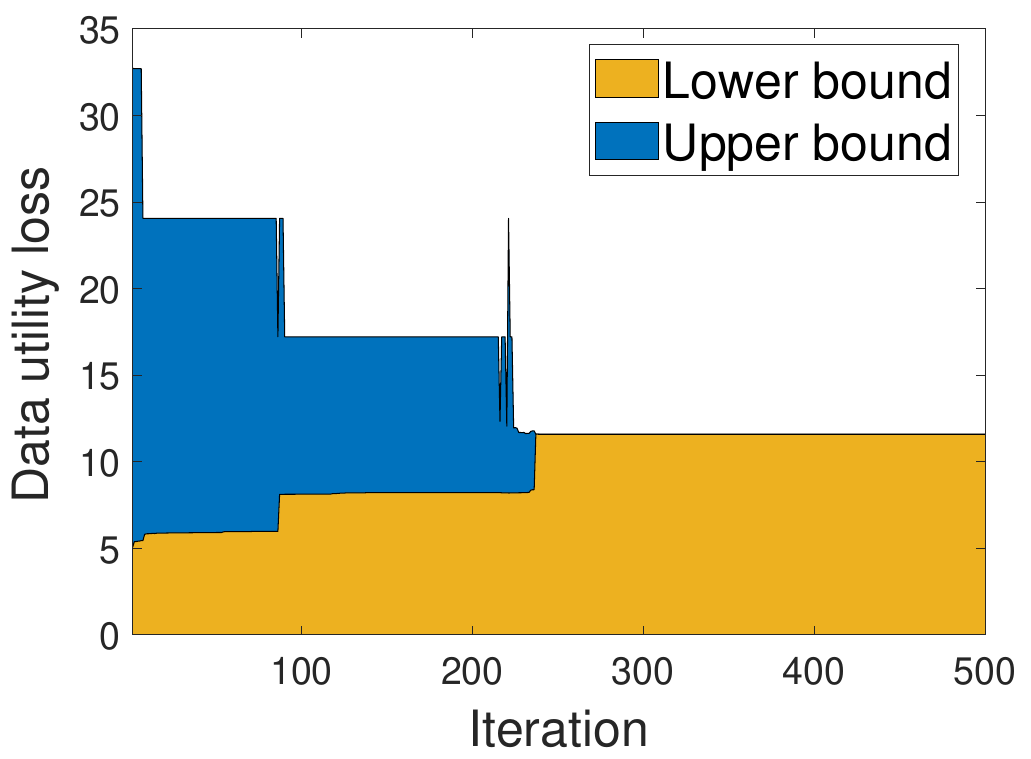}}
\vspace{-0.10in}
\end{minipage}
\caption{\normalsize Example: Convergence of Benders Decomposition (Synthetic dataset).}
\label{fig:convergence_random}
\end{figure*}

\end{document}